\documentclass{custom} 

\title[Wasserstein variational inference]{Variational inference via Wasserstein gradient flows}
\usepackage{algorithm}
\usepackage{algpseudocode}
\usepackage{amsmath, amssymb}
\usepackage{bbm}
\usepackage{caption}
\usepackage{enumitem}
\usepackage{graphicx}
\usepackage{mathrsfs}
\usepackage{mathtools}
\usepackage{microtype}
\usepackage{tikz}
\usepackage{times}
\usepackage[Symbolsmallscale]{upgreek}
\usepackage{thm-restate} 
\usepackage{mdframed} 
\usepackage{euscript}
\usepackage{wrapfig}

\usepackage{custom}

\makeatletter
\def\set@curr@file#1{\def\@curr@file{#1}} %temp workaround for 2019 latex release
\makeatother
\usepackage[load-configurations=version-1]{siunitx} % newer version

\usepackage{float}

\hypersetup{citecolor=cyan, colorlinks=true, linkcolor=blue!75!black}
\coltauthor{\Name{Marc Lambert} \Email{marc.lambert@inria.fr} \\ \addr{DGA, INRIA, Ecole Normale Sup\'erieure, PSL Research University} \AND
\Name{Sinho Chewi} \Email{schewi@mit.edu} \\ \addr{MIT} \AND
\Name{Francis Bach} \Email{francis.bach@inria.fr} \\ \addr{INRIA, Ecole Normale Sup\'erieure, PSL Research University} \AND
\Name{Silv\`ere Bonnabel} \Email{silvere.bonnabel@minesparis.psl.eu} \\ \addr{MINES Paris PSL, Universit\'e de la Nouvelle-Cal\'edonie} \AND
\Name{Philippe Rigollet} \Email{rigollet@math.mit.edu} \\ \addr{MIT}
}

\begin{document}

\maketitle

\begin{abstract}
Along with Markov chain Monte Carlo ({\sc mcmc}) methods, variational inference ({\sc vi}) has emerged as a central computational approach to large-scale Bayesian inference. Rather than sampling from the true posterior $\pi$, {\sc vi} aims at producing a simple but effective approximation $\hat \pi$ to $\pi$ for which summary statistics are easy to compute. However, unlike the well-studied {\sc mcmc} methodology, algorithmic guarantees for {\sc vi} are still relatively less well-understood. In this work, we propose principled methods for {\sc vi}, in which $\hat \pi$ is taken to be a Gaussian or a mixture of Gaussians, which rest upon the theory of gradient flows on the Bures--Wasserstein space of Gaussian measures. Akin to {\sc mcmc}, it comes with strong theoretical guarantees when $\pi$ is  log-concave.
\end{abstract}

% \tableofcontents

% \newpage 

\section{Introduction}

This work brings together three active research areas: variational inference, variational Kalman filtering, and gradient flows on the Wasserstein space.

\paragraph{Variational inference.}
The development of large-scale Bayesian methods has fueled the need for fast and scalable methods to approximate complex distributions. More specifically, Bayesian methodology typically generates a high-dimensional posterior distribution $\pi \propto \exp (-V)$ that is known only up to normalizing constants, making the computation even of simple summary statistics such as the mean and covariance a major computational hurdle. To overcome this limitation, two distinct computational approaches are largely favored. The first approach consists of Markov chain Monte Carlo (\mcmc{}) methods that rely on carefully constructed Markov chains which (approximately) converge to $\pi$. For example,  the \emph{Langevin diffusion}
\begin{align}\label{eq:langevin}
    \D X_t
    &= - \nabla V(X_t) \, \D t + \sqrt 2 \, \D B_t \,,
\end{align}
 where ${(B_t)}_{t\ge 0}$ denotes standard Brownian motion on $\R^d$, admits $\pi$ as a stationary distribution.
 Crucially, the Langevin diffusion can be discretized and implemented without knowledge of the normalizing constant of $\pi$, leading to practical algorithms for Bayesian inference. 
Recent theoretical efforts have produced sharp non-asymptotic convergence guarantees for algorithms based on the Langevin diffusion (or variants thereof), with many results known when $\pi$ is strongly log-concave or satisfies isoperimetric assumptions~\citep[see, e.g.,][]{durmusmajewski2019lmcconvex, shenlee2019randomizedmidpoint, vempala2019ulaisoperimetry, chenetal2020hmc, dalalyanrioudurand2020underdamped, chewietal2021lmcpoincare, leeshentian2021rgo, maetal2021nesterovmcmc, wuschche2022minimaxmala}.

More recently, Variational Inference (\vi{}) has emerged as a viable alternative to \mcmc{}~\citep{JorGhaJaa99, WaiJor08,BleKucMcA17}. The goal of \vi{} is to approximate the posterior~$\pi$ by a more tractable distribution $\hat \pi \in \cP$ such that
\begin{align}\label{eq:vi}
    \hat \pi
    &\in \argmin_{p \in \mc P} \KL(p \mmid \pi)\,.
\end{align}
A common example arises when $\mc P$ is the class of product distributions, in which case $\hat \pi$ is called the \emph{mean-field} approximation of $\mc P$. Unfortunately, by definition, mean-field approximations fail to capture important correlations present in the posterior $\pi$, and various remedies have been proposed, with varied levels of success.  In this paper, we largely focus on obtaining a Gaussian approximation to $\pi$, that is, we take $\mc P$ to be the class of non-degenerate Gaussian distributions on $\R^d$~\citep{Barber97, seeger1999bayesianmodelselection, honkelavalpola2004variationalbayes, Opper09, ZhaSunDuv18, XuCam22computationalgvi}.
The expressive power of the variational model may then be further increased by considering mixture distributions~\citep{linkhanschmidt2019mixturevi, daudeldouc2021mixtureopt, daudeldoucportier2021alphadiv}.

Although the solution $\hat \pi$ of~\eqref{eq:vi} is no longer equal to the true posterior, variational inference remains heavily used in practice because the problem~\eqref{eq:vi} can be solved for simple models $\mc P$ via scalable optimization algorithms.
In particular, \vi{} avoids many of the practical hurdles associated with \mcmc{} methods---such as the potentially long ``burn-in'' period of samplers and the lack of effective stopping criteria for the algorithm---while still producing informative summary statistics.
In this regard, we highlight the fact that obtaining an approximation for the covariance matrix of $\pi$ via \mcmc{} methods requires drawing potentially many samples, whereas for many choices of $\mc P$ (e.g., the Gaussian approximation) the covariance matrix of $\hat\pi$ can be directly obtained from the solution to the \vi{} problem~\eqref{eq:vi}.

\begin{wrapfigure}[14]{r}{0.5\textwidth}
\vspace{-1.5em}
  \centering
    \includegraphics[width=0.24\textwidth]{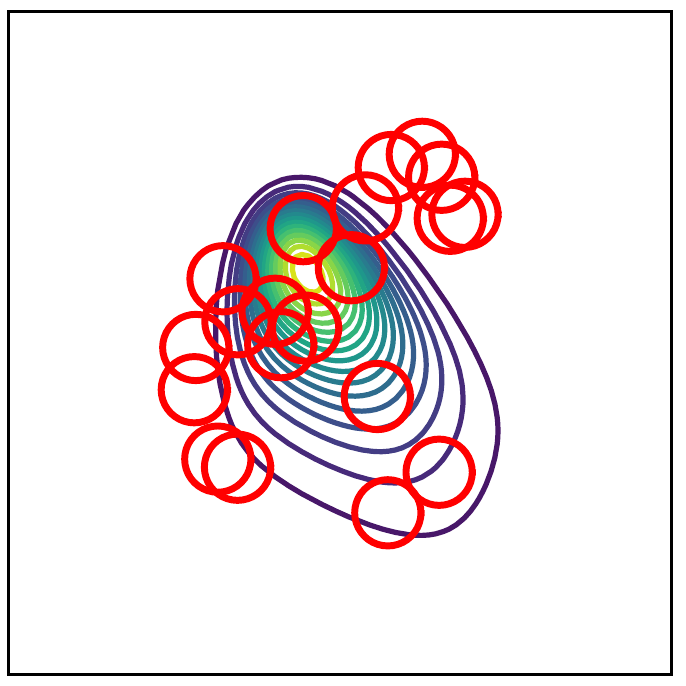}
    \includegraphics[width=0.24\textwidth]{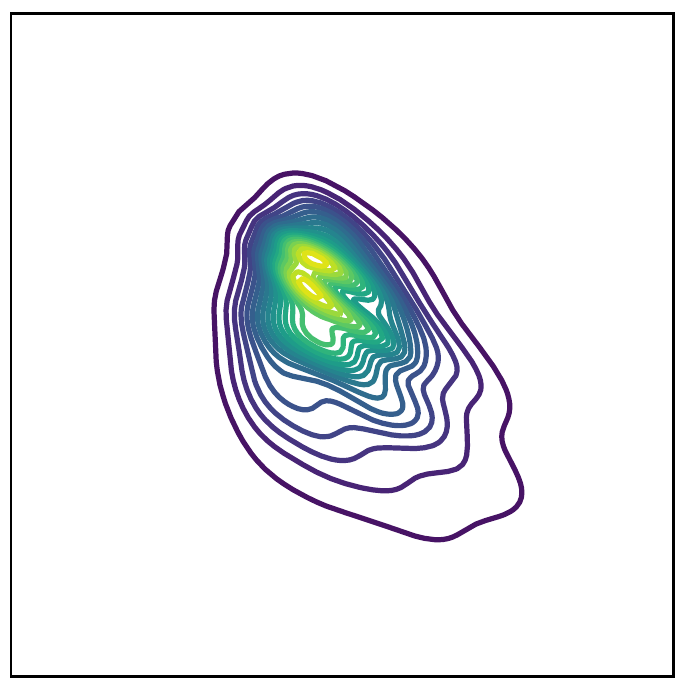}
    \caption{Left: randomly initialized mixture of $20$ Gaussians (the initial covariances are depicted as red circles) and contour plot of a logistic target~$\pi$. Right: contour lines of a mixture of Gaussians approximation $\hat \pi$ obtained from the gradient flow in Section~\ref{scn:mixtures}.}
    \label{fig:intro}
\end{wrapfigure}

However, in contrast with \mcmc{} methods, to date there have not been many theoretical guarantees for \vi{}, even when $\pi$ is strongly log-concave and $\mc P$ is taken to be the class of Gaussians $\normal(m, \Sigma)$.
The problem stems from the fact that the objective in~\eqref{eq:vi} is typically non-convex in the pair $(m,\Sigma)$.
Obtaining such guarantees remains a pressing challenge for the field.

\paragraph{Variational Kalman filtering.}
There is also considerable interest in extending ideas behind variational inference to dynamical settings of Bayesian inference. 
Consider a general framework where ${(\pi_t)}_{t}$ represents the marginal laws of a stochastic process indexed by time $t$, which can be discrete or continuous.
The goal is to recursively build a Gaussian approximation to ${(\pi_t)}_{t}$.

As a concrete example, suppose that ${(\pi_t)}_{t\ge 0}$ denotes the marginal law of the solution to the Langevin diffusion~\eqref{eq:langevin}.
In the context of Bayesian optimal filtering and smoothing,~\cite{Sarkka07}  proposed the following heuristic. 
Let $(m_t,\Sigma_t)$ denote the mean and covariance matrix of $\pi_t$.
Then, it can be checked (see Section~\ref{scn:mean_cov_fp}) that
\begin{align}\label{eq:mean_cov_fp}
   {\begin{aligned}
        \dot m_t
        &= -\E \nabla V(X_t) \\
        \dot \Sigma_t
        &= 2I - \E[\nabla V(X_t) \otimes (X_t - m_t) + (X_t - m_t) \otimes \nabla V(X_t)]
    \end{aligned}}
\end{align}
where $X_t \sim \pi_t$.
These ordinary differential equations (ODEs) are intractable because they involve expectations under the law of $X_t\sim \pi_t$, which is not available to the practitioner.
However, if we replace   $X_t\sim \pi_t$ with a Gaussian $Y_t \sim p_t = \normal(m_t,\Sigma_t)$ with the same mean and covariance as $X_t$,
then the system of ODEs
\begin{align}\label{eq:sarkka}
    \boxed{\begin{aligned}
        \dot m_t
        &= -\E  \nabla V(Y_t) \\
        \dot \Sigma_t
        &= 2I - \E [\nabla V(Y_t) \otimes (Y_t - m_t) + (Y_t - m_t) \otimes \nabla V(Y_t)]
    \end{aligned}}
\end{align}
yields a well-defined evolution of Gaussian distributions
${(p_t)}_{t\ge 0}$, which we may optimistically believe to be a good approximation of ${(\pi_t)}_{t\ge 0}$.
Moreover, the system of ODEs can be numerically approximated efficiently in practice using Gaussian quadrature rules to compute the above expectations. This is the principle behind the unscented Kalman filter~\citep{julieruhlmanndurrantwhyte2000ukf}.

In the context of the Langevin diffusion, S\"arkk\"a's heuristic \eqref{eq:sarkka} provides a promising avenue towards computational \vi{}.
Indeed, since $\pi \propto \exp(-V)$ is the unique stationary distribution of the Langevin diffusion~\eqref{eq:langevin}, an algorithm to approximate ${(\pi_t)}_{t\ge 0}$ is expected to furnish an algorithm to solve the VI problem~\eqref{eq:vi}.
However, at present there is little theoretical understanding of how the system~\eqref{eq:sarkka} approximates~\eqref{eq:mean_cov_fp}; moreover, S\"arkk\"a's heuristic only provides Gaussian approximations, and it is unclear how to extend the system~\eqref{eq:sarkka} to more complex models (e.g., mixtures of Gaussians).

\paragraph{Our contributions: bridging the gap via Wasserstein gradient flows.}
We show that the approximation ${(p_t)}_{t\ge 0}$ in S\"arkk\"a's heuristic~\eqref{eq:sarkka} arises precisely as the gradient flow of the Kullback--Leibler (KL) divergence $\KL(\cdot \mmid \pi)$ on the Bures{--}Wasserstein space of Gaussian distributions on $\R^d$ endowed with the $2$-Wasserstein distance from optimal transport~\citep{villani2003topics}. This perspective allows us to not only understand its convergence but also to extend it to the richer space of  mixtures of Gaussian distributions, and propose an implementation as a novel system of interacting ``Gaussian particles''. Below, we proceed to describe our contributions in greater detail.

Our framework builds upon the seminal work of~\cite{Jordan98}, which introduced the celebrated \emph{JKO scheme} in order to give meaning to the idea that the evolving marginal law of the Langevin diffusion~\eqref{eq:langevin} is a gradient flow of $\KL(\cdot \mmid \pi)$ on the Wasserstein space $\mc P_2(\R^d)$ of probability measures with finite second moments. Subsequently, in order to emphasize the Riemannian geometry underlying this result,~\citet{otto2001porousmedium} developed his eponymous calculus on $\mc P_2(\R^d)$, a framework which has had tremendous impact in analysis, geometry, PDE, probability, and statistics.

Inspired by this perspective, we show in Theorem~\ref{thm:main} that S\"arkk\"a's approximation ${(p_t)}_{t\ge 0}$ is also a gradient flow of $\KL(\cdot \mmid \pi)$, with the main difference being that it is \emph{constrained} to lie on the submanifold $\BW(\R^d)$ of $\mc P_2(\R^d)$ consisting of Gaussian distributions, known as the Bures--Wasserstein manifold. In turn,
our result paves the way for new theoretical understanding via the powerful theory of gradient flows.
As a first step, using well-known results about convex functionals on the Wasserstein space, we show in Corollary~\ref{COR:CONT_TIME_GUARANTEE} that ${(p_t)}_{t\ge 0}$ converges
rapidly to the solution of the \vi{} problem~\eqref{eq:vi} with $\mc P = \BW(\R^d)$ as soon as $V$ is convex. Moreover, in Section~\ref{scn:numerical_integration}, we apply numerical integration based on cubature rules for Gaussian integrals to the system of ODEs~\eqref{eq:sarkka}, thus arriving at a fast method with robust empirical performance (details in Sections~\ref{scn:xp-vi} and \ref{scn:xp-gmmvi}).

This combination of results brings \vi{} closer to Langevin-based \mcmc{} both on the practical and theoretical fronts, but still falls short of achieving non-asymptotic discretization guarantees as pioneered by~\citet{Dal17a} for \mcmc{}.
To further close the theoretical gap between \vi{} and the state of the art for \mcmc{}, we propose in Section~\ref{scn:bwsgd} a stochastic gradient descent (SGD) algorithm as a time discretization of   the Bures{--}Wasserstein  gradient flow.
This algorithm comes with convergence guarantees that establish \vi{} as a solid competitor to \mcmc{} not only from a practical standpoint but also from a theoretical one.
Both have their relative merits; whereas \mcmc{} targets the true posterior, \vi{} leads to fast computation of summary statistics of the approximation $\hat\pi$ to $\pi$.

In Section~\ref{scn:mixtures}, we consider an
extension of these ideas to the substantially more flexible class of mixtures of Gaussians. Namely, the space of mixtures of Gaussians can be identified as a Wasserstein space over $\BW(\R^d)$ and hence inherits Otto's differential calculus. Leveraging this viewpoint, in Theorem~\ref{THM:MIXTURE_GF} we derive the gradient flow of $\KL(\cdot \mmid \pi)$ over the space of mixtures of Gaussians and propose to implement it via a system of interacting particles. Unlike typical particle-based algorithms,
here our particles correspond to Gaussian distributions, and the collection thereof to a Gaussian mixture which is better equipped to approximate a continuous measure. We validate the empirical performance of our method with promising experimental results (see Section~\ref{scn:xp-gmmvi}). Although we focus on the \vi{} problem in this work, we anticipate that our notion of ``Gaussian particles'' may be a broadly useful extension of classical particle methods for PDEs.

\paragraph{Related work.}
Classical \vi{} methods define a parametric family $\mc P=\{p_\theta\,:\, \theta \in \Theta\}$ and minimize $\theta \mapsto \KL(p_\theta\mmid \pi)$ over $\theta \in \Theta$ 
using off-the-shelf optimization algorithms~\citep{PaiBleJor12,RanGerBle14}. Since~\eqref{eq:vi} is an optimization problem over the space of probability distributions, we argue for methods that respect a natural geometric structure on this space.
In this regard, previous approaches to \vi{} using natural gradients implicitly employ a different geometry~\citep{Wu19, huangetal2022derivfree, Khan22}, namely the reparameterization-invariant Fisher--Rao geometry~\citep{AmaNag00}.
The application of Wasserstein gradient flows to \vi{} was introduced earlier in work on normalizing flows and Stein Variational Gradient Descent (SVGD)~\citep{liuwang2016svgd, liu2017svgdgf}.

Our work falls in line with a number of recent papers aiming to place \vi{} on a solid theoretical footing~\citep{alqridcho2016vi, wangblei2019vbayesconsistency, domke2020vismooth, knojewdam2022vi, XuCam22computationalgvi}.
Some of these works in particular have obtained non-asymptotic algorithmic guarantees for specific examples, see, e.g.,~\citet{chabar2013gaussiankl}.
We also mention that the approach we take in this paper is closely related to the algorithms and analysis arrived at in~\citet{alquierridgway2020variational, domke2020vismooth, GalPerOpp21vargauss}.
In particular,~\citet{GalPerOpp21vargauss} derive an algorithm for low-rank Gaussian \vi{} by seeking a descent condition for the KL divergence, yielding a method resembling Algorithm~\ref{alg:bwsgd} albeit without quantitative convergence guarantees.
Also,~\citet{alquierridgway2020variational, domke2020vismooth} show that parametrizing the Gaussian by the \emph{square root} of the covariance matrix yields convexity and smoothness properties for the Gaussian \vi{} objective, which in turn allows for applying Euclidean gradient methods. This choice of parametrization is closely related to the Bures{--}Wasserstein geometry approach we take, see Appendix~\ref{scn:bures_wasserstein} for background.
However, we note that these works do not analyze the effect of stochastic gradients, which is crucial for implementation.

The connection between \vi{} and Kalman filtering was studied in the static case by~\cite{Lambert21,Lambert22}, and extended to the dynamical case by \cite{Lambert22b}, providing a first justification of S\"arkk\"a's heuristic in terms of local variational Gaussian approximation. In particular, the closest linear process to the Langevin diffusion~\eqref{eq:langevin} is a Gaussian process governed by a McKean--Vlasov equation whose Gaussian marginals have parameters evolving according to S\"arkk\"a's ODEs.

Constrained gradient flows on the Wasserstein space have also been extensively studied~\citep{carlengangbo2003constrained, cagliotietal2009constrainedns, tudorascuwunsch2011nonlocal, eberleniethammerschlichting2017constrainedfp}, although our interpretation of S\"arkk\"a's heuristic is, to the best of our knowledge, new.

\section{Background}

In order to define gradient flows on the space of probability measures, we must first endow this space with a geometry; see Appendix~\ref{scn:otto} for more details. Given probability measures $\mu$ and $\nu$ on $\R^d$, define the \emph{$2$-Wasserstein distance}
\begin{align*}
    W_2(\mu,\nu)
    &= \Bigl[\inf_{\gamma \in \eu C(\mu,\nu)} \int \norm{x-y}^2 \, \D \gamma(x,y)\Bigr]^{1/2}\,,
\end{align*}
where $\eu C(\mu,\nu)$ is the set of \emph{couplings} of $\mu$ and $\nu$, that is, joint distributions on $\R^d\times\R^d$ whose marginals are $\mu$ and $\nu$ respectively.
This quantity is finite as long as $\mu$ and $\nu$ belong to the space $\mc P_2(\R^d)$ of probability measures over $\R^d$ with finite second moments.
The $2$-Wasserstein distance has the interpretation of measuring the smallest possible mean squared displacement of mass required to \emph{transport} $\mu$ to $\nu$; we refer to~\citet{villani2003topics, villani2009ot, santambrogio2015ot} for textbook treatments on optimal transport. Unlike other notions of distance between probability measures, such as the total variation distance, the $2$-Wasserstein distance respects the geometry of the underlying space $\R^d$, leading to numerous applications in modern data science~\citep[see, e.g.,][]{peyre2019computational}.

The space $(\mc P_2(\R^d), W_2)$ is a metric space~\citep[Theorem 7.3]{villani2003topics}, and we refer to it as the \emph{Wasserstein space}. However, as shown by Otto~\citep{otto2001porousmedium}, it has a far richer geometric structure: formally, $(\mc P_2(\R^d), W_2)$ can be viewed as a Riemannian manifold, a fact which allows for considering gradient flows of functionals on $\mc P_2(\R^d)$. A fundamental example of such a functional is the KL divergence $\KL(\cdot \mmid \pi)$ to a target density $\pi  \propto \exp(-V)$ on $\R^d$, for which~\citet{Jordan98} showed that the Wasserstein gradient flow is the same as the evolution of the marginal law of the Langevin diffusion~\eqref{eq:langevin}.
This optimization perspective has had tremendous impact on our understanding and development of \mcmc{} algorithms~\citep{wibisono2018samplingoptimization}.

\section{Variational inference with Gaussians} \label{sec2}

In this section we describe our  problem using two equivalent approaches: a variational approach based on a modified version of the JKO scheme of \citet{Jordan98} (Section    \ref{pb1}), and a Wasserstein gradient flow approach based on Otto calculus (Section \ref{pb2}). Both lead to the same result (Section \ref{scn:gradient_flow_kl_bw}). While the former is more accessible to readers who are unfamiliar with gradient flows on the
Wasserstein space, the latter leads to strong convergence guarantees (Section \ref{Convergence}). 

\subsection{Variational approach: the Bures--JKO  scheme} \label{pb1}
The space of non-degenerate Gaussian distributions on $\R^d$ equipped with the $W_2$ distance forms the \emph{Bures--Wasserstein space} $\BW(\R^d) \subseteq \mc P_2(\R^d)$. On $\BW(\R^d)$, the  Wasserstein distance $W_2^2(p_0, p_1)$ between two Gaussians   $p_0=\mathcal{N}(m_0,\Sigma_0)$ and  $p_1=\mathcal{N}(m_1,\Sigma_1)$ admits the following closed form:
 \begin{align}
 &W_2^2(p_0, p_1)= \norm{m_0-m_1}^2 + \mathcal{B}^2(\Sigma_0,\Sigma_1)\,,\label{distance:BW:eq}
    \end{align}
where  $\mathcal{B}^2(\Sigma_0,\Sigma_1)=\tr(\Sigma_0+ \Sigma_1 - 2 \, (\Sigma_0^{\frac{1}{2}} \Sigma_1 \Sigma_0^{\frac{1}{2}})^{\frac{1}{2}})$ is the squared Bures metric~\citep{Bures69}.

Given a target density $\pi \propto \exp(-V)$ on $\R^d$, and with a step size $h > 0$, we may define the iterates of the proximal point algorithm
\begin{align}\label{BJKO}
    p_{k+1,h}
    &\deq \argmin_{p\in \BW(\R^d)}\Bigl\{ \KL(p \mmid \pi) + \frac{1}{2h} \, W_2^2(p, p_{k,h})\Bigr\}\,.
\end{align}
Using~\eqref{distance:BW:eq}, this is an explicit optimization problem involving the mean and covariance matrix of $p$.
Although~\eqref{BJKO} is not solvable in closed form, by letting $h\searrow 0$ we obtain a limiting curve ${(p_t)}_{t\ge 0}$ via $p_t = \lim_{h\searrow 0} p_{\lfloor t/h \rfloor, h}$, which can be interpreted as the Bures--Wasserstein gradient flow of the KL divergence $\KL(\cdot \mmid \pi)$.
This procedure mimics the JKO scheme~\citep{Jordan98} with the additional constraint that the iterates lie in $\BW(\R^d)$, and we therefore call it the Bures--JKO scheme.

\subsection{Geometric approach: the Bures--Wasserstein gradient flow of the KL divergence} \label{pb2}

In the formal sense of Otto described above,  $\BW(\R^d)$ is a submanifold of $\mc P_2(\R^d)$. Moreover, since Gaussians can be parameterized by their mean and covariance, $\BW(\R^d)$ can be identified with the manifold $\R^d\times \mb S_{++}^d$, where $\mb S_{++}^d$ is the cone of symmetric positive definite $d\times d$ matrices. Hence, $\BW(\R^d)$ is a genuine Riemannian manifold in its own right~\citep[see][]{modin2017matrixdecomposition, malmonpis2018bw, Bhatia19},  and gradient flows can be defined using Riemannian geometry~\citep{docarmo1992riemannian}. See Section~\ref{scn:bures_wasserstein} for more details. Since the functional $\mu \mapsto \eu F(\mu)=\KL(\mu \mmid \pi)$ defined over $\mc P_2(\R^d)$ restricts to a functional over $\BW(\R^d)$, we can also consider the gradient flow of $\eu F$ over the Bures{--}Wasserstein space; note that this latter gradient flow is necessarily a curve ${(p_t)}_{t\ge 0}$ such that each $p_t$ is a Gaussian measure.

\subsection{Variational inference via the Bures--Wasserstein gradient flow}\label{scn:gradient_flow_kl_bw}
 
Using either approach, we can prove the following theorem.

\begin{theorem}\label{thm:main}
Let $\pi \propto \exp(-V)$ be the target density on $\R^d$.
Then, the limiting curve ${(p_t)}_{t\ge 0}$ where $p_t=\normal(m_t, \Sigma_t)$ is obtained via the Bures--JKO scheme \eqref{BJKO}, or equivalently, the Bures--Wasserstein gradient flow ${(p_t)}_{t\ge 0}$ of the KL divergence $\KL(\cdot \mmid \pi)$, satisfies S\"arkk\"a's system of ODEs~\eqref{eq:sarkka}.
\end{theorem}
\begin{proof}
The proof using the Bures--JKO scheme is given in Section~\ref{direct:proof:sec} and the proof using Otto calculus is presented in Section~\ref{Proof2}.
\end{proof}

This theorem shows that  S\"arkk\"a's heuristic~\eqref{eq:sarkka} precisely yields the Wasserstein gradient flow of the KL divergence over the submanifold $\BW(\R^d)$.
Equipped with this interpretation, we are now able to obtain information about the asymptotic behavior of the approximation ${(p_t)}_{t\ge 0}$.
Namely, we can hope that it converges to constrained minimizer $\hat \pi = \argmin_{p \in \BW(\R^d)} \KL(p\mmid \pi)$, i.e., precisely the solution to the \vi{} problem~\eqref{eq:vi}. In the next section, we show that this convergence in fact holds as soon as $V$ is convex, and moreover with quantitative rates.

The solution $\hat\pi$ to~\eqref{eq:vi}, and consequently the limit point of S\"arkk\"a's approximation, is well-studied in the variational inference literature~\citep[see,  e.g.,][]{Opper09}, and we recall standard facts about $\hat\pi$ here for completeness.
It is known that $\hat\pi$ satisfies the equations
\begin{align}\label{eq:score_matching}
    \E_{\hat \pi} \nabla V = 0 \qquad\text{and}\qquad \E_{\hat \pi}\nabla^2 V = \hat\Sigma^{-1},
\end{align}
where $\hat\Sigma$ is the covariance matrix of $\hat\pi$ (these equations can also be derived as first-order necessary conditions by setting the Bures{--}Wasserstein gradient derived in Section~\ref{Proof2} to zero).
In particular, it follows from~\eqref{eq:score_matching} that if $\nabla^2 V$ enjoys the bounds $\alpha I \preceq \nabla^2 V \preceq \beta I$ for some $-\infty \le \alpha \le \beta \le \infty$, then any solution $\hat\pi$ to the constrained problem also satisfies $\beta^{-1} \, I \preceq \hat\Sigma \preceq (\alpha\vee 0)^{-1}\, I$.

\subsection{Continuous-time convergence} \label{Convergence}

Besides providing an intuitive interpretation of S\"arkk\"a's heuristic, Theorem~\ref{thm:main} readily yields convergence criteria for the system~\eqref{eq:sarkka} which rest upon general principles for gradient flows. We begin with a key observation.
For a functional $\eu F : \BW(\R^d)\to\R\cup\{\infty\}$ and $\alpha\in\R$, we say that $\eu F$ is \emph{$\alpha$-convex} if for all constant-speed geodesics ${(p_t)}_{t\in [0,1]}$ in $\BW(\R^d)$,
\begin{align*}
    \eu F(p_t)
    &\le (1-t) \, \eu F(p_0) + t \, \eu F(p_1) - \frac{\alpha \, t \, (1-t)}{2} \, W_2^2(p_0, p_1)\,, \qquad t\in [0,1]\,.
\end{align*}

\begin{lemma}\label{lem:strong_cvxty_kl_bw}
    For any $\alpha \in \R$, if $\nabla^2 V \succeq \alpha I$, then $\KL(\cdot \mmid \pi)$ is $\alpha$-convex on $\BW(\R^d)$.
\end{lemma}
\begin{proof}
    The assumption that $\nabla^2 V \succeq \alpha I$ entails that the functional $\KL(\cdot \mmid \pi)$ is $\alpha$-convex on the entire Wasserstein space $(\mc P_2(\R^d), W_2)$ \citep[see, e.g.,][Theorem 17.15]{villani2009ot}. Since $\BW(\R^d)$ is a geodesically convex subset of $\mc P_2(\R^d)$ (see Section~\ref{scn:bures_wasserstein}), then the geodesics in $\BW(\R^d)$ agree with the geodesics in $\mc P_2(\R^d)$, from which it follows that $\KL(\cdot \mmid \pi)$ is $\alpha$-convex on $\BW(\R^d)$.
\end{proof}

Consequently, we obtain the following corollary. Its proof is postponed to Section~\ref{scn:pf_cont_time_guarantee}.

\begin{corollary}\label{COR:CONT_TIME_GUARANTEE}
    Suppose that $\nabla^2 V \succeq \alpha I$ for some $\alpha \in \R$. Then, for any $p_0 \in \BW(\R^d)$, there is a unique solution to the $\BW(\R^d)$ gradient flow of $\KL(\cdot \mmid \pi)$ started at $p_0$. Moreover:

\noindent 1.  If $\alpha > 0$, then for all $t\ge 0$, $
            W_2^2(p_t, \hat \pi)
            \le \exp(-2\alpha t) \, W_2^2(p_0, \hat \pi)
        $.
        
\noindent 2.  If $\alpha > 0$, then for all $t\ge 0$, 
        $
            \KL(p_t \mmid \pi) - \KL(\hat \pi \mmid \pi)
            \le \exp(-2\alpha t) \, \{ \KL(p_0 \mmid \pi) - \KL(\hat \pi \mmid \pi)\}
       $.

\noindent 3. If $\alpha = 0$, then for all $t > 0$,
        $
            \KL(p_t \mmid \pi) - \KL(\hat\pi \mmid \pi)
            \le \frac{1}{2t}\, W_2^2(p_0, \hat\pi)
        $.
\end{corollary}

The assumption that $\nabla^2 V \succeq \alpha I$ for some $\alpha > 0$, i.e., that $\pi$ is \emph{strongly log-concave}, is a standard assumption in the \mcmc{} literature.
Under this same assumption, Corollary~\ref{COR:CONT_TIME_GUARANTEE} yields convergence for the Bures--Wasserstein gradient flow of $\KL(\cdot \mmid \pi)$; however, the flow must first be discretized in time for implementation.
If we assume additionally that the smoothness condition $\nabla^2 V \preceq \beta I$ holds, then a surge of recent research has succeeded in obtaining precise non-asymptotic guarantees for discretized \mcmc{} algorithms. In Section~\ref{scn:bwsgd} below, we will show how to do the same for \vi{}.

\section{Time discretization of the Bures--Wasserstein gradient flow} \label{sec3}

We are now equipped with dual perspectives on a dynamical solution to Gaussian \vi{}:
ODE and gradient flow. Each perspective leads to a different implementation. On the one hand, we discretize the system of ODEs defined in~\eqref{eq:sarkka} using numerical integration. On the other, we discretize the gradient flow using stochastic gradient descent in the Bures{--}Wasserstein space.

\begin{figure}
\centering
\includegraphics[scale=0.4]{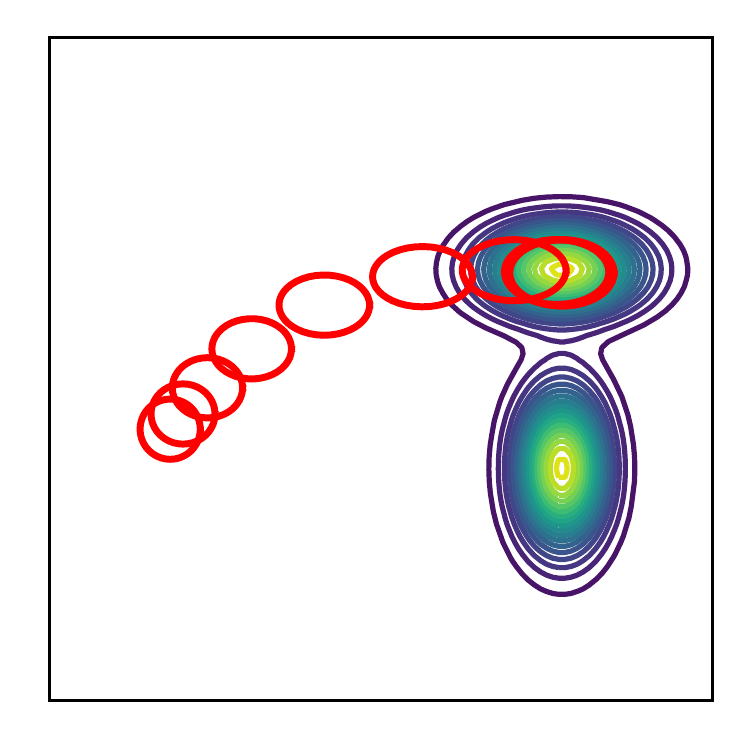}
\includegraphics[scale=0.4]{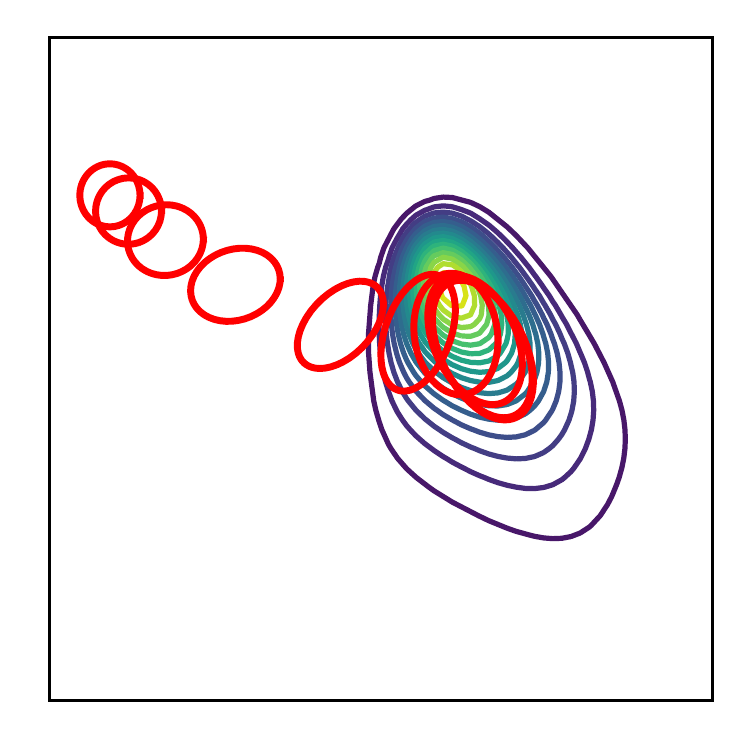}
\includegraphics[scale=0.4]{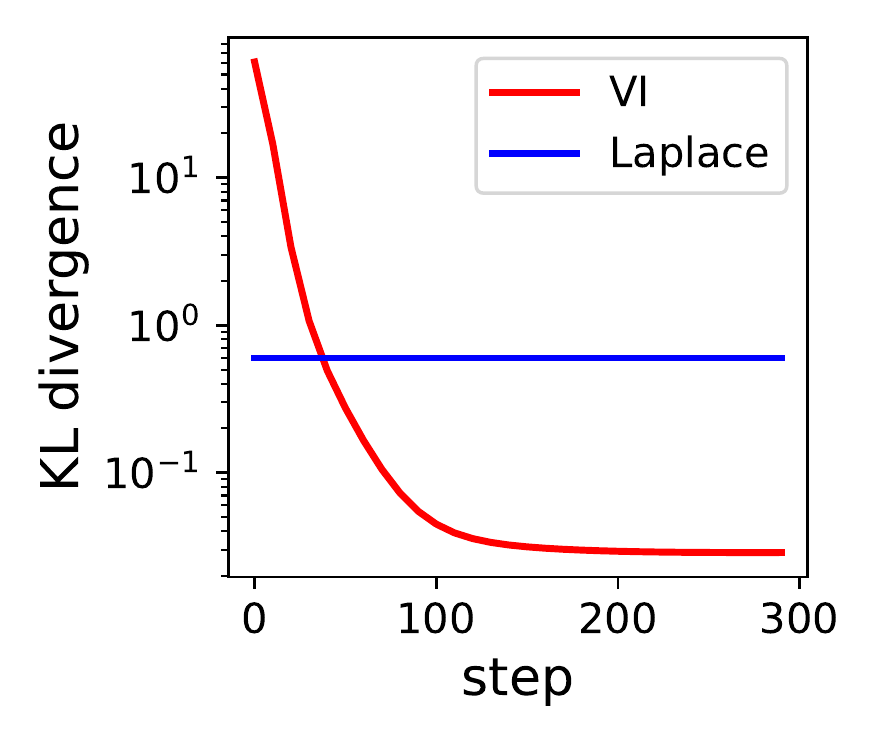}
\includegraphics[scale=0.4]{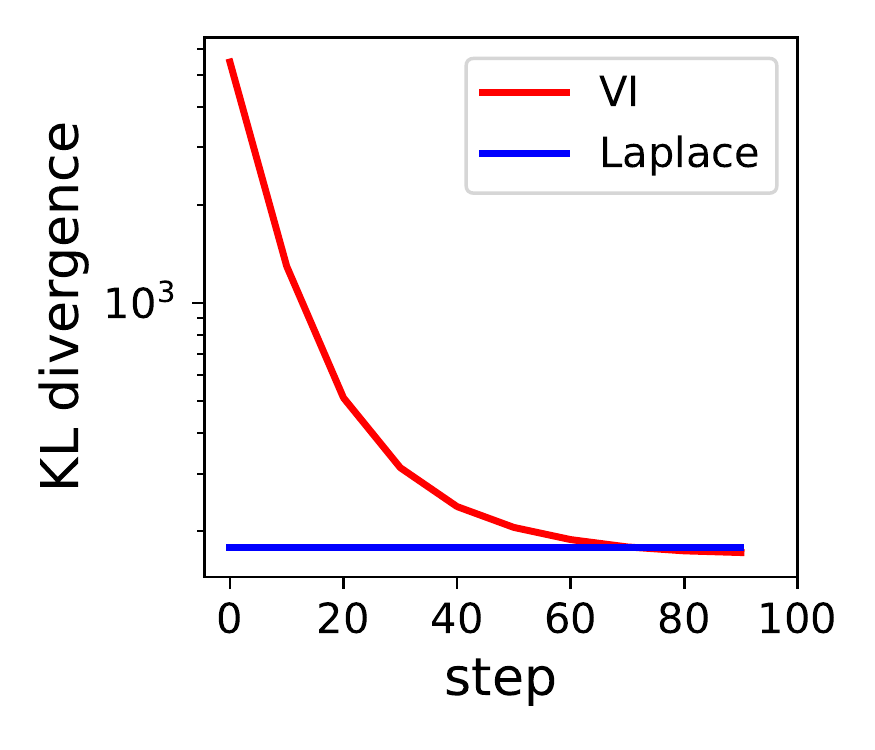}
\caption{Two left plots: approximation of a bimodal target and a logistic target. Two right plots: convergence of the KL in dimension $2$ and $100$ for the logistic target. Our algorithm yields better approximation in KL than the Laplace approximation (see Appendix  \ref{scn:xphd} for details).
}
\label{GPxp0}
\end{figure}

\subsection{Numerical integration of the ODEs}\label{scn:numerical_integration}

The system of ODEs~\eqref{eq:sarkka} can be integrated in time using a classical Runge--Kutta scheme. The expectations under a Gaussian support are approximated by cubature rules used in Kalman filtering~\citep{Arasaratnam09}.
Moreover, a square root version of the ODE is also considered to ensure that covariance matrices remain symmetric and positive. See Appendix~\ref{scn:numerical} for more details. We have tested our method on a bimodal distribution and on a posterior distribution arising from a logistic regression problem. We observe fast convergence as shown in Figure~\ref{GPxp0}.

\subsection{Bures--Wasserstein SGD and theoretical guarantees for VI}\label{scn:bwsgd}

Although the ODE discretization proposed in the preceding section enjoys strong empirical performance, it is unclear how to quantify its impact on the convergence rates established in Corollary~\ref{COR:CONT_TIME_GUARANTEE}.
Therefore, we now propose a stochastic gradient descent algorithm over the Bures--Wasserstein space, for which useful analysis tools have been developed~\citep{CheMauRig20a,altschuleretal2023bwbarycenter}. This approach bypasses the use of the system of ODEs~\eqref{eq:sarkka}, and instead discretizes the Bures--Wasserstein gradient flow directly.
Under the standard assumption of strong log-concavity and log-smoothness,
it leads to an algorithm (Algorithm~\ref{alg:bwsgd}) for approximating $\hat \pi$ with provable convergence guarantees.

\begin{wrapfigure}[10]{r}{0.52\textwidth}
\vspace{-0.5em}
\hfill
\begin{minipage}{0.5\textwidth}
\vspace{-1em}
\begin{algorithm}[H]
\caption{Bures{--}Wasserstein SGD} \label{alg:bwsgd}
\LinesNotNumbered
\KwData{strong convexity parameter $\alpha > 0$; step size $h > 0$; mean $m_0$ and covariance matrix $\Sigma_0$}
\For{$k=1,\dotsc,N$}{
draw a sample $\hat X_k \sim p_k$\;

set $m_{k+1} \gets m_k - h \, \nabla V(\hat X_k)$\;

set $M_k \gets I - h \, (\nabla^2 V(\hat X_k) - \Sigma_k^{-1})$\;

set $\Sigma_k^+ \gets M_k \Sigma_k M_k$\;

set $\Sigma_{k+1} \gets \clip^{1/\alpha} \Sigma_k^+$\;
}
\end{algorithm}
\end{minipage}
\end{wrapfigure}

Algorithm~\ref{alg:bwsgd} maintains a sequence of Gaussian distributions ${(p_k)}_{k\in\N}$; here $(m_k, \Sigma_k)$ denote the mean vector and covariance matrix at iteration~$k$ (see Section~\ref{Proof3} for a derivation of the algorithm as SGD in the Bures{--}Wasserstein space).
The clipping operator $\clip^\tau$, which is introduced purely for the purpose of theoretical analysis, simply truncates the eigenvalues from above; see Section~\ref{Proof3}.
Our theoretical result for \vi{} is given as the following theorem, whose proof is deferred to Section~\ref{Proof3}.

\begin{theorem} \label{THEOREM2}
    Assume that $0 \prec \alpha I \preceq \nabla^2 V \preceq I$.
    Also, assume that $h \le \frac{\alpha^2}{60}$ and that we initialize Algorithm~\ref{alg:bwsgd} at a matrix satisfying $\frac{\alpha}{9} \, I \preceq \Sigma_{\mu_0} \preceq \frac{1}{\alpha} \, I$.
    Then, for all $k\in\N$,
    \begin{align*}
        \E W_2^2(p_k, \hat\pi)
        &\le \exp(-\alpha kh) \, W_2^2(p_0, \hat \pi) + \frac{36dh}{\alpha^2}\,.
    \end{align*}
    In particular, we obtain $\E W_2^2(p_k,\hat \pi) \le \varepsilon^2$ provided we set $h \asymp \frac{\alpha^2 \varepsilon^2}{d}$ and the number of iterations to be $k \gtrsim \frac{d}{\alpha^3 \varepsilon^2} \log(W_2(p_0,\hat \pi)/\varepsilon)$.
\end{theorem}

The upper bound $\nabla^2 V \preceq I$ is notationally convenient for our proof but not necessary; in any case, any strongly log-concave and log-smooth density $\pi$ can be rescaled so that the assumption holds.

Theorem~\ref{THEOREM2} is similar in flavor to modern results for \mcmc{}, both in terms of the assumptions (Hessian bounds and query access to the derivatives\footnote{A notable downside of Algorithm~\ref{alg:bwsgd} is the requirement of a Hessian oracle for $V$, which results in a higher per-iteration cost than typical \mcmc{} samplers.} of $V$) and the conclusion (a non-asymptotic polynomial-time algorithmic guarantee). We hope that such an encouraging result for 
\vi{} will prompt more theoretical studies aimed at closing the gap between the two approaches.

\section{Variational inference with mixtures of Gaussians}\label{scn:mixtures}

Thus far, we have shown that the tractability of Gaussians can be readily exploited in the context of Bures--Wasserstein gradient flows and translated into useful results for variational inference. Nevertheless, these results are limited by the lack of expressivity of Gaussians, namely their inability to capture complex features such as multimodality and, more generally, heterogeneity. To overcome this limitation, mixtures of Gaussians arise as a natural and powerful alternative; indeed, universal approximation of arbitrary probability measures by mixtures of Gaussians is well-known~\citep[see, e.g.,][]{delon2020mixture}. As we show next, the space of mixtures of Gaussians can also be equipped with a Wasserstein structure which gives rise to implementable gradient flows.

\subsection{Geometry of the space of mixtures of Gaussians}

We begin with the key observation already made by~\cite{chen2019mixture}, that any mixture of Gaussians can be canonically identified with a probability distribution (the mixing distribution) over the parameter space $\Theta = \R^d\times \mb S_{++}^d$ (the space of means and covariance matrices). Explicitly a probability measure $\mu \in \mc P(\Theta)$  corresponds to a Gaussian mixture as follows:
\begin{align}
    \mu \qquad\qquad \leftrightarrow \qquad\qquad \p_\mu \deq \int p_\theta \, \D \mu(\theta)\,,\label{infinite_mixture:eq}
\end{align}
where $p_\theta$ is the Gaussian distribution with parameters $\theta \in \Theta$.
Equivalently, $\mu$ can be thought of as a probability measure over $\BW(\R^d)$, and hence the space of Gaussian mixtures on $\R^d$ can be identified with the Wasserstein space $\mc P_2(\BW(\R^d))$ over the Bures{--}Wasserstein space which is endowed with the distance \eqref{distance:BW:eq} between Gaussian measures. Indeed, the theory of optimal transport can be developed with any Riemannian manifold (rather than $\R^d$) as the base space~\citep{villani2009ot}. As before, the space $\mc P_2(\BW(\R^d))$ is endowed with a formal Riemannian structure, which respects the geometry of the base space $\BW(\R^d)$, and we can consider Wasserstein gradient flows over $\mc P_2(\BW(\R^d))$.

Note that this framework encompasses both discrete mixtures of Gaussians (when $\mu$ is a discrete measure) and continuous mixtures of Gaussians. In the case when the mixing distribution $\mu$ is discrete, the geometry of $\mc P_2(\BW(\R^d))$ was studied by~\citet{chen2019mixture,delon2020mixture}. An important insight of our work, however, is that it is fruitful to consider the full space $\mc P_2(\BW(\R^d))$ for deriving gradient flows, even if we eventually develop algorithms which propagate a finite number of mixture components.

\subsection{Gradient flow of the KL divergence and particle discretization}\label{scn:mixture_gf_and_discretization}

We consider the gradient flow of the KL divergence functional
\begin{align}\label{eq:kl_mixture}
    \mu \mapsto \eu F(\mu)
    &\deq \KL(\p_\mu \mmid \pi)
\end{align}
over the space $\mc P_2(\BW(\R^d))$. The proof of the following theorem is given in Section~\ref{scn:pf_mixtures}.

\begin{theorem}\label{THM:MIXTURE_GF}
    The gradient flow ${(\mu_t)}_{t\ge 0}$ of the functional $\eu F$ defined in~\eqref{eq:kl_mixture} over $\mc P_2(\BW(\R^d))$ can be described as follows. Let $\theta_0 = (m_0, \Sigma_0) \sim \mu_0$, and let $\theta_t = (m_t,\Sigma_t)$ evolve according to the   ODE
    \begin{align}\label{eq:mixture_ode}
        \boxed{\begin{aligned}
        \dot m_t
        &= -\E \nabla \ln \frac{\p_{\mu_t}}{\pi}(Y_t) \\
        \dot \Sigma_t
        &= -\E \nabla^2 \ln \frac{\p_{\mu_t}}{\pi}(Y_t) \, \Sigma_t  - \Sigma_t \E \nabla^2 \ln \frac{\p_{\mu_t}}{\pi}(Y_t)
        \end{aligned}}
    \end{align}
  where $Y_t \sim \normal(m_t, \Sigma_t)$. Then $\theta_t \sim \mu_t$.
\end{theorem}

The gradient flow in Theorem~\ref{THM:MIXTURE_GF} describes the evolution of a particle $\theta_t$ which describes the parameters of a Gaussian measure, hence the name \emph{Gaussian particle}. 
The intuition behind this evolution is as follows. Suppose we draw infinitely many initial particles (each being a Gaussian) from $\mu_0$. By evolving all those  particles through  \eqref{eq:mixture_ode}, which interact with each other via the term $\p_{\mu_t}$,  they tend to aggregate in some parts of the space of Gaussian parameters and spread out in others. This distribution of Gaussian particles is precisely the mixing measure $\mu_t$, which, in turn, corresponds to a Gaussian mixture. Since an infinite number of Gaussian particles  is impractical, consider initializing this evolution at a finitely supported distribution $\mu_0$, thus  corresponding to a more familiar Gaussian mixture model with a finite number of components:
\begin{align*}
    \mu_0
    &= \frac{1}{N} \sum_{i=1}^N \delta_{\theta_0^{(i)}}
    = \frac{1}{N} \sum_{i=1}^N \delta_{(m_0^{(i)}, \Sigma_0^{(i)} )} \qquad \leftrightarrow \qquad \p_{\mu_0} \deq  \frac{1}{N} \sum_{i=1}^N p_{(m_0^{(i)}, \Sigma_0^{(i)} )}\,.
\end{align*}
Interestingly, it can be readily checked that the system of ODEs~\eqref{eq:mixture_ode} thus initialized maintains a finite mixture distribution:
$$
\mu_t= \frac{1}{N} \sum_{i=1}^N \delta_{\theta_t^{(i)}}
    = \frac{1}{N} \sum_{i=1}^N \delta_{(m_t^{(i)}, \Sigma_t^{(i)} )}\,,
$$
where the parameters $\theta_t^{(i)}=(m_t^{(i)}, \Sigma_t^{(i)} )$ evolve according to the following interacting particle system, for   $i\in [N]$
\begin{align}
    \dot m_t^{(i)}
    &= -\E \nabla \ln \frac{\p_{\mu_t}}{\pi}(Y_t^{(i)})\,, \label{mean:particle:eq}\\
    \dot \Sigma_t^{(i)}
    &= -\E\nabla^2 \ln \frac{\p_{\mu_t}}{\pi} (Y_t^{(i)})\, \Sigma_t^{(i)} - \Sigma_t^{(i)} \E\nabla^2 \ln \frac{\p_{\mu_t}}{\pi}(Y_t^{(i)})\,,
\label{cov:particle:eq}\end{align}
where $Y_t^{(i)}\sim p_{\theta_t^{(i)}}$. This finite system of particles can now be implemented using the same numerical tools as for Gaussian~\vi{}, see Section~\ref{scn:xp-gmmvi}.
Note that due to this property of the dynamics, we can hope at best to converge to the best mixture of $N$ Gaussians approximating $\pi$, but this approximation error is expected to vanish as $N\to\infty$.
Also, similarly to~\eqref{eq:sarkka}, it is possible to write down Hessian-free updates using integration by parts, see Appendix~\ref{gmm-vi:proof:sec}.
   
The above system of particles may also be derived using a proximal point method similar to the Bures--JKO scheme, see Section~\ref{gmm-vi:proof:sec}. Indeed, infinitesimally, it has the variational interpretation
\begin{align*}
    (\theta_{t+h}^{(1)},\dotsc,\theta_{t+h}^{(N)})
    &\approx \argmin_{\theta^{(1)},\dotsc,\theta^{(N)} \in \Theta}\biggl\{\KL\Bigl( \frac{1}{N} \sum_{i=1}^N p_{\theta^{(i)}} \Bigm\Vert \pi\Bigr) + \frac{1}{2Nh} \sum_{i=1}^N W_2^2(p_{\theta^{(i)}}, p_{\theta_t^{(i)}})\biggr\}\,.
\end{align*}

Reassuringly,  Equations \eqref{mean:particle:eq}-\eqref{cov:particle:eq} reduce to~\eqref{eq:sarkka} when $\mu_0=\delta_{(m_0, \Sigma_0)}$ is a point mass,   indicating that the theorem provides a natural extension of our previous results. However, although the model~\eqref{infinite_mixture:eq} is substantially more expressive than the Gaussian \vi{} considered in Section~\ref{sec2}, it has the downside that we lose many of the theoretical guarantees. For example, even when $V$ is convex, the objective functional $\eu F$ considered here need not be convex; see Section~\ref{scn:lack_of_cvxty}.
We nevertheless validate the practical utility of our approach in experiments (see Figure~\ref{xpgmm} and  Section~\ref{scn:xp-gmmvi}).

Unlike typical interacting particle systems which arise from discretizations of Wasserstein gradient flows, at each time $t$, the distribution $\p_{\mu_t}$ is continuous.
This extension provides considerably more flexibility---from a mixture of point masses to a mixture of Gaussians---compared to
interacting particle-based algorithms hitherto considered for either sampling~\citep{liuwang2016svgd, liu2017svgdgf, duncan2019geometrysvgd, chewietal2020svgd}, or solving partial differential equations~\citep{carrilloetal2011aggregation, carilloetal2012confinement, bonaschietal2015nonlocal, craigbertozzi2016blob, carrillocraigpatacchini2019blob, craigetal2022blob}. 

\begin{figure}[!h]
\centering
\includegraphics[scale=0.45]{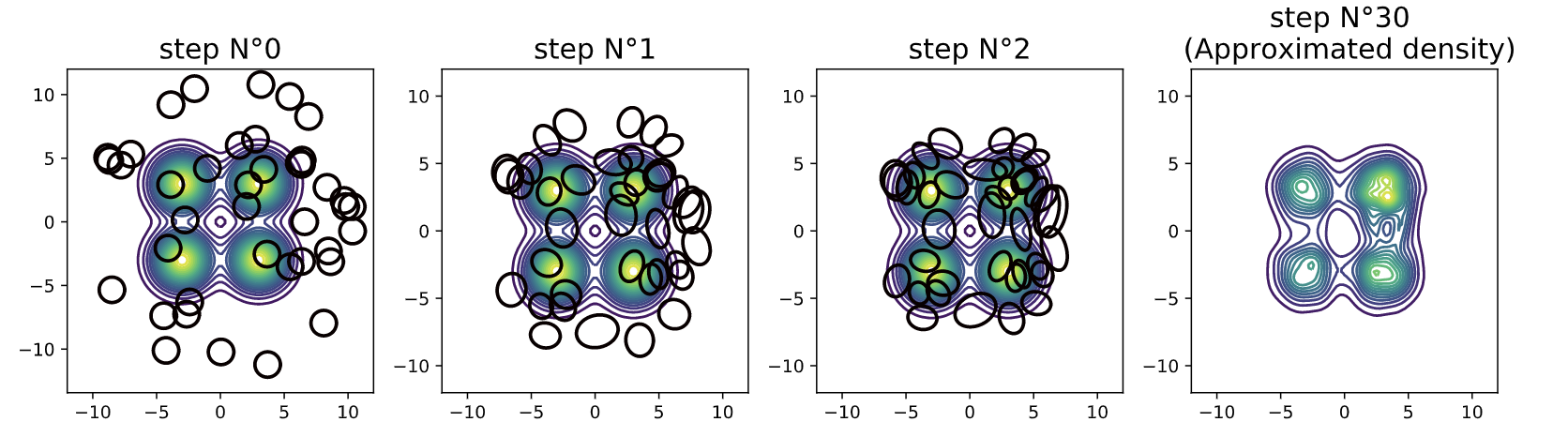}
\caption{Approximation of a Gaussian mixture target $\pi$ with $40$ Gaussian particles. The particles are represented by their covariance ellipsoids shown at Steps $0$, $1$, and $2$. The right figure shows the final step with the approximated density in contour-lines. More figures are available in Appendix~\ref{scn:xp-gmmvi}.}
\label{xpgmm}
\end{figure}

\section{Conclusion} \label{sec5}

Using the powerful theory of Wasserstein gradient flows, we derived  new algorithms for \vi{} using either Gaussians or mixtures of Gaussians as approximating distributions. The consequences are twofold. On the one hand, strong convergence guarantees under classical conditions contribute markedly to closing the theoretical gap between \mcmc{} and Gaussian \vi{}. On the other hand, discretization of the Wasserstein gradient flow for mixtures of Gaussians yields a new \emph{Gaussian particle method} for time discretization which, unlike classical particle methods, maintains a continuous probability distribution at each time.

We conclude by briefly listing some possible directions for future study. For Gaussian variational inference, our theoretical result (Theorem~\ref{THEOREM2}) can be strengthened by weakening the assumption that~$\pi$ is strongly log-concave, or by developing algorithms which do not require Hessian information for $V$. For mixtures of Gaussians, it is desirable to design a principled algorithm which also allows for the mixture weights to be updated.

Towards the latter question, in Section~\ref{scn:wfr} we derive the gradient flow of the KL divergence with respect to the Wasserstein{--}Fisher{--}Rao geometry~\citep{lieromielkesavare2016wfr1, chizatetal2018wfr, lieromielkesavare2018wfr2}, which yields an interacting system of Gaussian particles with changing weights. The equations are given as follows: at each time $t$, the mixing measure is the discrete measure
\begin{align*}
    \mu_t
    &= \sum_{i=1}^N w_t^{(i)} \delta_{(m_t^{(i)}, \Sigma_t^{(i)})}\,.
\end{align*}
Let $Y_t^{(i)} \sim \mc N(m_t^{(i)}, \Sigma_t^{(i)})$, and let $r_t^{(i)} = \sqrt{w_t^{(i)}}$.
Then, the system of ODEs is given by
\begin{align*}
    \dot m_t^{(i)}
    &= -\E \nabla \ln \frac{\p_{\mu_t}}{\pi}(Y_t^{(i)})\,, \\
    \dot \Sigma_t^{(i)}
    &= -\E\nabla^2 \ln \frac{\p_{\mu_t}}{\pi} (Y_t^{(i)})\, \Sigma_t^{(i)} - \Sigma_t^{(i)} \E\nabla^2 \ln \frac{\p_{\mu_t}}{\pi}(Y_t^{(i)})\,, \\
    \dot r_t^{(i)}
    &= -\Bigl(\E\ln \frac{\p_{\mu_t}}{\pi}(Y_t^{(i)}) - \frac{1}{N} \sum_{j=1}^N \E\ln \frac{\p_{\mu_t}}{\pi}(Y_t^{(j)})\Bigr) \, r_t^{(i)}\,.
\end{align*}
We have implemented these equations and their empirical performance is encouraging. However, a fuller investigation of algorithms for \vi{} with changing weights is beyond the scope of this work and we leave it for future research.

\medskip{}

\noindent{}Code for the experiments is available at \url{https://github.com/marc-h-lambert/W-VI}.

\acks{We thank Yian Ma for helpful discussions, as well as anonymous reviewers for useful references and suggestions.
ML acknowledges support from the French Defence procurement agency (DGA). SC is supported by the Department of Defense (DoD) through the National Defense Science \& Engineering Graduate Fellowship (NDSEG) Program. FB and ML acknowledge support from the French government under the management of the Agence Nationale de la Recherche as part of the “Investissements d’avenir” program,
reference ANR-19-P3IA-0001 (PRAIRIE 3IA Institute), as well as from the European Research Council (grant SEQUOIA 724063). PR is supported by NSF grants IIS-1838071, DMS-2022448, and CCF-2106377.}

\newpage
\appendix

\section{Proofs via the Bures--JKO scheme}\label{scn:bures_jko_pfs}

\subsection{Proof of Theorem~\ref{thm:main}}\label{direct:proof:sec}

Given a Gaussian distribution at time $t$ written $p_t=\mathcal{N}(m_t,\Sigma_t)$ and a target distribution $\pi$, we seek the solution $p$ at time $t+h$ of the following JKO scheme, where $p=\mathcal{N}(m,\Sigma)$ is constrained to lie on the space $\BW(\R^d)$ of Gaussians equipped with the Wasserstein distance: 
\begin{align}
    & \min_{p \in \BW(\R^d)} L(p) = \KL(p \mmid \pi) + \frac{1}{2h} \, W_2^2(p, p_t)\,.\label{JKO}
\end{align}
Using the expression for the Wasserstein distance $W_2^2(p_t, p)$ given in~\eqref{distance:BW:eq} 
it is equivalent to finding the Gaussian parameters which solve:
\begin{align}
    & \min_{m, \Sigma} L(m,\Sigma) = \KL(\mathcal{N}(m,\Sigma) \mmid \pi) +  \frac{1}{2h} \,  \norm{m_t-m}^2 + \frac{1}{2h} \,  \mathcal{B}^2(\Sigma_t,\Sigma)\, .\label{alternate:L:eq}
\end{align}
We  first compute the critical points of $L$  and then take the limit as $h \searrow 0$ to get the desired differential equations (ODEs) for the parameters. This boils down to computing the Wasserstein gradient flow of $\KL(\cdot \mmid \pi)$ over the Bures--Wasserstein manifold $\BW(\R^d)$. 

The left KL divergence is a sum of two terms $\KL(p \mmid \pi)=H(p)-\E_p[ \ln \pi]$, where $H(p)$ is the negative entropy of a Gaussian. It satisfies
\begin{align*}
    \nabla_m H(p)=0\qquad\text{and}\qquad \nabla_\Sigma H(p)=\nabla_\Sigma \bigl(-\frac{1}{2}  \ln \det \Sigma\bigr)= -\frac{1}{2} \, \Sigma^{-1}\,.
\end{align*}

To alleviate notation, for any function $f$ we both let $x$ denote its argument  and  $\mathbb{E}_p[ f(x)]$ denote expectation over  $x\sim p = \mathcal{N}(m,\Sigma)$   throughout the present proof, depending on the context. 

The gradient of the left KL divergence with respect to $m$ is given by: 
\begin{align*}
&\nabla_m \KL(p \mmid \pi) =-\nabla_m \mathbb{E}_p[\ln \pi(x)]=-\E_p[\nabla_x \ln \pi(x)], 
\end{align*}
where we have used integration by parts (assuming $\pi$ is continuously differentiable) and the property of Gaussian densities $\nabla_m \mathcal{N}(x \mid m,\Sigma)=-\nabla_x \mathcal{N}(x\mid m,\Sigma)$ to get a derivative with respect to $x$.
The critical point of $L$ given by \eqref{alternate:L:eq} w.r.t.\ the mean parameter $m$ thus writes:     
\begin{align}\label{eq:jko_mean_opt}
\nabla_m L(m,\Sigma)=\frac{1}{h}\,(m - m_t )-  \E_p[\nabla_x \ln \pi(x)]=0\,.
\end{align}
 Taking the limit as $h \searrow 0$, we find that $m_t$ must satisfy the following ODE:
 \begin{align*}
\dot{m_t}=\mathbb{E}_{p_t}[\nabla_x \ln \pi(x)]=-\E_{p_t}[\nabla_x V(x)]\,,
\end{align*}
where we recall that $\pi \propto \exp (-V)$. This recovers the first line of \eqref{eq:sarkka}.

The gradient of the left KL divergence with respect to $\Sigma$ is given by: 
\begin{align*}
&\nabla_\Sigma \KL(p \mmid \pi) =-\frac{1}{2}  \,\Sigma^{-1}-\nabla_\Sigma \mathbb{E}_p[\ln \pi(x)]=-\frac{1}{2} \,\Sigma^{-1}- \frac{1}{2} \E_p[\nabla_x^2 \ln\pi(x)]\,,
\end{align*}
 where we have used two integrations by parts (supposing $\pi$ is twice continuously differentiable) and the property of Gaussian densities $\nabla_\Sigma \mathcal{N}(x\mid m,\Sigma)=\frac{1}{2}\, \nabla^2_x \mathcal{N}(x\mid m,\Sigma)$ to let  a Hessian w.r.t.\ $x$ appear. The Bures derivative is given by~\citep[see][]{Bhatia19}:
\begin{align*}
&\nabla_\Sigma \mathcal{B}^2(\Sigma_t,\Sigma) = I - T^{\Sigma,\Sigma_t}\,, 
\end{align*}
where $T^{A,B}$ is the optimal transport map from $\mc N(0,A)$ to $\mc N(0, B)$, with the explicit expression $T^{A,B} = A^{-\frac{1}{2}} (A^{\frac{1}{2}} B A^{\frac{1}{2}})^{\frac{1}{2}} A^{-\frac{1}{2}}=(T^{B,A})^{-1}$.
The gradient of the variational loss $L$ in~\eqref{alternate:L:eq}  is thus:  
\begin{align*}
&\nabla_\Sigma L(m,\Sigma)=\frac{1}{2h}\, I - \frac{1}{2h}\, T^{\Sigma,\Sigma_t} -\frac{1}{2}  \, \Sigma^{-1}-  \frac{1}{2} \E_p[\nabla_x^2 \ln \pi(x)]\,.
\end{align*}

Zeroing this equation gives:
 \begin{align}\label{eq:jko_cov_opt}
&  {I} =T^{\Sigma,\Sigma_t} +  h\, \Sigma^{-1}+ h \E_p[\nabla_x^2 \ln \pi(x)]\,.
\end{align}

Multiplying by $\Sigma$ on the left, as well as  on the right, yields the two following equations:
 \begin{align}
& \Sigma = \Sigma T^{\Sigma,\Sigma_t}+ h  I + h \,\Sigma \E_p[\nabla_x^2 \ln \pi(x)]\,, \\
& \Sigma = T^{\Sigma,\Sigma_t} \Sigma + h I + h \E_p[\nabla_x^2 \ln \pi(x)]\, \Sigma\,.
\end{align}

Adding them we obtain the  symmetrized form: 
 \begin{align}
& \Sigma = \frac{1}{2} \, T^{\Sigma,\Sigma_t} \Sigma + \frac{1}{2} \,\Sigma T^{\Sigma,\Sigma_t} + h I +  \frac{1}{2} \, h\, \Sigma \E_p[\nabla_x^2 \ln \pi(x)] +   \frac{1}{2} \, h \E_p[\nabla_x^2 \ln \pi(x)] \,\Sigma\,. \label{eqDerivP}
\end{align}

Let us denote  $T^{\Sigma,\Sigma_t}=T(\Sigma)$. Since $T(\Sigma)$ pushes forward $\Sigma$ to $\Sigma_t$, it follows that $T(\Sigma) \, \Sigma \, T(\Sigma) = \Sigma_t$ (which can be checked directly from the expression for $T^{\Sigma,\Sigma_t}$).
The first variation of this equality w.r.t.\ $\Sigma$ at $I$ gives
\begin{align}\label{eq:linearization_dt}
    dT \,\Sigma_t +d\Sigma+\Sigma_t \, dT=0\,.
\end{align}

Let us now term $\Sigma=\Sigma_{t+h}$ the solution to \eqref{eqDerivP}. Up to the first order in $h$ we have $ \Sigma_{t+h}=\Sigma_t+d\Sigma=\Sigma_t+h\dot\Sigma_t$. Let $dT$ denote the corresponding first variation of $T$, that is,   $T(\Sigma_{t+h})=T(\Sigma_t)+dT=I+dT$ up to the first order in $h$. Substituting into \eqref{eqDerivP},   using the previously found relation~\eqref{eq:linearization_dt}, dividing by $h$ and letting $h\searrow 0$ , we finally  obtain the desired ODE:
 \begin{align}
\dot \Sigma_t
&=2I + \Sigma_t \E_{p_t}[\nabla_x^2 \ln \pi(x)] +   \E_{p_t}[\nabla_x^2 \ln \pi(x)]\, \Sigma_t\\
&=2I + \E_{p_t}[\nabla_x \ln \pi(x) \otimes (x-m_t)] +   \E_{p_t}[(x-m_t) \otimes \nabla_x \ln \pi(x)]\,,
\end{align}
where the relation $\mathbb{E}_p[\nabla^2_x \ln \pi(x)]\, \Sigma=\mathbb{E}_p[\nabla_x \ln \pi(x) \otimes (x-m)]$  comes from Gaussian integration by parts and yields a Hessian-free form. Letting $\pi \propto \exp (-V)$ yields the second line of \eqref{eq:sarkka}.

\paragraph{Interpretation in terms of Wasserstein gradient flows.} Let $T_{t+h \to t}$ denote the optimal transport map from $p_{t+h}$ to $p_t$, so that $T_{t+h \to t} = m_t + T^{\Sigma_{t+h}, \Sigma_t} \, (x-m_{t+h})$. Combining the equations~\eqref{eq:jko_mean_opt} and~\eqref{eq:jko_cov_opt}, it reads
\begin{align*}
    \frac{T_{t+h \to t}(x) - x}{h}
    &= \frac{1}{h} \, \{m_t - m_{t+h} + (T^{\Sigma_{t+h}, \Sigma_t} - I) \, (x - m_{t+h})\} \\
    &= \E_{p_{t+h}} \nabla V + (\E_{p_{t+h}} \nabla^2 V - \Sigma_{t+h}^{-1}) \, (x - m_{t+h})\,.
\end{align*}
In Section~\ref{Proof2}, this equality will be written
\begin{align}\label{eq:interpretation_as_gf}
    \frac{T_{t+h \to t} - {\id}}{h}
    &= [\nabla_{\BW} \KL(\cdot \mmid \pi)](p_{t+h})\,,
\end{align}
where $\frac{1}{h} \, (T_{t+h\to t} - {\id})$ and $[\nabla_{\BW} \KL(\cdot \mmid \pi)](p_{t+h})$ are the Bures--Wasserstein gradients of the functionals $-\frac{1}{2h} \, W_2^2(\cdot, p_t)$ and $\KL(\cdot \mmid \pi)$ at $p_{t+h}$ respectively. The equation~\eqref{eq:interpretation_as_gf} is a first-order optimality condition for the Bures--JKO scheme~\eqref{BJKO} and mimics the known optimality condition for the original JKO scheme, see~\cite[equation (8.4)]{santambrogio2015ot}.

The quantity $\frac{1}{h} \, (T_{t+h\to t} - {\id})$ is a difference quotient which measures the infinitesimal displacement of a particle traveling along the gradient flow.
As $h \searrow 0$, we will interpret this quantity as $-v_t$, the negative of the \emph{tangent vector} to the curve at time $t$ (the negative sign appears because $T_{t+h \to t}$ is the transport map \emph{backwards} in time). Hence, the equation~\eqref{eq:interpretation_as_gf} states that as $h\searrow 0$, the tangent vector to the curve ${(p_t)}_{t\ge 0}$ is the negative Bures--Wasserstein gradient of the KL divergence, which is the definition of a gradient flow.

From this perspective, the computation of the linearization in~\eqref{eq:linearization_dt} is equivalent to computing the tangent vector to the
 Wasserstein geodesic, which is given in~\eqref{eq:evolution_cov}.
 
\subsection{Extension to mixtures of Gaussians} \label{gmm-vi:proof:sec}

We now consider a finite Gaussian mixture model $p=\frac{1}{N} \sum_{i=1}^N p_{\theta^{(i)}}$ where $\theta^{(i)}=(m^{(i)},\Sigma^{(i)})$. We consider the following  variational problem:
 \begin{align*}
    & \min_{\theta^{(1)},\dotsc,\theta^{(N)} \in \Theta} \frac{1}{2Nh} \sum_{i=1}^N W_2^2(p_{\theta^{(i)}}, p_{\theta_t^{(i)}})+\KL\Bigl( \frac{1}{N} \sum_{i=1}^N p_{\theta^{(i)}} \Bigm\Vert \pi\Bigr)\,,
\end{align*}
where, as before, $W_2$ is the Wasserstein distance between two Gaussians distribution: 
 \begin{align*}
 &W_2^2(p_{\theta^{(i)}}, p_{\theta_t^{(i)}})= \norm{m^{(i)}-m^{(i)}_t}^2 +  \mathcal{B}^2(\Sigma^{(i)},\Sigma^{(i)}_t)\,.
 \end{align*}
 
 The KL divergence is now written
 \begin{align*}
 &\KL\Bigl( \frac{1}{N} \sum_{i=1}^N p_{\theta^{(i)}} \Bigm\Vert \pi\Bigr)= \frac{1}{N} \sum_{i=1}^N  \int  p_{\theta^{(i)}} \ln p - \frac{1}{N} \sum_{i=1}^N    \int  p_{\theta^{(i)}}  \ln \pi\,.
 \end{align*}
 For $k\in [N]$, the derivative of this divergence with respect to $m^{(k)}$ gives:
 \begin{align*}
 \nabla_{m^{(k)}} \KL(p \mmid \pi)&=  \int  \frac{1}{N} \, \nabla_{m^{(k)}} p_{\theta^{(k)}}  \ln p + \int p \, \nabla_{m^{(k)}}  \ln p -   \int   \frac{1}{N} \,  \nabla_{m^{(k)}}  p_{\theta^{(k)}}  \ln \pi\\
 &=  \int   \frac{1}{N} \, p_{\theta^{(k)}} \nabla_x \ln \frac{p}{\pi}\,,
 \end{align*}
 where we have used the same integration by parts as in the Section \ref{direct:proof:sec}, i.e., $\int p_{\theta^{(k)}} \,  \nabla_{m^{(k)}}  \ln p=\int  p_{\theta^{(k)}} \, \nabla_x  \ln p$, and the Fisher score property $\int p \, \nabla_{m^{(k)}}  \ln p =0$. Mimicking Section \ref{direct:proof:sec}, see \eqref{eq:jko_mean_opt}, we obtain in  the limit $h \searrow 0$ 
 \begin{align*}
 &\dot m^{(k)}=-\mathbb{E}_{p_{\theta^{(k)}} }\bigl[\nabla_x  \ln \frac{p}{\pi}\bigr]\,,
 \end{align*}
 which is the desired equation \eqref{mean:particle:eq}. 
 
 The derivative of the $\KL$ divergence with respect to $\Sigma^{(k)}$ gives:
 \begin{align*}
 \nabla_{\Sigma^{(k)}} \KL(p \mmid \pi)
 &= \int \frac{1}{N}\, \nabla_{\Sigma^{(k)}} p_{\theta^{(k)}}\ln p + \int p \, \nabla_{\Sigma^{(k)}} \ln p  -    \int  \frac{1}{N}\,\nabla_{\Sigma^{(k)}}  p_{\theta^{(k)}} \ln \pi\\
 &= \frac{1}{N}\,\Bigl(\frac{1}{2} \int  p_{\theta^{(k)}} \,\nabla^2_x  \ln p - \frac{1}{2}   \int   p_{\theta^{(k)}}\,\nabla^2_x \ln \pi\Bigr)=  \frac{1}{2N} \E_{p_{\theta^{(k)}}}\bigl[ \nabla^2_x  \ln \frac{p}{\pi}\bigr]\,,
 \end{align*}
 where we have used a double integration by parts $\int \nabla_{\Sigma^{(k)}} p_{\theta^{(k)}}\ln p=  \frac{1}{2}   \int  p_{\theta^{(k)}}\, \nabla^2_x  \ln p$ as in Section~\ref{direct:proof:sec} and the Fisher score property $\int p \, \nabla_{\Sigma^{(k)}}  \ln p =0$.

Using the  Bures derivative, the critical points of the variational loss with respect to $\Sigma_k$ satisfy: 
 \begin{align*}
 &\frac{1}{2Nh}\,( I -  T^{\Sigma^{(k)},\Sigma^{(k)}_t})+  \frac{1}{2N}  \E_{p_{\theta^{(k)}} }\bigl[ \nabla^2_x  \ln \frac{p}{\pi}\bigr]=0\,.
 \end{align*}
 Multiplying on the left and  on the right by $\Sigma^{(k)}$ and taking the average as in Section \ref{direct:proof:sec}, we find:
 \begin{align*}
 &\frac{1}{ h}\,\bigl(\Sigma^{(k)}-  \frac{1}{2}\,(\Sigma^{(k)} T^{\Sigma^{(k)},\Sigma^{(k)}_t}+T^{\Sigma^{(k)},\Sigma^{(k)}_t}\Sigma^{(k)})\bigr)\\
 &\qquad =-\frac{1}{2} \,\bigl( \E_{p_{\theta^{(k)}} }\bigl[ \nabla^2_x  \ln \frac{p}{\pi}\bigr]\,\Sigma^{(k)}+\Sigma^{(k)} \E_{p_{\theta^{(k)}} }\bigl[ \nabla^2_x  \ln \frac{p}{\pi}\bigr]\bigr)\,.
 \end{align*}
 We can now use the   first-order approximation  $\frac{1}{2}\,(\Sigma^{(k)} T^{\Sigma^{(k)},\Sigma^{(k)}_t}+T^{\Sigma^{(k)},\Sigma^{(k)}_t}\Sigma^{(k)}) \approx \Sigma^{(k)} - \frac{h}{2}\, \dot\Sigma^{(k)}$   shown in Section~\ref{direct:proof:sec}  to obtain:
 \begin{align*}
 &\dot\Sigma^{(k)} =- \E_{p_{\theta^{(k)}} }\bigl[\nabla^2_x  \ln \frac{p}{\pi}\bigr]\,\Sigma^{(k)} - \Sigma^{(k)} \E_{p_{\theta^{(k)}} }\bigl[ \nabla^2_x  \ln \frac{p}{\pi}\bigr]\,.
 \end{align*}
 This yields the desired ODE~\eqref{cov:particle:eq}, which can be rewritten in a Hessian-free form:
 \begin{align*}
 &\dot\Sigma^{(k)}=- \E_{p_{\theta^{(k)}} }\bigl[\nabla_x  \ln \frac{p}{\pi} \otimes (x-m_k)\bigr] - \E_{p_{\theta^{(k)}} }\bigl[(x-m_k) \otimes \nabla_x  \ln \frac{p}{\pi}\bigr]\,.
 \end{align*}

\section{Background on Otto calculus}\label{scn:otto}

\subsection{Overview and history}

Historically, the connection between dissipative evolution equations and the theory of gradient flows on the Wasserstein space was discovered in~\cite{otto1998magnetic}. Subsequently, this link was further developed and strengthened in the seminal works~\cite{Jordan98, otto2001porousmedium}. Although the paper~\cite{Jordan98} chronologically precedes~\cite{otto2001porousmedium}, the intuition of the former is based heavily on the work of Otto in the latter paper, in which he develops the formal\footnote{Here, ``formal'' is not a synonym for ``rigorous''.} rules governing the calculus which now bears his name.

\emph{Otto calculus} endows the space $\mc P_2(\R^d)$ of probability measures over $\R^d$ with finite second moment with a formal Riemannian structure inspired by fluid dynamics. To describe the idea, suppose that ${(\mu_t)}_{t\ge 0}$ is a curve of probability measures, with $\mu_t$ representing the fluid density at time $t$. Also, let ${(v_t)}_{t\ge 0}$ denote the velocity vector fields governing the dynamics of the particles; this means that the trajectory $t\mapsto x_t$ of an individual particle evolves according to the ODE
\begin{align}\label{eq:particle_ode}
    \dot x_t = v_t(x_t)\,.
\end{align}
In probabilistic language, if $x_0$ is a random variable drawn from the density $\mu_0$ and it evolves according to~\eqref{eq:particle_ode}, then $x_t \sim \mu_t$ for all $t\ge 0$. From this, we can derive a partial differential equation (PDE) governing the evolution of ${(\mu_t)}_{t\ge 0}$ as follows: fix a test function $\varphi : \R^d\to\R$ (which is bounded, smooth, etc.). Formally, if the integration by parts is justified, then
\begin{align*}
    \int \varphi \, \partial_t \mu_t
    &= \partial_t \int \varphi \, \D \mu_t
    = \partial_t \E \varphi(x_t)
    = \int \langle \nabla \varphi, v_t \rangle \, \D \mu_t
    = -\int \varphi \divergence(\mu_t v_t)
\end{align*}
from which we deduce the \emph{continuity equation} of fluid dynamics:
\begin{align}\label{eq:cont_eq}
    \partial_t \mu_t + \divergence(\mu_t v_t) = 0\,.
\end{align}
Conversely, if ${(\mu_t)}_{t\ge 0}$ is a sufficiently nice curve, then it is always possible to find a family of vector fields ${(v_t)}_{t\ge 0}$ such that the equation~\eqref{eq:cont_eq} holds, i.e., we can interpret ${(\mu_t)}_{t\ge 0}$ as the evolution of a fluid density. However, the choice of vector fields is not unique, since we may always replace $v_t$ with another vector field $\tilde v_t$ such that $\divergence(\mu_t \, (v_t - \tilde v_t)) = 0$. This motivates the search for a \emph{distinguished} choice of vector fields to describe the evolution of the curve of measures.

To do so, we pick $v_t$ to minimize the \emph{kinetic energy},
\begin{align*}
    v_t
    &= \argmin{\Bigl\{ \int \norm{w_t}^2 \, \D \mu_t \Bigm\vert w_t : \R^d\to\R~~\text{satisfies}~\divergence(\mu_t w_t) = -\partial_t \mu_t\Bigr\}}\,.
\end{align*}
If $\mu_t$ is regular (admits a density w.r.t.\ Lebesgue measure), then the minimum is attained at a gradient vector field: $v_t = \nabla \psi_t$ for a function $\psi_t : \R^d\to\R$. We are led to define the tangent space
\begin{align*}
    T_\mu \mc P_2(\R^d)
    &= \{\nabla \psi \mid \psi : \R^d\to\R\}
\end{align*}
and endow it with the inner product
\begin{align*}
    \langle v, w\rangle_\mu
    &= \int \langle v, w \rangle \, \D \mu\,.
\end{align*}
This yields a formal Riemannian structure on $\mc P_2(\R^d)$. Moreover, the choice of picking the vector field with minimal kinetic energy is closely related to the idea of optimal transport of mass~\citep[see][]{villani2003topics}, and in fact~\cite{benamoubrenier1999} showed that
\begin{align}\label{eq:benamou_brenier}
    W_2^2(\mu_0, \mu_1)
    &= \inf\Bigl\{ \int \norm{v_t}_{\mu_t}^2 \, \D t \Bigm\vert {(\mu_t,v_t)}_{t\in [0,1]}~~\text{solves the continuity equation}~\eqref{eq:cont_eq}\Bigr\}\,.
\end{align}
From the lens of Riemannian geometry, this says that the notion of distance induced by the Riemannian structure is precisely the quadratic Wasserstein distance, and hence we refer to the space $\mc P_2(\R^d)$ equipped with this Riemannian structure as the \emph{Wasserstein space}.

This formal picture already allows one to compute gradients of functionals defined over $\mc P_2(\R^d)$ and hence to consider gradient flows, as well as to derive criteria which imply quantitative rates of convergence for these flows. However, it is a considerable technical undertaking to make the preceding formal considerations fully rigorous, and this was only accomplished later in the comprehensive monograph~\cite{ambrosio2008gradient}. Instead, in~\cite{Jordan98}, the authors sidestep this difficulty by considering an implicit time-discretization scheme which only requires the metric structure of $(\mc P_2(\R^d), W_2)$. For a step size $h > 0$, define the discrete updates
\begin{align}\label{eq:jko}
    \mu_{h,k+1}
    &:= \argmin_{\mu\in \mc P_2(\R^d)}\Bigl\{ \eu F(\mu) + \frac{1}{2h} \, W_2^2(\mu, \mu_{h,k})\Bigr\}\,,
\end{align}
where $\eu F : \mc P_2(\R^d)\to\R \cup \{\infty\}$ is the functional of interest defined over the Wasserstein space. Note that in optimization, this is known as the ``proximal point method'' for minimizing $\eu F$.

As $h \searrow 0$, one hopes that we have convergence $\mu_{h,\lfloor t/h\rfloor} \to \mu_t$ in a suitable sense, and then the limiting curve ${(\mu_t)}_{t\ge 0}$ can be interpreted as the Wasserstein gradient flow of $\eu F$. This is indeed what~\cite{Jordan98} showed in a particular, but important case.
Namely, if $\pi \propto\exp(-V)$ is a density on $\R^d$ obeying mild regularity conditions, and we take the functional to be the KL divergence, $\eu F(\mu) = \KL(\mu \mmid \pi)$, then the sequence of discrete approximations converges to the solution of the Fokker{--}Planck equation
\begin{align}\label{eq:fokker_planck}
    \partial_t \mu_t
    &= \divergence\bigl(\mu_t \,\nabla \ln \frac{\mu_t}{\pi}\bigr)\,.
\end{align}
It is well-known that the Fokker{--}Planck equation governs the evolution of the marginal law of the Langevin diffusion
\begin{align*}
    \D X_t
    &= -\nabla V(X_t) \, \D t + \sqrt 2 \, \D B_t\,,
\end{align*}
where ${(B_t)}_{t\ge 0}$ is a standard Brownian motion on $\R^d$. Hence, this celebrated result says that the Langevin diffusion can be interpreted as the Wasserstein gradient flow of the KL divergence.
The implicit discretization~\eqref{eq:jko} is now commonly known as the ``JKO scheme'' after the authors Jordan, Kinderlehrer, and Otto.

Although the Wasserstein space is not truly a Riemannian manifold, many of the formal calculations of~\cite{otto2001porousmedium} can now be justified rigorously, under appropriate technical conditions, due to the extensive theory developed in~\cite{ambrosio2008gradient, villani2009ot}. This perspective leads to intuitive derivations of gradient flows, as explained in Section~\ref{Proof2}, and much more.

\subsection{Geometry of the Wasserstein space}\label{scn:w2_space}

In this section, we provide further details about the geometry of $(\mc P_2(\R^d), W_2)$.

Let $\mu_0,\mu_1 \in \mc P_2(\R^d)$, and for simplicity assume that $\mu_0$ admits a density with respect to Lebesgue measure. Then, Brenier's theorem~\cite[Theorem 2.12]{villani2003topics} says that there exists a proper, convex, lower semicontinuous $\varphi : \R^d\to\R \cup \{\infty\}$ such that $\nabla \varphi$ solves the optimal transport problem from $\mu_0$ to $\mu_1$: namely, ${(\nabla \varphi)}_\# \mu_0 = \mu_1$ and $W_2^2(\mu_0,\mu_1) = \int \norm{\nabla \varphi(x) - x}^2 \, \D \mu_0(x)$. We refer to $\nabla \varphi$ as the \emph{optimal transport map} from $\mu_0$ to $\mu_1$.

The (unique) constant-speed geodesic ${(\mu_t)}_{t\in [0,1]}$ joining $\mu_0$ to $\mu_1$ is then described via
\begin{align}\label{eq:w2_geodesic}
    \mu_t
    &= {(\nabla \varphi_t)}_\# \mu_0\,, \qquad \nabla \varphi_t := (1-t) \, {\id} + t \, \nabla \varphi\,.
\end{align}
In view of the fluid dynamical perspective, the constant-speed geodesics in the Wasserstein space correspond to particle trajectories $t\mapsto x_t$ which are straight lines traversed at constant speed: indeed, $x_t = \nabla \varphi_t(x_0) =  (1-t) \, x_0 + t \, \nabla \varphi(x_0)$. Since $\dot x_t = \nabla \varphi(x_0) - x_0 = (\nabla \varphi - {\id}) \circ {(\nabla \varphi_t)}^{-1}(x_t)$, then along the geodesic we see that ${(\mu_t,v_t)}_{t\in [0,1]}$ solves the continuity equation~\eqref{eq:cont_eq}, where the vector field is $v_t = (\nabla \varphi - {\id}) \circ {(\nabla \varphi_t)}^{-1}$. This solution achieves the minimum in~\eqref{eq:benamou_brenier}.

Recall that on a Riemannian manifold $\mc M$, the Riemannian exponential map at $p$ is defined on a subset of the tangent space $T_p \mc M$, and it maps $v$ to the endpoint of the constant-speed geodesic at time $1$ which emanates from $p$ with velocity $v$ (at time $0$). The Riemannian logarithmic map $\log_p$ is the inverse mapping: it maps an element $q \in \mc M$ to the element $v \in T_p \mc M$ such that the constant-speed geodesic joining $p$ to $q$ in one unit of time has velocity $v$ at time $0$. In the previous paragraph, we have identified the logarithmic map: $\log_\mu \nu = \nabla \varphi_{\mu\to\nu} - {\id}$, where $\nabla \varphi_{\mu\to\nu}$ is the optimal transport map from $\mu$ to $\nu$. Thus, the Riemannian exponential map is $\exp_\mu v = {({\id} + v)}_\# \mu$.

\subsection{The Bures{--}Wasserstein space}\label{scn:bures_wasserstein}

The space of non-degenerate Gaussian distributions equipped with the $W_2$ metric is known as the Bures{--}Wasserstein space, after~\cite{Bures69}. We denote this space as $\BW(\R^d)$.

Given $m\in\R^d$ and $\Sigma \succ 0$, we denote by $p_{m,\Sigma}$ the Gaussian on $\R^d$ with mean $m$ and covariance $\Sigma$. Conversely, for a non-degenerate Gaussian $p$ we write $(m_p, \Sigma_p)$ for its mean and covariance. Via this correspondence, we can therefore identify the space of non-degenerate Gaussians with the manifold $\R^d\times \mb S_{++}^d$, where $\mb S_{++}^d$ denotes the cone of positive definite matrices. Abusing notation, we will do so whenever there is no danger of confusion.

Suppose that $p_{m_0,\Sigma_0}, p_{m_1,\Sigma_1} \in \BW(\R^d)$. Then, the optimal transport map from $p_0 := p_{m_0,\Sigma_0}$ to $p_1 := p_{m_1,\Sigma_1}$ is
\begin{align*}
    \nabla \varphi(x)
    &= m_1 + \Sigma_0^{-1/2} \, {(\Sigma_0^{1/2} \Sigma_1 \Sigma_0^{1/2})}^{1/2} \, \Sigma_0^{-1/2} \, (x-m_0)\,.
\end{align*}
Observe that $\nabla \varphi$ is an affine map. Since the pushforward of a Gaussian via an affine map is also Gaussian, it follows from~\eqref{eq:w2_geodesic} that the constant speed geodesic ${(p_t)}_{t\in [0,1]}$ joining $p_0$ to $p_1$ also lies in $\BW(\R^d)$. In other words, $\BW(\R^d)$ is a \emph{geodesically convex} subset of $\mc P_2(\R^d)$.

The tangent vector to the geodesic at time $0$ is always an affine map of the form $x \mapsto a + S \, (x-m_{p_0})$, where $a\in\R^d$ and $S$ is a symmetric matrix. The tangent space is
\begin{align*}
    T_p \BW(\R^d)
    &= \{x \mapsto a + S\, (x-m_p) \mid a \in \R^d, \; S \in \mb S^d\}\,,
\end{align*}
which can therefore be identified with pairs $(a,S) \in \R^d\times \mb S^d$.
With this abuse of notation, if $(a,S), (a',S') \in T_p \BW(\R^d)$, then
\begin{align}\label{eq:bures_metric}
    \langle (a,S), (a', S') \rangle_p
    &= \int \langle a + S \, (x-m_p), a' + S' \,(x-m_p)\rangle \, \D p(x)
    = \langle a, a'\rangle + \langle S, \Sigma_p S'\rangle\,.
\end{align}
Specializing the notions from the previous section, we obtain
\begin{align*}
    \log_p(q)
    &= \bigl(m_q - m_p, \; \Sigma_p^{-1/2} \, {(\Sigma_p^{1/2} \Sigma_q \Sigma_p^{1/2})}^{1/2} \, \Sigma_p^{-1/2} - I\bigr)\,, \\
    \exp_p(a, S)
    &= {\bigl(m_p + a + (S + I) \, (\cdot - m_p)\bigr)}_\# p = \normal\bigl(m_p + a, \;(S+I) \, \Sigma_p \, (S+I)\bigr)\,.
\end{align*}
Here, $\exp_p(a, S)$ is defined if $S \succ -I$.

This definition of the tangent space is consistent with the Wasserstein space, in that we have the inclusion $T_p \BW(\R^d) \hookrightarrow T_p \mc P_2(\R^d)$, but the abuse of notation $T_p \BW(\R^d) = \R^d\times \mb S^d$ can sometimes cause confusion.
Indeed, if ${(p_t = p_{m_t,\Sigma_t})}_{t\in [0,1]}$ is a constant-speed geodesic in $\BW(\R^d)$, and the tangent vector at time $0$ is $(a, S)$, then
\begin{align*}
    p_t
    &= \exp_{p_0}\bigl(t \, (a, S)\bigr)
    = \normal\bigl(m_p + ta,\; (tS + I) \, \Sigma_p \, (tS + I)\bigr)\,.
\end{align*}
In particular, $\Sigma_t \ne \Sigma_0 + t \, (S - I)$, and
\begin{align}
    \dot m_0
    &= a\,, \label{eq:evolution_mean}\\
    \dot \Sigma_0
    &= S\Sigma_0 + \Sigma_0 S\,. \label{eq:evolution_cov}
\end{align}
Although we derived the equations~\eqref{eq:evolution_mean} and~\eqref{eq:evolution_cov} for geodesic curves, they also hold for any curve ${(p_t)}_{t\ge 0}$ with tangent vector equal to $(a, S)$ at time $0$. Using this, we can derive an expression for the Bures{--}Wasserstein gradient $\nabla_{\BW} f$ of a function $f : \R^d\times \mb S_{++}^d \to \R$. By definition, this satisfies, for any curve ${(m_t,\Sigma_t)}_{t\ge 0}$ with tangent vector $(a, S)$ at time $0$,
\begin{align*}
    \langle \nabla_{\BW} f(m_0,\Sigma_0), (a, S) \rangle_{p_{m_0,\Sigma_0}}
    &= \partial_t\big|_{t=0} f(m_t,\Sigma_t)\,.
\end{align*}
Write $(\bar a, \bar S) = \nabla_{\BW} f(m_0,\Sigma_0)$.
Then, we want
\begin{align*}
    \langle \bar a, a \rangle + \langle \bar S, \Sigma_0 S\rangle
    &= \langle \nabla_m f(m_0,\Sigma_0), \dot m_0 \rangle + \langle \nabla_\Sigma f(m_0,\Sigma_0), \dot \Sigma_0 \rangle \\
    &= \langle \nabla_m f(m_0,\Sigma_0), a \rangle + 2\,\langle \nabla_\Sigma f(m_0,\Sigma_0), \Sigma_0 S \rangle\,,
\end{align*}
where $\nabla_m$, $\nabla_\Sigma$ denote the usual Euclidean gradients. Hence, by identification, we conclude that the Bures{--}Wasserstein gradient of $f$ is related to the Euclidean gradient of $f$ via
\begin{align}\label{eq:bw_gradient}
    \nabla_{\BW} f(m,\Sigma)
    &= \bigl(\nabla_m f(m,\Sigma), \; 2\,\nabla_\Sigma f(m,\Sigma)\bigr)\,.
\end{align}

See~\cite[Appendix A]{altschuleretal2023bwbarycenter} for further discussion.

\subsection{Evolution of the mean and covariance along the Fokker{--}Planck equation}\label{scn:mean_cov_fp}

It is known that the Wasserstein gradient of $\eu F := \KL(\cdot \mmid \pi)$ is
\begin{align}\label{eq:grad_kl}
    \nabla_{W_2} \eu F(\mu)
    &= \nabla \ln \frac{\mu}{\pi}\,.
\end{align}
\citep[See, e.g.,][Theorem 10.4.13.]{ambrosio2008gradient} Also, as shown by~\cite{Jordan98}, the Langevin diffusion is the gradient flow of $\KL(\cdot \mmid \pi)$. In Otto calculus, this means that the law ${(\pi_t)}_{t\ge 0}$ of the Langevin diffusion obeys the continuity equation~\eqref{eq:cont_eq} with velocity vector field $v_t = -\nabla_{W_2} \eu F(\pi_t) = -\nabla \ln(\pi_t/\pi)$, which is consistent with the Fokker--Planck equation~\eqref{eq:fokker_planck}.

According to the particle interpretation~\eqref{eq:particle_ode} of dynamics in the Wasserstein space, if  $x_0 \sim \pi_0$ and
\begin{align*}
    \dot x_t
    &= v_t(x_t)
    = -\nabla \ln \frac{\pi_t}{\pi}(x_t)\,,
\end{align*}
then $x_t \sim \pi_t$. Note that ${(x_t)}_{t\ge 0}$ is \emph{not} the Langevin diffusion~\eqref{eq:langevin} as it is the solution to a deterministic ODE (albeit with random initial condition), but the marginal law of ${(x_t)}_{t\ge 0}$ agrees with that of the Langevin diffusion. This provides a convenient tool for calculating the evolution of the mean and covariance along the Fokker{--}Planck equation, as we now demonstrate.

The evolution of the mean is
\begin{align*}
    \dot m_t
    &= \partial_t \E x_t
    = \E \dot x_t
    = -\E\nabla \ln \frac{\pi_t}{\pi}(x_t)\,.
\end{align*}
Since $\E\nabla \ln \pi_t(x_t) = 0$ (which is verified via integration by parts), and $\pi \propto e^{-V}$, this can also be written as
\begin{align*}
    \dot m_t
    &= -\E_{\pi_t}\nabla V\,.
\end{align*}
Next, for the evolution of the covariance,
\begin{align*}
    \partial_t \E(x_t \otimes x_t)
    &= \E(x_t \otimes \dot x_t + \dot x_t \otimes x_t)
    = -\E\bigl( x_t \otimes \nabla \ln \frac{\pi_t}{\pi}(x_t) + \nabla \ln \frac{\pi_t}{\pi}(x_t) \otimes x_t\bigr) \\
    \partial_t \E(x_t) \otimes \E(x_t)
    &= m_t \otimes \E(\dot x_t) + \E(\dot x_t) \otimes m_t
    = -\E\bigl( m_t \otimes \nabla \ln \frac{\pi_t}{\pi}(x_t) + \nabla \ln \frac{\pi_t}{\pi}(x_t) \otimes m_t\bigr)
\end{align*}
which yields
\begin{align*}
    \dot \Sigma_t
    &= -\E\bigl( (x_t-m_t) \otimes \nabla \ln \frac{\pi_t}{\pi}(x_t) + \nabla \ln \frac{\pi_t}{\pi}(x_t) \otimes (x_t-m_t)\bigr)\,.
\end{align*}
Integration by parts yields
\begin{align*}
    &\int (\bullet - m_t) \otimes \nabla \ln \pi_t \, \D \pi_t + \int \nabla \ln \pi_t \otimes (\bullet - m_t) \, \D \pi_t \\
    &\qquad = \int (\bullet - m_t) \otimes \nabla \pi_t + \int \nabla \pi_t \otimes (\bullet - m_t)
    = -2I\,.
\end{align*}
Hence,
\begin{align*}
    \dot \Sigma_t
    &= 2I - \E_{\pi_t}[\nabla V \otimes (\bullet - m_t) + (\bullet - m_t) \otimes \nabla V]\,.
\end{align*}
This verifies equation~\eqref{eq:mean_cov_fp}. The equations in this section can also be derived using It\^o calculus.

\section{Proofs via Otto calculus}\label{Proof2}

Our aim in this section is to derive the Wasserstein gradient flow of the KL divergence $\KL(\cdot \mmid \pi)$ \emph{constrained} to lie in the Bures--Wasserstein space of non-degenerate Gaussian measures.

Since the Bures--Wasserstein space can be formally viewed as a submanifold of the Wasserstein space, it leads to two natural approaches for computing the constrained gradient flow.
In the first approach, we take the Wasserstein gradient of $\KL(\cdot \mmid \pi)$ and we compute the orthogonal projection onto the tangent space of the Bures--Wasserstein space. In the second approach, we note that the geometry of the Bures--Wasserstein space has been studied in its own right~\citep[see, e.g.,][]{Bhatia19} and in particular, the explicit expression~\eqref{eq:bw_gradient} for the Bures{--}Wasserstein gradient is known.
We can therefore view $\KL(\cdot \mmid \pi)$ as a functional over $\BW(\R^d)$ and compute its gradient directly using~\eqref{eq:bw_gradient}.

\subsection{Orthogonal projection approach}\label{scn:orthogonal_proj}

First, we justify why computing the orthogonal projection of the $\mc P_2(\R^d)$ gradient gives the same result as computing the intrinsic gradient on $\BW(\R^d)$.
Let $\eu F$ be any functional on $\mc P_2(\R^d)$.
By definition, the Bures--Wasserstein gradient $\nabla_{\BW} \eu F$ satisfies
\begin{align}\label{eq:defn_bw_grad}
    \partial_t \eu F(p_t)
    &= \langle \nabla_{\BW} \eu F(p_t), v_t \rangle_{p_t}
\end{align}
for any curve ${(p_t)}_{t\in\R}$ in $\BW(\R^d)$ with tangent vectors ${(v_t)}_{t\in\R}$. Here, $\nabla_{\BW} \eu F(p_t) \in T_{p_t} \BW(\R^d)$. On the other hand, since ${(p_t)}_{t\ge 0}$ is also a curve in $\mc P_2(\R^d)$ and the Riemannian structure of $\BW(\R^d)$ is consistent with that of $\mc P_2(\R^d)$, the definition of the gradient in $\mc P_2(\R^d)$ yields
\begin{align*}
    \partial_t \eu F(p_t)
    &= \langle \nabla_{W_2} \eu F(p_t), v_t\rangle_{p_t}\,.
\end{align*}
Note that the orthogonal projection
\begin{align*}
    \proj_{T_{p_t} \BW(\R^d)} \nabla_{W_2} \eu F(p_t) = \argmin_{w \in T_{p_t} \BW(\R^d)}{\norm{w - \nabla_{W_2} \eu F(p_t)}_{p_t}^2}
\end{align*}
is characterized as the unique element of $T_{p_t} \BW(\R^d)$ satisfying
\begin{align*}
    \langle \proj_{T_{p_t} \BW(\R^d)} \nabla_{W_2} \eu F(p_t), v \rangle_{p_t}
    &= \langle \nabla_{W_2} \eu F(p_t), v \rangle_{p_t}
\end{align*}
for all $v \in T_{p_t} \BW(\R^d)$. Thus,~\eqref{eq:defn_bw_grad} holds with
\begin{align*}
    \nabla_{\BW} \eu F(p)
    &= \proj_{T_{p} \BW(\R^d)} \nabla_{W_2} \eu F(p)\,.
\end{align*}
This argument clearly works for arbitrary Riemannian submanifolds.

Next, we compute the projection of the $\mc P_2(\R^d)$ gradient of the KL divergence.

 Using the formula~\eqref{eq:grad_kl} for the $\mc P_2(\R^d)$ gradient of the KL divergence and the description of the tangent space to $\BW(\R^d)$ in Section~\ref{scn:bures_wasserstein} and~\eqref{eq:bures_metric}, the projected gradient $(\bar a, \bar S) \in \R^d\times \mb S^d$ is such that for all $(a,S) \in \R^d\times \mb S^d$,
\begin{align*}
    \int \bigl\langle \nabla \ln \frac{p}{\pi}(x), a + S \, (x-m_p) \bigr\rangle \, \D p(x)
    &= \langle (\bar a, \bar S), (a, S) \rangle_p
    = \langle \bar a, a \rangle + \langle \bar S, \Sigma_p S\rangle\,.
\end{align*}
Using $\nabla p(x) = -\Sigma_p^{-1} \, (x-m_p) \, p(x)$ and integration by parts,
\begin{align*}
    &\int \bigl\langle \nabla \ln \frac{p}{\pi}(x), a + S \, (x-m_p) \bigr\rangle \, \D p(x) \\
    &\qquad = \bigl\langle \E_p \nabla \ln \frac{p}{\pi}, a \bigr\rangle + \int \bigl\langle \Sigma_p S\, \nabla \ln \frac{p}{\pi}(x), \Sigma_p^{-1} \, (x-m_p) \bigr\rangle \, \D p(x) \\
    &\qquad = \bigl\langle \E_p \nabla \ln \frac{p}{\pi}, a \bigr\rangle - \int \bigl\langle \Sigma_p S\, \nabla \ln \frac{p}{\pi}(x), \nabla p(x) \bigr\rangle \, \D x \\
    &\qquad = \bigl\langle \E_p \nabla \ln \frac{p}{\pi}, a \bigr\rangle + \int \divergence\bigl(\Sigma_p S\, \nabla \ln \frac{p}{\pi}\bigr)(x) \, \D p(x) \\
    &\qquad = \bigl\langle \E_p \nabla \ln \frac{p}{\pi}, a \bigr\rangle + \bigl\langle \E_p \nabla^2 \ln \frac{p}{\pi}, \Sigma_p S \bigr\rangle\,.
\end{align*}
Hence,
\begin{align}\label{eq:bw_gradient_kl_proj}
    (\bar a, \bar S)
    &= \bigl(\E_p \nabla \ln \frac{p}{\pi}, \; \E_p \nabla^2 \ln \frac{p}{\pi}\bigr)\,.
\end{align}
Using the fact that $\E_p \nabla \ln p = 0$, this can also be written
\begin{align*}
    (\bar a, \bar S)
    &= (\E_p \nabla V, \; \E_p\nabla^2 V - \Sigma_p^{-1})
\end{align*}
which corresponds to the affine map
\begin{align}\label{eq:grad_kl_affine_map}
    x \mapsto \E_p \nabla V + (\E_p\nabla^2 V - \Sigma_p^{-1}) \, (x-m_p)\,.
\end{align}

If ${(p_t = p_{m_t, \Sigma_t})}_{t\ge 0}$ evolves according to the constrained gradient flow, then using the expression for the projected Wasserstein gradient together with~\eqref{eq:evolution_mean} and~\eqref{eq:evolution_cov},
\begin{align*}
    \boxed{\begin{aligned} \dot m_t &= -\E_{p_t} \nabla V\,, \\ \dot \Sigma_t &= 2I - \Sigma_t \E_{p_t} \nabla^2 V - \E_{p_t}\nabla^2 V \, \Sigma_t\,. \end{aligned}}
\end{align*}
The sign in the above equations comes from the fact that we perform steepest \emph{descent} in Bures{--}Wasserstein descent, i.e., the tangent vector to the curve at time $t$ is $-\proj_{T_{p_t} \BW(\R^d)} \nabla_{W_2} \eu F(p_t)$.

The system of equations we have derived here differs from the system~\eqref{eq:sarkka}, but we can check that they agree using integration by parts. Indeed,
\begin{align*}
    \dot \Sigma_t
    &= 2I - \Sigma_t \int \nabla^2 V \, \D p_t - \int \nabla^2 V \, \D p_t
    = 2I + \Sigma_t \int \nabla p_t \otimes \nabla V + \int \nabla V \otimes \nabla p_t \, \Sigma_t \\
    &= 2I + \Sigma_t \int \nabla \ln p_t \otimes \nabla V \, \D p_t + \int \nabla V \otimes \nabla \ln p_t \, \D p_t \, \Sigma_t \\
    &= 2I - \E_{p_t}[(\bullet - m_t) \otimes \nabla V + \nabla V \otimes (\bullet - m_t)]\,.
\end{align*}

\subsection{Alternate proof using direct Bures{--}Wasserstein calculation}\label{scn:direct_bw_calc}

In the second approach, we view $\eu F$ as a functional on the Bures{--}Wasserstein space.
Explicitly,
\begin{align*}
    \eu F(m,\Sigma)
    &= \int p_{m,\Sigma} \ln \frac{p_{m,\Sigma}}{\pi}\,.
\end{align*}
Using~\eqref{eq:bw_gradient},
\begin{align}
    \nabla_{\BW} \eu F(m,\Sigma)
    &= \bigl(\nabla_m \eu F(m,\Sigma), \; 2 \, \nabla_\Sigma \eu F(m,\Sigma)\bigr) \nonumber \\
    &=\Bigl(\int \nabla_m p_{m,\Sigma} \ln \frac{p_{m,\Sigma}}{\pi}, \; 2\int \nabla_\Sigma p_{m,\Sigma} \ln \frac{p_{m,\Sigma}}{\pi}\Bigr)\,.\label{eq:bw_gradient_kl_direct}
\end{align}
Furthermore, using the identities
\begin{align}\label{eq:gaussian_differential_identities}
    \nabla_m p_{m,\Sigma}(x)
    &= -\nabla_x p_{m,\Sigma}(x) \qquad\text{and}\qquad
    \nabla_\Sigma p_{m,\Sigma}(x)
    = \frac{1}{2} \, \nabla^2_x p_{m,\Sigma}(x)
\end{align}
for the Gaussian distribution, integration by parts verifies that~\eqref{eq:bw_gradient_kl_direct} agrees with~\eqref{eq:bw_gradient_kl_proj}.

\section{Proof of Corollary~\ref{COR:CONT_TIME_GUARANTEE}}\label{scn:pf_cont_time_guarantee}

Corollary~\ref{COR:CONT_TIME_GUARANTEE} is a consequence of general and well-known principles for gradient flows.
To emphasize this generality, we will consider an abstract $\alpha$-convex differentiable functional $\eu F$ defined over a geodesically convex subset of a Riemannian manifold; this ensures that the logarithmic map is well-defined in the following calculations. We assume that $\eu F$ is minimized at $p^\star$; by adding a constant to $\eu F$, we can assume $\inf \eu F = 0$. Let $\msf d$ denote the distance function on the manifold. If ${(p_t)}_{t\ge 0}$, ${(q_t)}_{t\ge 0}$ are two solutions to the gradient flow for $\eu F$, then
\begin{align*}
    \partial_t \msf d^2(p_t, q_t)
    &= 2 \,\langle \log_{p_t}(q_t), \nabla \eu F(p_t) \rangle_{p_t} + 2\, \langle \log_{q_t}(p_t), \nabla \eu F(q_t) \rangle_{q_t}\,.
\end{align*}
(The reader who is unfamiliar with Riemannian geometry should keep in mind that in Euclidean space, $\log_p(q) = q - p$.) Next, the $\alpha$-convexity of $\eu F$ implies
\begin{align*}
    \eu F(p_t)
    &\ge \eu F(q_t) + \langle \nabla \eu F(q_t), \log_{q_t}(p_t) \rangle_{q_t} + \frac{\alpha}{2} \, \msf d^2(p_t, q_t)\,, \\
    \eu F(q_t)
    &\ge \eu F(p_t) + \langle \nabla \eu F(p_t), \log_{p_t}(q_t) \rangle_{p_t} + \frac{\alpha}{2} \, \msf d^2(p_t, q_t)\,.
\end{align*}
Adding these equations and rearranging yields
\begin{align*}
    \partial_t \msf d^2(p_t, q_t)
    &\le -2\alpha \, \msf d^2(p_t, q_t)\,.
\end{align*}
By Gr\"onwall's inequality, it implies
\begin{align*}
    \msf d^2(p_t,q_t)
    &\le \exp(-2\alpha t) \, \msf d^2(p_0, q_0)\,.
\end{align*}
This inequality has two consequences.
First, for any $\alpha \in \R$, $p_0 = q_0$ implies $p_t = q_t$: the solution to the gradient flow is unique. Second, if $\alpha > 0$, then we can set $q_t = p^\star$ for all $t\ge 0$ to deduce exponential contraction of the gradient flow to the minimizer $p^\star$, which is the first statement of Corollary~\ref{COR:CONT_TIME_GUARANTEE}.

To obtain convergence in functional values, observe that by definition of the gradient flow, we have on the one hand that
\begin{align}\label{eq:time_deriv_gf}
    \partial_t \eu F(p_t)
    &= -\norm{\nabla \eu F(p_t)}_{p_t}^2\,.
\end{align}
On the other hand, if $\alpha > 0$, the convexity inequality and Young's inequality respectively, yield
\begin{align}
    0 = \eu F(p^\star)
    &\ge \eu F(p) + \langle \nabla \eu F(p), \log_p(p^\star) \rangle_p + \frac{\alpha}{2} \, \msf d^2(p, p^\star) \label{eq:cvxty_ineq} \\
    &\ge \eu F(p) - \frac{1}{2\alpha} \, \norm{\nabla \eu F(p)}_p^2 - \frac{\alpha}{2} \, {\underbrace{\norm{\log_p(p^\star)}_p^2}_{= \msf d^2(p, p^\star)}} + {\frac{\alpha}{2} \, \msf d^2(p, p^\star)} \nonumber
\end{align}
and hence $\norm{\nabla \eu F(p)}^2 \ge 2\alpha \, \eu F(p)$. Substituting this into~\eqref{eq:time_deriv_gf} and applying Gr\"onwall's inequality again, we deduce
\begin{align*}
    \eu F(p_t)
    &\le \exp(-2\alpha t) \, \eu F(p_0)\,.
\end{align*}

Finally, suppose $\alpha = 0$.
We consider the Lyapunov functional
\begin{align*}
    \eu L_t
    &\deq t \, \eu F(p_t) + \frac{1}{2} \, \msf d^2(p_t, p^\star)\,.
\end{align*}
Differentiating in time,
\begin{align*}
    \partial_t \eu L_t
    &= \eu F(p_t) - t \, \norm{\nabla \eu F(p_t)}_{p_t}^2 + \langle \log_{p_t}(p^\star), \nabla \eu F(p_t)\rangle_{p_t}\,.
\end{align*}
On the other hand, applying the convexity inequality in~\eqref{eq:cvxty_ineq} with $\alpha = 0$ yields $\partial_t \eu L_t \le 0$.
Hence, $\eu L_t \le \eu L_0$, and
\begin{align*}
    \eu F(p_t)
    &\le \frac{\msf d^2(p_0, p^\star)}{2t}\,.
\end{align*}

\section{Proof of Theorem~\ref{THEOREM2}} \label{Proof3}

In this section, we use the Riemannian exponential and logarithmic maps, as discussed in Section~\ref{scn:bures_wasserstein}. Also, let $\eu F := \KL(\cdot \mmid \pi)$ denote the KL divergence.

For $\tau > 0$, the eigenvalue clipping operation is defined as
\begin{align}\label{eq:clip}
    \clip^\tau :\qquad \Sigma = \sum_{i=1}^d \lambda_i u_i u_i^\T\quad \mapsto \quad\clip^\tau\Sigma \deq \sum_{i=1}^d (\lambda_i \wedge \tau) \, u_i u_i^\T\,.
\end{align}

In the proof of Theorem~\ref{thm:main} in Section~\ref{Proof2}, we showed that the Bures{--}Wasserstein gradient is
\begin{align}\label{eq:bw_gradient_of_kl}
    g_p
    &\deq \nabla_{\BW} \eu F(p)
    = \bigl(\E_p \nabla V, \; \E_p\nabla^2 V - \Sigma^{-1} \bigr)
\end{align}
where $\Sigma$ is the covariance matrix of $p$.
Here, the first component of the gradient governs the evolution of the mean, whereas the second component governs the evolution of the covariance; see Section~\ref{scn:bures_wasserstein}.
We propose to estimate the gradient in~\eqref{eq:bw_gradient_of_kl} via a sample,
\begin{align*}
    \hat g_p
    &\deq \bigl(\nabla V(\hat X), \; \nabla^2 V(\hat X) - \Sigma^{-1} \bigr)\,, \qquad \hat X \sim p\,.
\end{align*}
By comparing Algorithm~\ref{alg:bwsgd} and the definition of the exponential map in Section~\ref{scn:bures_wasserstein}, one can check that for $p_k^+ \deq p_{m_{k+1}, \Sigma_k^+}$ and\footnote{This latter requirement is needed because $\BW(\R^d)$ has a finite injectivity radius.} $h\le 1$
\begin{align*}
    p_k^+
    &= \exp_{p_k}(-h \hat g_k)\,,
\end{align*}
where $\hat g_k \in T_{p_k} \BW(\R^d)$ is the stochastic gradient
\begin{align*}
    \hat g_k(x)
    &= \nabla V(\hat X_k) + (\nabla^2 V(\hat X_k) - \Sigma_k^{-1}) \, (x-m_k)\,.
\end{align*}
Thus, aside from the eigenvalue clipping operation (which is harmless, due to Lemma~\ref{lem:clipping_contraction} below), Algorithm~\ref{alg:bwsgd} is exactly a stochastic gradient descent scheme on $\BW(\R^d)$. Note also that from the definition of the exponential map in Section~\ref{scn:w2_space}, the update can also be written at the particle level: if $X_k \sim p_k$ is independent of $\hat g_k$, then
\begin{align}\label{eq:w2_sgd}
    X_k^+
    \deq X_k - h \, \hat g_k(X_k)
    \sim p_k^+\,.
\end{align}

In the next lemma, we obtain a uniform control on the smallest eigenvalues of the covariance matrices of the iterates.

\begin{lemma}\label{lem:control_iterates}
    Assume that $0 \prec \alpha I \preceq \nabla^2 V \preceq I$ holds and $h\le \alpha^2/60$.
    Also, in Algorithm~\ref{alg:bwsgd}, assume that $\Sigma_k \succeq \frac{\alpha}{9} \, I$.
    Then, $\Sigma_k^+ \succeq \frac{\alpha}{9} \, I$.
\end{lemma}
\begin{proof}
    Since the statement of the lemma only involves the covariance matrices, we can suppose that all of the mean vectors are zero.
    
    The key is to write $\Sigma_k^+$ as a generalized Bures{--}Wasserstein barycenter at $\Sigma_k$ for an appropriate distribution. Recall that
    \begin{align}
    \label{eq:update_rule}
        \Sigma_k^+
        &= \bigl(I + h \, \Sigma_k^{-1} - h \, \nabla^2 V(\hat X_k)\bigr)\, \Sigma_k\,\bigl(I + h \, \Sigma_k^{-1} - h \, \nabla^2 V(\hat X_k)\bigr)\,.
    \end{align}
    Note that $\Sigma_k^{-1}$ is the optimal transport map from the Gaussian $p_{0,\Sigma_k}$ to $p_{0,\Sigma_k^{-1}}$.\footnote{This observation was also used in the analysis of Bures{--}Wasserstein gradient descent for entropically regularized barycenters in~\cite{altschuleretal2023bwbarycenter}.}
    Hence,
    \begin{align*}
        h \, \Sigma_k^{-1} - h \, \nabla^2 V(\hat X_k)
        &= h \, (\Sigma_k^{-1} - I) + h \, (I - \nabla^2 V(\hat X_k)) \\
        &= h \log_{\Sigma_k}(\Sigma_k^{-1}) + h \log_{\Sigma_k}(\tilde \Sigma)
    \end{align*}
    where we defined the matrix $\tilde \Sigma = (2I - \nabla^2 V(\hat X_k)) \, \Sigma_k \, (2I - \nabla^2 V(\hat X_k))$.
    To check that this is valid, we need $2I - \nabla^2 V(\hat X_k) \succeq 0$, i.e., $\nabla^2 V(\hat X_k) \preceq 2I$, which follows from $\nabla^2 V \preceq I$.

    We have shown that
    \begin{align*}
        \Sigma_k^+
        &= \exp_{\Sigma_k}\Bigl(\int \log_{\Sigma_k}(\Sigma) \, \D P(\Sigma)\Bigr)
    \end{align*}
    where
    \begin{align*}
        P
        &= (1-2h) \, \delta_{\Sigma_k} + h\, \delta_{\Sigma_k^{-1}} + h \, \delta_{\tilde \Sigma}
        = (1-2h)\,\delta_{\Sigma_k} + 2h\,\bigl(\frac{1}{2}\,\delta_{\Sigma_k^{-1}} + \frac{1}{2}\,\delta_{\tilde\Sigma}\bigr)\,.
    \end{align*}
    This is precisely the definition of a generalized Bures{--}Wasserstein barycenter.
    
    Next, suppose that $\Sigma_k \succeq \lambda I$ for some $\lambda > 0$. Since $\Sigma_k \preceq \alpha^{-1} I$, and $I \preceq 2I - \nabla^2 V(\hat X_k) \preceq 2I$,
    \begin{align*}
        \alpha\,I
        &\preceq \Sigma_k^{-1} \preceq \frac{1}{\lambda}\,I\,, \qquad\text{and}\qquad
        \lambda \, I
        \preceq \tilde \Sigma \preceq \frac{4}{\alpha}\,I\,.
    \end{align*}
    Then,~\cite[Theorem 1]{altschuleretal2023bwbarycenter}\footnote{See the latest revision.} implies the following.
    If we define the quantities
    \begin{align*}
        \lambda_-
        &\deq \Bigl( \frac{1}{2}\,\sqrt\alpha + \frac{1}{2}\,\sqrt\lambda\Bigr)^2\,, \qquad
        \lambda_+
        \deq \frac{1}{2}\, \frac{1}{\lambda} + \frac{1}{2}\,\frac{4}{\alpha}\,,
    \end{align*}
    then for step sizes $2h\le \frac{\lambda_-}{2\lambda_+}$ and if $\Sigma_k \succeq \frac{\lambda_-}{4}\,I$, we also have $\Sigma_k^+ \succeq \frac{\lambda_-}{4}\, I$.
    To use this result, let us choose $\lambda$ such that $\frac{\lambda_-}{4} = \lambda$; it can be seen that this holds with $\lambda = \frac{\alpha}{9}$.
    Since $\lambda_+ = \frac{13}{2\alpha}$, the step size condition then translates into $h \le \frac{2\alpha^2}{117}$, for which it suffices to have $h \le \frac{\alpha^2}{60}$.
\end{proof}

We also recall an important fact about the eigenvalue clipping operation.

\begin{lemma}[{\citet[Proposition 3]{altschuleretal2023bwbarycenter}}]\label{lem:clipping_contraction}
    For any $m\in\R^d$, $\tau > 0$, and $\Sigma, \Sigma' \in \mb S_{++}^d$,
    \begin{align*}
        W_2(p_{m,\clip^\tau \Sigma}, \;p_{m,\clip^\tau \Sigma'})
        &\le W_2(p_{m,\Sigma}, \;p_{m,\Sigma'})\,.
    \end{align*}
\end{lemma}

We now turn towards the proof of Theorem~\ref{THEOREM2}.
In the proof, we let
\begin{align*}
    \ms F_k \deq \sigma(\hat X_0,\hat X_1,\hat X_2,\dotsc,\hat X_{k-1})
\end{align*}
be the $\sigma$-algebra generated by the random samples up until iteration $k$.

\medskip{}

\begin{proof}[Proof of Theorem~\ref{THEOREM2}]
    Conditioned on $\ms F_k$, and independently of $\hat X_k$, let $X_k \sim p_k$ and $Z \sim \hat \pi$ be optimally coupled; let $\bar\E$ denote the expectation taken w.r.t.\ $(X_k, Z)$.
    Using Lemma~\ref{lem:clipping_contraction}, the fact that $\hat \Sigma \preceq \frac{1}{\alpha} \, I$ (see discussion in Section~\ref{scn:gradient_flow_kl_bw}), and~\eqref{eq:w2_sgd}, we have
    \begin{align*}
        &\E[W_2^2(p_{k+1},\hat \pi)\mid \ms F_k]
        \le \E[W_2^2(p_k^+, \hat \pi) \mid \ms F_k] \\
        &\qquad \le \E\bigl[ \bar\E[\norm{X_k - h \, \hat g_k(X_k) - Z}^2] \bigm\vert \ms F_k\bigr] \\
        &\qquad = \E\bigl[\bar\E[\norm{X_k - Z}^2 - 2h \, \langle \hat g_k(X_k), X_k - Z\rangle + h^2 \, \norm{\hat g_k(X_k)}^2] \bigm\vert \ms F_k\bigr] \\
        &\qquad = W_2^2(p_k, \hat\pi) - 2h\, \bar\E\langle g_k(X_k), X_k - Z\rangle + h^2 \E\bigl[\bar\E[\norm{\hat g_k(X_k)}^2] \bigm\vert \ms F_k\bigr]\,,
    \end{align*}
    where we abbreviated $g_k \deq g_{p_k}$.
    From strong convexity of $\KL(\cdot \mmid \pi)$ on $\BW(\R^d)$ (Lemma~\ref{lem:strong_cvxty_kl_bw}),
    \begin{align*}
        \bar\E\langle g_k(X_k), X_k - Z\rangle
        &\ge \KL(p_k \mmid \pi) - \KL(\hat \pi \mmid \pi) + \frac{\alpha}{2} \, W_2^2(p_k, \hat \pi) \\
        &\ge \alpha \, W_2^2(p_k, \hat \pi)\,.
    \end{align*}
    Thus,
    \begin{align*}
        \E[W_2^2(p_{k+1},\hat \pi)\mid \ms F_k]
        &\le (1-2\alpha h) \, W_2^2(p_k,\hat\pi) + h^2 \underbrace{\E\bigl[\bar\E[\norm{\hat g_k(X_k)}^2] \bigm\vert \ms F_k\bigr]}_{=: \msf{err}}\,.
    \end{align*}
    It remains to bound the error term.

    Recall that
    \begin{align*}
        \hat g_k(X_k)
        &= (\nabla^2 V(\hat X_k) - \Sigma_k^{-1}) \, (X_k - m_k) + \nabla V(\hat X_k)\,.
    \end{align*}
    We bound the terms one by one.
    First,
    \begin{align*}
        \bar \E[\norm{\Sigma_k^{-1} \, (X_k - m_k)}^2]
        = \tr(\Sigma_k^{-1})
        \le \frac{9d}{\alpha}
    \end{align*}
    where we used Lemma~\ref{lem:control_iterates}.
    Next, since $\nabla^2 V \preceq I$ by assumption,
    \begin{align*}
        \bar\E[\norm{\nabla^2 V(\hat X_k) \, (X_k - m_k)}^2]
        &\le \bar\E[\norm{X_k - m_k}^2]
        = \tr(\Sigma_k)
        \le \frac{d}{\alpha}\,.
    \end{align*}
    Lastly, let $\hat Z \sim \hat \pi$ be optimally coupled with $\hat X_k$.
    By the optimality condition for $\hat \pi$ (Section~\ref{scn:gradient_flow_kl_bw}), we know that $\E\nabla V(\hat Z) = 0$.
    Applying the Poincar\'e inequality for $\hat \pi$ (which holds because $\hat \pi$ is strongly log-concave, see~\cite[Theorem 4.8.4]{bakrygentilledoux2014})
    \begin{align*}
        \bar\E[\norm{\nabla V(\hat X_k)}^2]
        &\le 2\, \bar\E[\norm{\nabla V(\hat Z)}^2] + 2\,\bar \E[\norm{\hat X_k - \hat Z}^2] \\
        &\le \frac{2}{\alpha} \E_{\hat \pi}[\norm{\nabla^2 V}_{\rm HS}^2] + 2 \, W_2^2(p_k, \hat \pi) \\
        &\le \frac{2 d}{\alpha} + 2 \, W_2^2(p_k, \hat \pi)\,.
    \end{align*}
    Collecting the terms,
    \begin{align*}
        \msf{err}
        &\le \frac{36d}{\alpha} + 6 \, W_2^2(p_k, \hat \pi)\,.
    \end{align*}
    From the assumption $h \le \frac{\alpha^2}{60}$.
    \begin{align*}
        \E[W_2^2(p_{k+1},\hat \pi) \mid \ms F_k]
        &\le (1-\alpha h) \, W_2^2(p_k, \hat \pi) + \frac{36dh^2}{\alpha}\,.
    \end{align*}
    Iterating this bound proves the result.
\end{proof}

\section{Proof of Theorem~\ref{THM:MIXTURE_GF}}\label{scn:pf_mixtures}

In order to present the proof of Theorem~\ref{THM:MIXTURE_GF}, we first review relevant facts about the Wasserstein space over a Riemannian manifold $(\mc M, \mf g)$. We refer readers to~\citet{villani2009ot} for an in-depth treatment.

Similarly to the Euclidean setting, we can define the space of probability measures over $\mc M$ with finite second moment,
\begin{align*}
    \mc P_2(\mc M)
    &\deq \Bigl\{\mu \in \mc P(\mc M) \Bigm\vert \int \msf d^2(p_0, \cdot) \, \D \mu < \infty~\text{for some}~p_0 \in \mc M\Bigr\}\,,
\end{align*}
where $\msf d$ denotes the induced distance on $\mc M$. We equip $\mc P_2(\mc M)$ with the $2$-Wasserstein metric
\begin{align*}
    W_2^2(\mu,\nu)
    &\deq \Bigl[\inf_{\gamma \in \eu C(\mu,\nu)} \int \msf d^2(x,y) \, \D \gamma(x,y) \Bigr]^{1/2}\,,
\end{align*}
which makes $(\mc P_2(\mc M), W_2)$ into a metric space.
Moreover, at each regular measure $\mu \in \mc P_2(\mc M)$, we can define the tangent space
\begin{align*}
    T_\mu \mc P_2(\mc M)
    &\deq \overline{\{\nabla \psi \mid \psi \in \mc C_{\rm c}^\infty(\mc M)\}}^{L^2(\mu)}
\end{align*}
equipped with the inner product
\begin{align*}
    \langle v, w \rangle_\mu
    &\deq \int \mf g_p\bigl(v(p), w(p)\bigr) \, \D \mu(p)\,,
\end{align*}
which endows $(\mc P_2(\mc M), W_2)$ with the structure of a formal Riemannian manifold. Curves ${(\mu_t)}_{t\ge 0}$ in $\mc P_2(\mc M)$ are still described by the continuity equation
\begin{align}\label{eq:mfld_cont_eq}
    \partial_t \mu_t
    + \divergence(\mu_t v_t) = 0
\end{align}
where now $v_t$ is an element of the tangent bundle $T\mc M$ and $\divergence$ denotes the divergence operator on the Riemannian manifold. Equation~\eqref{eq:mfld_cont_eq} is to be interpreted in the weak sense, i.e., for any test function $\varphi : \mc M \to \R$,
\begin{align}\label{eq:cont_eq_weak}
    \partial_t \int \varphi \, \D \mu_t
    &= \int \mf g(\nabla \varphi, v_t) \, \D \mu_t\,.
\end{align}
If ${(\mu_t)}_{t\ge 0}$ is a smooth curve such that $\mu_t$ admits a density $\rho_t$ w.r.t.\ the Riemannian volume measure, then this is equivalent to the partial differential equation (PDE)
\begin{align*}
    \partial_t \rho_t
    &= \divergence(\rho_t v_t)\,.
\end{align*}

As before, the continuity equation admits a particle interpretation: if $p_0 \sim \mu_0$ and ${(p_t)}_{t\ge 0}$ evolves via the ODE
\begin{align}\label{eq:mfld_cont_eq_particle}
    \dot p_t
    &= v_t(p_t)\,,
\end{align}
then $p_t \sim \mu_t$ for all $t\ge 0$.

Given a functional $\eu F : \mc P_2(\mc M) \to \R \cup \{\infty\}$ defined over the Wasserstein space, its gradient at $\mu$ is, by definition, the element $\nabla_{W_2} \eu F(\mu) \in T_\mu \mc P_2(\mc M)$ such that: for all curves ${(\mu_t)}_{t\in \R}$ satisfying the continuity equation~\eqref{eq:mfld_cont_eq} with $\mu_0 = \mu$, it holds that
\begin{align*}
    \partial_t\big|_{t=0} \eu F(\mu_t)
    &= \langle \nabla_{W_2} \eu F(\mu), v_0 \rangle_\mu
    = \int \mf g\bigl(\nabla_{W_2} \eu F(\mu), v_0\bigr) \, \D \mu\,.
\end{align*}
Using the continuity equation~\eqref{eq:cont_eq_weak}, it follows by direct identification that 
\begin{align*}
    \nabla_{W_2} \eu F(\mu)
    &= \nabla \delta \eu F(\mu)\,,
\end{align*}
where $\delta \eu F(\mu) : \mc M\to\R$, the first variation of $\eu F$ at $\mu$, is defined up to an additive constant and satisfies
\begin{align*}
    \partial_t\big|_{t=0} \eu F(\mu)
    &= \int \delta \eu F(\mu) \, \partial_t\big|_{t=0} \mu_t\,.
\end{align*}
A gradient flow of $\eu F$ is a curve ${(\mu_t)}_{t\ge 0}$ which satisfies the continuity equation~\eqref{eq:mfld_cont_eq} with velocity vector field $v_t = -\nabla_{W_2} \eu F(\mu_t)$, which in turn admits the particle interpretation~\eqref{eq:mfld_cont_eq_particle}.

We now consider the functional
\begin{align*}
    \eu F(\mu)
    &\deq \KL(\p_\mu \mmid \pi)
\end{align*}
and compute its first variation. Let $\mf m$ denote the Riemannian volume measure; let ${(\rho_t)}_{t\in\R}$ be a smooth curve of densities $\rho_t = \frac{\D\mu_t}{\D\mf m}$. Since
\begin{align*}
    \eu F(\mu)
    &= \int V \, \D \p_\mu + \int \p_\mu \ln \p_\mu \\
    &= \iint V \, \D p_\theta \, \rho(\theta) \, \D \mf m(\theta) + \iint \ln \Bigl( \int p_{\theta'} \, \rho(\theta') \, \D \mf m(\theta')\Bigr) \, \D p_\theta \, \rho(\theta) \, \D \mf m(\theta)
\end{align*}
then
\begin{align*}
    \partial_t \eu F(\mu_t)
    &= \iint V \, \D p_\theta \, \dot\rho_t(\theta) \, \D \mf m(\theta) + \iint \frac{\int p_{\theta'} \, \dot \rho_t(\theta') \, \D \mf m(\theta')}{\int p_{\theta'} \, \rho_t(\theta') \, \D \mf m(\theta')} \, \D p_\theta \, \rho_t(\theta) \, \D \mf m(\theta) \\
    &\qquad{} + \iint \ln \Bigl( \int p_{\theta'} \, \rho(\theta') \, \D \mf m(\theta')\Bigr) \, \D p_\theta \, \dot\rho_t(\theta) \, \D \mf m(\theta) \\
    &= \iint (V + \ln \p_{\mu_t} + 1) \, \D p_\theta \, \dot\rho_t(\theta) \, \D \mf m(\theta)\,.
\end{align*}
From this,
\begin{align*}
    \delta \eu F(\mu) : \theta \mapsto \int (V + \ln \p_\mu + 1) \, \D p_\theta
    = \int \ln \frac{\p_\mu}{\pi} \, \D p_\theta + 1\,.
\end{align*}
Next, we compute the Bures{--}Wasserstein gradient using~\eqref{eq:bw_gradient} and~\eqref{eq:gaussian_differential_identities}:
\begin{align*}
    \nabla_{\BW} \delta \eu F(\mu)(m,\Sigma)
    &= \Bigl( \int \ln \frac{\p_\mu}{\pi} \, \nabla_m p_{m,\Sigma},\; 2\int \ln \frac{\p_\mu}{\pi} \,\nabla_\Sigma p_{m,\Sigma}\Bigr) \\
    &= \Bigl( \int \nabla \ln \frac{\p_\mu}{\pi} \, \D p_{m,\Sigma},\; \int \nabla^2 \ln \frac{\p_\mu}{\pi} \,\D p_{m,\Sigma}\Bigr)\,.
\end{align*}

Finally, to derive the system of ODEs~\eqref{eq:mixture_ode}, we combine the above expression for the Wasserstein gradient of $\eu F$ together with the particle interpretation~\eqref{eq:mfld_cont_eq_particle} and the equations~\eqref{eq:evolution_mean} and~\eqref{eq:evolution_cov} for dynamics on the Bures{--}Wasserstein space.

\section{Lack of convexity of the KL divergence for mixtures of Gaussians}\label{scn:lack_of_cvxty}

In this section, we provide counterexamples for the lack of convexity of the objective functional $\mu\mapsto \eu F(\mu) = \KL(\p_\mu \mmid \pi)$ on the space $\mc P_2(\BW(\R^d))$.

First, we point out that even when $\pi$ is strongly log-concave, the functional $\eu F$ can be badly behaved.
For example, if $\pi = p_{0,1} = \mc N(0, 1)$ is a Gaussian of variance $1$, then we can write it as a Gaussian mixture in many ways: $\pi = \int \mc N(m, a) \, \D \nu_{1-a}(m)$ for any $a \in [0,1]$, where $\nu_a = \mc N(0, a)$. In particular, the set of minimizers of $\eu F$ is not a singleton, and includes all of the measures $\nu_a \otimes \delta_a$ ($(m,\sigma^2)$ is a random pair with independent components, where $m \sim \normal(0, 1-a)$ and $\sigma^2=a$ almost surely) for $a \in [0,1]$ (as well as all convex combinations---i.e., mixtures---thereof).

Next, we give an explicit example which demonstrates the lack of convexity of the entropy functional $\mu \mapsto \eu H(\p_\mu) \deq \int \p_\mu \ln \p_\mu$. This can be understood as the KL divergence with zero potential ($V = 0$).
Note that the entropy functional $\eu H$ is convex on $\mc P_2(\R^d)$~\citep[Section 9.4]{ambrosio2008gradient}, but our claim is that its composition with the map $\mu\mapsto \p_\mu$ is not convex on $\mc P_2(\BW(\R^d))$.

In one dimension let $\mu_0 = \mc N(0, 1) \otimes \delta_1$ and $\mu_1 = \mc N(0, \tau^2) \otimes \delta_1$. In words, a random pair $(m_0, \sigma_0^2)$ drawn from $\mu_0$ satisfies $m_0 \sim \mc N(0, 1)$ and $\sigma_0^2 = 1$, and similarly for $\mu_1$. What is the optimal coupling of $\mu_0$ and $\mu_1$? Clearly $\sigma_0^2 = \sigma_1^2 = 1$ is the trivial coupling, and since the Bures--Wasserstein distance over the means is the same as the Euclidean distance between the means, we want the usual $W_2$ optimal coupling between $\normal(0, 1)$ and $\normal(0, \tau^2)$; it follows that $m_1 = \tau m_0$. Hence, the Bures geodesic between is $\{(m_t, \sigma_t^2) = ((1-t + t\tau) \, m_0, 1)\}_{t \in [0,1]}$; equivalently the (Bures--)Wasserstein geodesic between $\mu_0$ and $\mu_1$ is $\{\mu_t = \normal(0, {(1-t+t\tau)}^2) \otimes \delta_1\}_{t \in [0,1]}$.

Next, recall that the Gaussian mixture $\p_{\mu_t}$ is the law of $X$ drawn in the two-stage procedure: first we draw $(m_t,\sigma_t^2) \sim \mu_t$, and given $(m_t,\sigma_t^2)$ we draw $X \sim p_{m_t,\sigma_t^2}$. Thus, 
$$
\p_{\mu_t} = \int \normal(m, \sigma^2)\,\D \mu_t(m, \sigma^2)=\int \normal(m, 1)\,\D \nu_{{(1-t+t\tau)}^2}(m)
= \normal(0, 1+ {(1-t+t\tau)}^2)\,.
$$
Hence,
$$
 \eu H(\p_{\mu_t})=  \int \p_{\mu_t} \ln \p_{\mu_t} =-\frac12 \ln(2\pi e) - \frac12 \ln\bigl(1+{(1-t+t\tau)}^2\bigr)\,.
$$
Then, the convexity of $t\mapsto \eu H(\p_{\mu_t})$ is equivalent to the convexity of $t\mapsto -\ln(1+ {(1-t+t\tau)}^2)$, which fails when, e.g., $\tau = 1/2$; in that case, the function is, in fact, concave on the interval $[0,1]$.

\section{The Wasserstein--Fisher--Rao gradient flow}\label{scn:wfr}

Similarly to the setting in Section~\ref{scn:mixtures}, here we identify probability measures $\mu$ over the Bures{--}Wasserstein space with the corresponding Gaussian mixture $\p_\mu$. The aim of this section is to derive the gradient flow of the KL divergence $\mu\mapsto \KL(\p_\mu \mmid \pi)$, except we now equip the space $\mc P_2(\BW(\R^d))$ with the Wasserstein--Fisher--Rao geometry~\citep{lieromielkesavare2016wfr1, chizatetal2018wfr, lieromielkesavare2018wfr2}. Deriving the gradient flow with respect to this geometry leads to dynamics for a system of interacting Gaussian particles in which the weight of each particle is also updated at each iteration.

\subsection{Background on Wasserstein--Fisher--Rao geometry}

Here we briefly summarize the relevant background on the Wasserstein--Fisher--Rao (WFR) geometry. The WFR metric is also called the \emph{Hellinger--Kantorovich} metric by some authors.

\paragraph{The Fisher{--}Rao metric.} The Fisher{--}Rao metric is a metric on the space $\mc M_+(\R^d)$ of positive measures (not necessarily probability measures). It is the induced metric on $\mc M_+(\R^d)$ if we enforce that the mapping $\mu\mapsto \sqrt\mu$ (defined for smooth probability densities $\mu$) is an isometry into $L^2(\R^d)$. This means that
\begin{align*}
    \msf d_{\msf{FR}}^2(\mu_0,\mu_1)
    &= \int (\sqrt{\mu_0} - \sqrt{\mu_1})^2\,,
\end{align*}
and if $\mu_0$ and $\mu_1$ are probability measures then this is known to statisticians (up to a constant factor) as the squared Hellinger distance.
(If we apply the analogous procedure to discrete probability measures, then this amounts to identifying the simplex with a subset of the unit sphere.) The Fisher{--}Rao metric is well-studied in the field of information geometry~\citep{AmaNag00, ayetal2017infogeometry}.

Next, we describe the Riemannian geometry underlying the Fisher{--}Rao metric.
Consider a curve $t\mapsto \mu_t$ of positive measures with time derivative $\dot \mu$. Since the Fisher{--}Rao metric endows the square root of the density with a Hilbert metric, we place endow the time derivative of the square root, $\dot{\sqrt \mu} = \dot \mu/(2\sqrt\mu)$, with the Hilbert norm $\norm{\dot \mu/(2\sqrt \mu)}_{L^2(\R^d)}$. Thus, the norm at the tangent space $T_\mu \mc M_+(\R^d)$ is given by
\begin{align*}
    \norm{\dot \mu}_\mu^2
    &= \int \frac{\dot \mu^2}{4\mu}\,.
\end{align*}

Actually, because we are working with positive measures (called \emph{unbalanced} measures to distinguish from the usual optimal transport problem which requires the measures to have the same total mass), this kind of geometry is useful for studying problems in which the total mass changes over time.
For example, PDEs of the form $\partial_t \mu_t = \alpha_t \mu_t$ are called reaction equations because they describe, e.g., how the concentration of a chemical changes over time in reaction to the environment. Motivated by this application, we parameterize $\dot \mu$ via $\dot \mu = \alpha \mu$, in which case the norm is
\begin{align}\label{eq:fr_norm}
    \norm{\alpha}_\mu^2
    &= \frac{1}{4} \int \alpha^2 \, \D \mu\,.
\end{align}

\paragraph{Wasserstein geometry.} We recall from Section~\ref{scn:otto} that Wasserstein geometry is motivated by a completely different class of PDEs, namely \emph{transport equations} encoded by the continuity equation
\begin{align*}
    \partial_t \mu_t + \divergence(\mu_t v_t) = 0\,,
\end{align*}
which describe the evolving law of a particle $x_t$ tracing out an integral curve of the family of vector fields: $\dot x_t = v_t(x_t)$. The Riemannian structure is obtained by equipping the tangent space $T_\mu \mc P_2(\R^d)$ with the norm
\begin{align*}
    \norm v_\mu^2
    &= \int \norm v^2 \,\D \mu\,.
\end{align*}

\paragraph{Wasserstein{--}Fisher{--}Rao geometry.} Next we combine the two geometric structures, which can model transport-reaction equations such as
\begin{align}\label{eq:continuity_reaction}
    \partial_t \mu_t + \divergence(\mu_t v_t) = \alpha_t \mu_t\,.
\end{align}
The tangent space norm is then given by the combination combination
\begin{align*}
    \norm{(\alpha, v)}_\mu^2
    &= \int ( \alpha^2 + \norm v^2) \, \D \mu\,.
\end{align*}
(At this point some authors add a factor $\frac{1}{4}$ in front of the $\alpha^2$, which is natural in view of~\eqref{eq:fr_norm}. This is convenient for studying geometric properties of the space, but it is not necessary for our purposes.) As in the pure Fisher{--}Rao case, this is a metric on the space of positive measures $\mc M_+(\R^d)$.

It induces the distance
\begin{align*}
    \WFR^2(\mu_0,\mu_1)
    &\deq \inf\Bigl\{ \int_0^1 \norm{(\alpha_t,v_t)}_{\mu_t}^2 \, \D t \Bigm\vert {(\mu_t,\alpha_t,v_t)}_{t\in [0,1]}~\text{solves}~\eqref{eq:continuity_reaction}\Bigr\}\,.
\end{align*}

One can show that the tangent space to $\mc M_+(\R^d)$ consists of pairs $(\alpha, v)$ for which $\alpha = u$ and $v = \nabla u$ for some function $u : \R^d\to\R$. Thus, compared to the Wasserstein metric in which the tangent space norm is the $\dot H^1(\mu)$ norm $\norm u_{\dot H^1(u)} = \norm{\nabla u}_{L^2(\mu)}$, the Wasserstein{--}Fisher{--}Rao metric has the interpretation of completing the tangent space norm to the full Sobolev norm $H^1(\mu)$.

\paragraph{Constraining the dynamics to lie within probability measures.}
In order to have our dynamics stay on the space of probability measures, we follow~\citet{lulunolen2019birthdeath} and consider instead the equation
\begin{align*}
    \partial_t \mu_t + \divergence(\mu_t v_t) = \Bigl(\alpha_t - \int \alpha_t \, \D \mu_t\Bigr) \, \mu_t\,,
\end{align*}
which now conserves mass. The tangent space norm is modified to read
\begin{align*}
    \norm{(\alpha, v)}_\mu^2
    &= \int \Bigl[ \Bigl(\alpha - \int \alpha \, \D \mu\Bigr)^2 + \norm v^2\Bigr] \, \D \mu\,.
\end{align*}

\paragraph{Particle interpretation.}
The particle interpretation of the WFR geometry is more complicated to state than for the Wasserstein geometry, but it can be done. Instead of considering a particle $x$, we consider a pair $(x,r)$ consisting of a particle $x\in\R^d$ and a number $r > 0$ (this number is actually interpreted as the \emph{square root} of the mass of the particle). The pair $(x,r)$ should be thought of as an element of the cone space $\mf C(\R^d) \deq (\R^d\times \R_+)/(\R^d \times \{0\})$ (in other words, we take the space $\R^d\times \R_+$ and identify all of the points with zero mass which sit at the ``tip of the cone''). The cone space is the natural setting for WFR geometry; for example, one can introduce a metric on $\mf C(\R^d)$ and show that the WFR distance is an optimal transport problem w.r.t.\ this metric. We will not go into such detail, but nevertheless we introduce the cone space because is important for the particle interpretation of WFR dynamics.

Curves of measures ${(\mu_t)}_{t\in [0,1]}$ in the WFR geometry admit a particle interpretation in terms of trajectories on $\mf C(\R^d)$. Namely, the equation~\eqref{eq:continuity_reaction} can be interpreted as follows.
There exists a curve of measures $t\mapsto \widetilde\mu_t$ over the cone space $\mf C(\R^d)$, such that if $r : \mf C(\R^d) \to \R_+$ denotes the mapping $(x,r) \mapsto r$, and $x : \mf C(\R^d) \to \R^d$ maps $(x, r) \mapsto x$, then
\begin{align*}
    \mu_t
    &= x_\# (r^2 \widetilde \mu_t)\,.
\end{align*}
Moreover, if we draw $(x_0, r_0) \sim \widetilde \mu_0$ and follow the ODEs
\begin{align*}
    \dot x_t
    &= v_t(x_t)\,, \\
    \dot r_t
    &= \Bigl( \alpha_t(x_t) - \int \alpha_t \, \D \mu_t\Bigr) \, r_t\,,
\end{align*}
then $(x_t,r_t) \sim \widetilde \mu_t$. Here the notation $\sim$ is an (egregious) abuse of notation because $\widetilde \mu_t$ is not a probability measure; by $(x,r) \sim \widetilde \mu$ more precisely we mean that $\widetilde \mu_t = (\msf{ODE}_t)_\# \widetilde \mu_0$ where $\msf{ODE}_t$ is the solution mapping $(x_0,r_0) \mapsto (x_t,r_t)$ to the above system of ODEs at time $t$.

To make this interpretation more concrete, we specialize to the case of discrete measures. Suppose that we start at a probability measure
\begin{align*}
    \mu_0
    &= \sum_{i=1}^N w^{(i)}_0 \delta_{x^{(i)}_0}\,.
\end{align*}
Then, we lift to the cone space:
\begin{align*}
    \widetilde \mu_0
    &= \sum_{i=1}^N \delta_{(x^{(i)}_0, \sqrt{w^{(i)}_0})}
    = \sum_{i=1}^N \delta_{(x^{(i)}_0, r^{(i)}_0)}
\end{align*}
where we set $r_t^{(i)} = \sqrt{w_t^{(i)}}$. Next, we follow the ODEs
\begin{align*}
    \dot x_t^{(i)}
    &= v_t(x_t^{(i)})\,, \\
    \dot r_t^{(i)}
    &= \Bigl(\alpha_t(x_t^{(i)}) - \sum_{j=1}^N w_t^{(j)} \alpha_t(x_t^{(j)}) \Bigr) \, r_t^{(i)}\,.
\end{align*}
Upon projecting back to the base space, we obtain another discrete measure
\begin{align*}
    \mu_t
    &= \sum_{i=1}^N w^{(i)}_t \delta_{x^{(i)}_t}
    = \sum_{i=1}^N (r^{(i)}_t)^2 \, \delta_{x^{(i)}_t}\,.
\end{align*}
As a sanity check, we check that these dynamics ensure that $\mu_t$ is a probability measure for all $t$. The time derivative of the sum of the weights is
\begin{align*}
    \partial_t \sum_{i=1}^N w^{(i)}_t
    &= 2 \sum_{i=1}^N r^{(i)}_t \,\partial_t r^{(i)}_t
    = 2 \sum_{i=1}^N (r^{(i)}_t)^2 \, \bigl( \alpha_t(x_i^{(t)}) - \E_{\mu_t} \alpha_t\bigr) \\
    &= 2 \,\Bigl(\sum_{i=1}^N w^{(i)}_t \alpha_t(x_i^{(t)})- \E_{\mu_t}\alpha_t\Bigr)
    = 0\,.
\end{align*}

\subsection{Derivation of the gradient flow}

Next, we derive the Wasserstein{--}Fisher{--}Rao gradient flow of the functional $\mu\mapsto\eu F(\mu) \deq \KL(\msf p_\mu \mmid \pi)$ on the space $(\mc P_2(\BW(\R^d)), \WFR)$ of Gaussian mixtures equipped with the Wasserstein{--}Fisher{--}Rao metric (over the Bures{--}Wasserstein space).
The WFR gradient of $\eu F$, $\nabla_{\WFR} \eu F(\mu)$, is the pair
\begin{align*}
    \nabla_{\WFR} \eu F(\mu)
    &= \Bigl( \nabla_{\BW} \delta \eu F(\mu),\; \delta \eu F(\mu) - \int \delta \eu F(\mu) \, \D \mu\Bigr)\,.
\end{align*}
This result is essentially stated as~\citet[Proposition A.1]{lulunolen2019birthdeath}, although we have generalized the formula to hold when the base space is no longer $\R^d$. Note also that we have already calculated the first variation of $\eu F$, as well as the BW gradient, in Section~\ref{scn:pf_mixtures}.

The interpretation of the formula is that in the gradient flow of $\eu F$, we have a particle $(m,\Sigma)$ associated with some mass $w$ evolving according to
\begin{align*}
    \dot m
    &= -\E_{p_{m,\Sigma}} \nabla \ln \frac{\p_\mu}{\pi}\,, \\
    \dot \Sigma
    &= -\Sigma \E_{p_{m,\Sigma}} \nabla^2 \ln \frac{\p_\mu}{\pi} - \E_{p_{m,\Sigma}} \nabla^2 \ln \frac{\p_\mu}{\pi} \, \Sigma\,, \\
    \dot r
    &= -\Bigl( \E_{p_{m,\Sigma}}\ln \frac{\p_\mu}{\pi} - \E_{\p_{\mu}} \ln \frac{\p_\mu}{\pi} \Bigr) \, r\,,
\end{align*}
where $r = \sqrt w$. The interpretation may be clearer in the discrete case, so suppose that we initialize the dynamics at a discrete measure
\begin{align*}
    \mu_0
    &= \sum_{i=1}^N w_0^{(i)} \delta_{(m^{(i)}_0, \Sigma^{(i)}_0)}\,.
\end{align*}
Next we solve the coupled system of ODEs, for $i\in [N]$,
\begin{align*}
    \dot m^{(i)}
    &= -\E_{p_{m^{(i)},\Sigma^{(i)}}} \nabla \ln \frac{\p_\mu}{\pi}\,, \\
    \dot \Sigma^{(i)}
    &= -\Sigma^{(i)} \E_{p_{m^{(i)},\Sigma^{(i)}}} \nabla^2 \ln \frac{\p_\mu}{\pi} - \E_{p_{m^{(i)},\Sigma^{(i)}}} \nabla^2 \ln \frac{\p_\mu}{\pi} \, \Sigma^{(i)}\,, \\
    \dot r^{(i)}
    &= -\Bigl( \E_{p_{m^{(i)},\Sigma^{(i)}}}\ln \frac{\p_\mu}{\pi} - \E_{\p_{\mu}} \ln \frac{\p_\mu}{\pi} \Bigr) \, r^{(i)}\,,
\end{align*}
where $r^{(i)} = \sqrt{w^{(i)}}$ and
\begin{align*}
    \mu_t
    &= \sum_{i=1}^N w_t^{(i)} \delta_{(m^{(i)}_t, \Sigma^{(i)}_t)}\,.
\end{align*}
Since the normalization constant of $\pi$ cancels out in the above equations, they are implementable without this knowledge.

\section{Experiments for Gaussian VI}\label{scn:xp-vi}

The goal of the present section is to conduct numerical experiments that illustrate the convergence of the Gaussian distribution corresponding to the ODE~\eqref{eq:sarkka}  to an approximation of the target distribution.  We consider two kinds of targets:  a  mixture  of two Gaussians, and   a log-concave target  that corresponds to the likelihood function in logistic regression.

\subsection{Setup} \label{gauss:target}
\subsubsection{Definition of  the target distributions}
\textbf{Bimodal target: mixture of two Gaussians} \\
We define a bimodal target as a  mixture of two Gaussians $\pi=\frac{1}{2}\, \mathcal{N}(\mu_1,\Sigma_1)+\frac{1}{2} \,\mathcal{N}(\mu_2,\Sigma_2)$ where $\Sigma_1$ and $\Sigma_2$ have non isotropic covariances with a ratio of $3$ between the largest and the smallest eigenvalues.

\medskip{}

\noindent{}\textbf{Log-concave target: Bayesian logistic regression} \\
The proposed log-concave target is generated in the context of the Bayesian  treatment of logistic regression   associated with a two-class synthetic dataset $\mathcal{D}= \{(x_i, y_i) : i = 1,\dotsc,N\}$. The probability of the binary  label $y_i\in\{0,1\}$ given the  corresponding covariate $x_i$ and parameter  $z \in \R^d$ is defined by the following Bernoulli distribution:
\begin{align}\label{eq:logistic_model}
    \pi(y_i|x_i, z)=\sigma(x_i^\T z)^{y_i}\,(1-\sigma(x_i^\T z))^{1-y_i}\,,
\end{align}
where $\sigma(x)=1/(1+\exp(-x))$ is the logistic function. We define the target distribution as the posterior associated to data $\mc D$ starting from an uninformative (flat) prior on $z$, that is, 
\begin{align}
\pi(z|\mathcal{D})=\frac{1}{Z}\prod_{i=1}^N \pi(y_i|x_i, z)\,, \label{logisticTarget}
\end{align} 
with $Z$   the normalization constant.  
The Langevin dynamics are associated with the gradient of $V$ then defined by: $$-\nabla V(z) = \nabla \log \pi(z|\mathcal{D})= \sum_{i=1}^N (y_i-\sigma(x_i^\T z))\,x_i\,.$$

To generate the synthetic data $\mathcal{D}$, we randomly draw labels   $y_i\in\{0,1\}$ and for the problem to be well-specified we have drawn the class-conditional covariates $x_i$ from Gaussian distributions   $\mc N(m_{y_i}^*, \Sigma^*)$ with   $m_1^*=-m_0^*=m^*$. We call $s$ the separation factor defined by  $\norm{m_1^* - m_0^*}=\norm{2m^*} = s$. For illustrative purposes we also plot Fisher's linear discriminant vector defined by  $z^* = 2\Sigma^{*-1}m^*$~\citep[see][chapter 4]{Bishop06}. An example of the generated data is displayed in Figure~\ref{logistic2D}.

\begin{figure}[!h]
\centering
\includegraphics[scale=0.6]{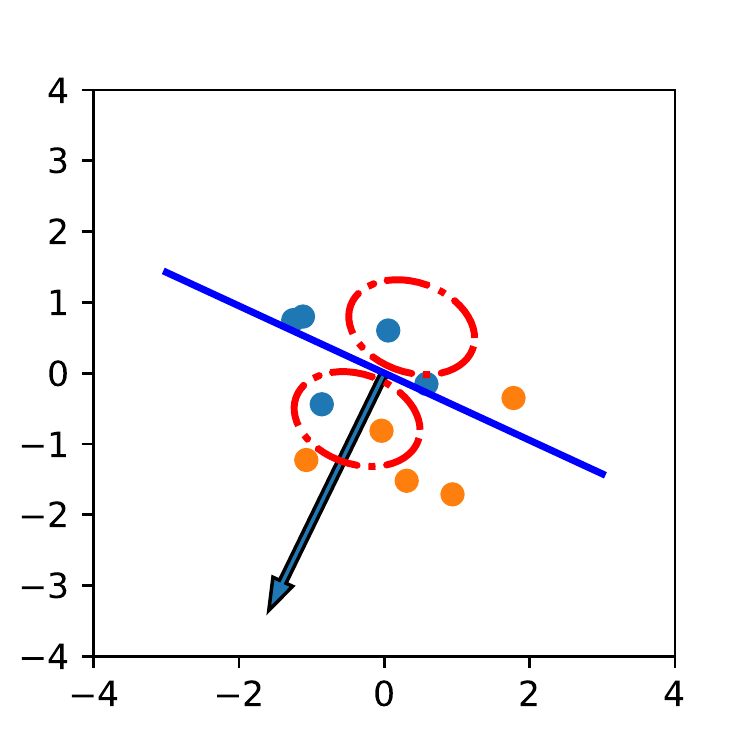}
\includegraphics[scale=0.6]{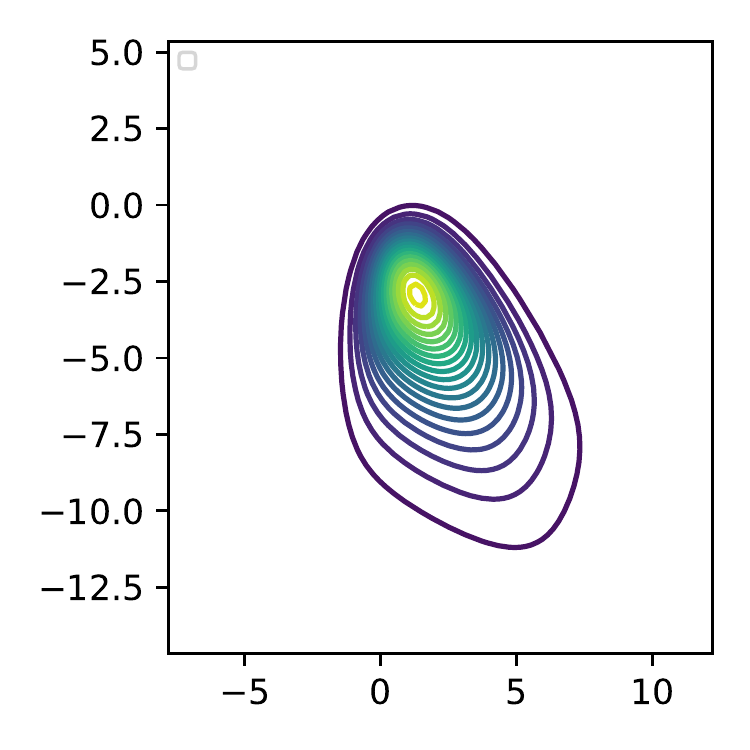}
\caption{The dataset $\mathcal{D}$ (left figure) used to generate the target distribution (right figure). The two Gaussians of equal covariance $\Sigma^*$ from which the covariates are generated are shown as red ellipsoids. The arrow represents Fisher's linear discriminant $z^*=  2\Sigma^{*-1}m^*$.}
\label{logistic2D}
\end{figure} 

\subsubsection{Evaluation of the KL divergence for the proposed log-concave targets} 
 
The target distribution~\eqref{logisticTarget} may be written $\pi (z|\mathcal{D})=\frac{1}{Z}\, \tilde\pi (z|\mathcal{D})$ where $\tilde\pi (z|\mathcal{D})=\prod_{i=1}^N\pi(y_i|x_i, z)$ is the unnormalized distribution. The divergence between any Gaussian distribution $p=\mathcal{N}(m,\Sigma)$ and the target then writes
\begin{align} 
\KL(p(z)\mmid \pi (z|\mathcal{D}))
&= \int p(z) \ln  \frac{p(z)}{\pi (z|\mathcal{D})} \, \D z=  \int p(z) \ln  \frac{p(z)}{\tilde\pi (z|\mathcal{D})} \, \D z+ \ln Z \\
&=-\int p(z) \ln \tilde\pi (z|\mathcal{D}) \, \D z + H(p) + \ln Z  
\end{align}
where $H(p)$ is the negative entropy of a Gaussian distribution for which a closed-form expression is known. We will see shortly we can   approximate the expectation under the Gaussian $p$ as follows
\begin{align}
\KL(p(z)\mmid \pi (z|\mathcal{D}))\approx - \sum_{i=1}^{K} \alpha_i \ln \tilde\pi (m + c_iRe_i|\mathcal{D}) + H(p) + \ln Z\,,\label{lnZ:eq}
\end{align}
using $K=2d$ sigma points  with cubature rules $(\alpha_i,c_i)=(\frac{1}{2d},\sqrt{d})$  for all $i$, 
and where $R$ is defined via the Cholesky decomposition $RR^\T=\Sigma$ (see Section~\ref{scn:numerical} for details).

\subsubsection{The Laplace approximation as a baseline} \label{Laplace}

We use the  widespread Laplace approximation~\citep[see][chapter 4]{Bishop06}  as a baseline for comparisons. 
In dimension 2, we compute the normalization constant $Z$   of~\eqref{logisticTarget}   using a grid. When we  turn to high dimension, normalization becomes intractable. However,
we may still compare our algorithm with  Laplace approximation  as follows.
Since our goal is mainly to illustrate the convergence of our algorithm using Laplace approximation as a baseline, we may choose an arbitrary value for the normalization constant $Z$ when evaluating the divergence to the target in equation~\eqref{lnZ:eq}.
This allows for comparison of the KL divergence between the approximating distribution---given by either Gaussian \vi{} or Laplace approximation---and the target $\pi$ up to the same additive constant for both methods.
By default, we let $Z=1$, but we sometimes use larger values of $Z$ in order to avoid plotting negative values for the unnormalized KL (albeit an  arbitrary choice).

To obtain the Laplace approximation, we first compute a mode of the target distribution $\pi$.
Once the mode $z_0$ has been found, we consider the following Taylor approximation around the mode:
\begin{align}
\ln \pi(z) \approx \ln \pi(z_0) - \frac{1}{2} \,(z-z_0)^\T H(z-z_0)\,,
\end{align}
where $H$ is the Hessian of the negative log-likelihood around $z_0$ defined by $H=\nabla^2 \log \frac{1}{\pi}(z_0)$.
Renormalizing, this yields the approximation
\begin{align}
\pi \approx \hat\pi^{\msf{Laplace}}
=\mathcal{N}(z_0,H^{-1})\,.
\end{align}
In our experiments, we use the L-BFGS algorithm \citep{Liu89} to find the mode $z_0$.

\subsection{Implementation} \label{scn:numerical}
 We follow~\citet{Sarkka07, Lambert22b} to compute   the expectations involved in  equation~\eqref{eq:sarkka}  using  quadrature rules.  We then numerically integrate the set of coupled ODEs in equation~\eqref{eq:sarkka} using a fourth-order Runge--Kutta method.
As a first step, we introduce a method to  enforce that the  covariance matrix $\Sigma$ remains symmetric and positive at all times.
\begin{itemize}
\item \textbf{Covariance matrices in square root form:} To numerically enforce that the covariance matrix $\Sigma$ remains symmetric and positive at each step, as is customary in the Kalman filtering literature, we consider a continuous-time ``square-root'' form of the covariance as developed in~\citet{Morf77} and applied in~\citet{Sarkka07}.
Let $R$ be a lower triangular matrix such that $\Sigma = RR^\T$. An ODE for $R$ is obtained as follows.  
\begin{align}
&\dot \Sigma = \dot R R^\T + R \dot R^\T 
\end{align}
Multiplying by $R^{-1}$ on the left and $R^{-\T}$ on the right yields:
\begin{align}
&R^{-1}\dot R+\dot R^\T R^{-\T}=R^{-1}\,\dot \Sigma\,R^{-\T}\,.
\end{align}As $R^{-1}\dot R+\dot R^\T R^{-\T}= R^{-1}\dot R+ (R^{-1}\dot R)^\T$, the solution is given by:
\begin{align}
R^{-1}\dot R
&=\mathrm{Tria}(R^{-1} \, \dot \Sigma\, R^{-\T})\,,\\
\dot R
&=R\, \mathrm{Tria}(R^{-1}\, \dot \Sigma\, R^{-\T})\,,
\end{align}
where $\mathrm{Tria}(A)$ gives the lower triangular matrix $L$ corresponding to $A$ such that $A=L+L^\T$ where $L_{i,i}=\frac{1}{2}\,A_{i,i}$, $L_{i,j}=A_{i,j}$ if $i>j$, and $L_{i,j}=0$ otherwise. Letting $\dot \Sigma$ be as in ~\eqref{eq:sarkka}, this yields an   ODE in terms of the square root factor $R$.
\item \textbf{Computing expectations:} 
We compute Gaussian expectations using a quadrature rule based on $2d$ sigma points $x_1,\dots,x_{2d}$~\citep{Julier04}:
\begin{align*}
&\mathbb{E}_{p_{m,\Sigma}}[f(x)] \approx  \sum_{n=1}^{2d} \alpha_n f(x_n)\,,
\end{align*}
where the sigma points are distributed according to $x_n= m + c_n Re_n$, where $RR^\T = \Sigma$, $e_n|_{n=1,\dots,d}$ is a basis, and $e_n|_{n=d+1,\dots,2d}$ is its negative. Many variants exist to choose $\alpha_n$ and $c_n$; here, we consider the cubature  points of~\citet{Arasaratnam09} defined by $\alpha_n=\frac{1}{2d}$ and $c_n=\sqrt{d}$ which are well-adapted for Gaussian integration.
\end{itemize}

\subsection{Results in dimension 2} 
We first conduct experiments in dimension $2$ to easily visualize the true posterior (normalization is performed using a discrete grid of size $100 \times 100$).

\subsubsection{Trajectories generated by numerical integration of the ODEs} 

In Figure~\ref{XP1}, we see that Gaussian \vi{} converges quickly to one mode of the bimodal target, and to the unique mode of the logistic target.
As shown in Figure~\ref{XP2}, the results still hold if we choose a larger step size for the Runge--Kutta scheme.

\begin{figure}[!h]
\centering
\includegraphics[scale=0.5]{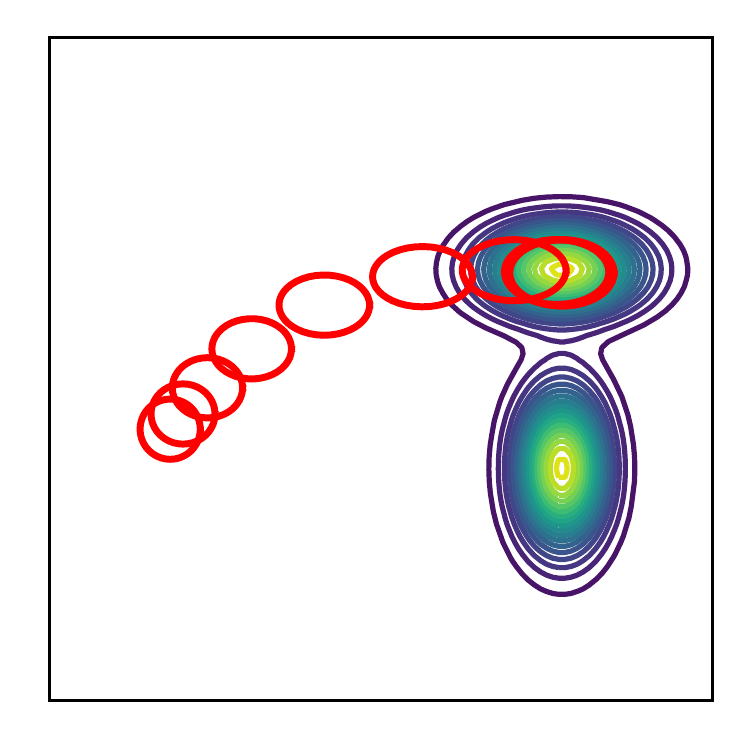}
\includegraphics[scale=0.5]{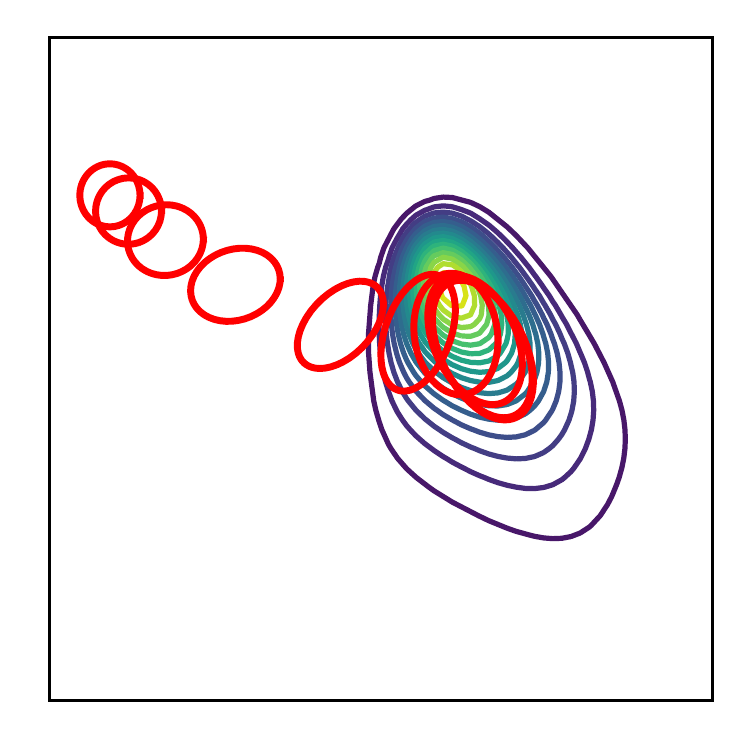}
\caption{Approximation of a bimodal target (left) and a logistic target (right). We use a Runge--Kutta scheme with step size $0.1$ and a time duration of $T=30$ (i.e., $300$ steps). The ellipsoids represent the Gaussian computed at successive steps.}
\label{XP1}
\end{figure} 

\begin{figure}[!h]
\centering
\includegraphics[scale=0.5]{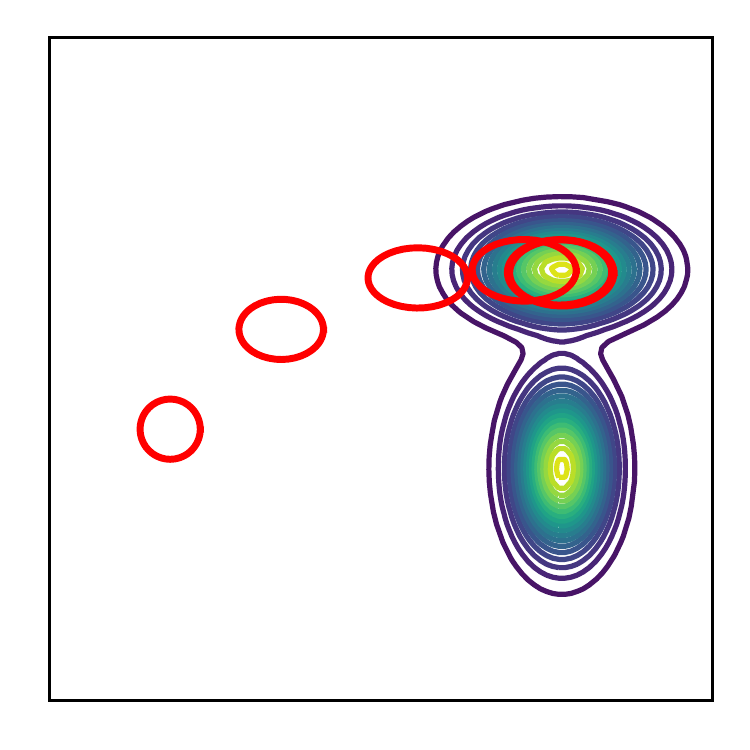}
\includegraphics[scale=0.5]{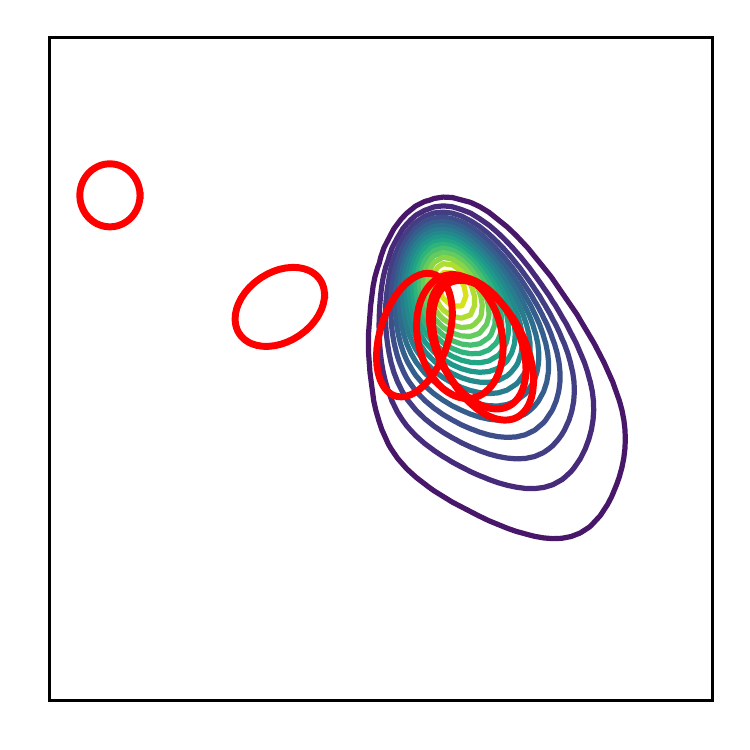}
\caption{Same as Figure~\ref{XP1} with a larger Runge--Kutta step size $1$ (i.e., 30 steps). In both cases, the algorithm converges to the same approximation as in Figure~\ref{XP1}.}
\label{XP2}
\end{figure} 

\subsubsection{Comparison with the Laplace approximation}

We compare Gaussian \vi{} with the Laplace approximation on the logistic target in dimension $2$ for the setting described in Section~\ref{gauss:target} with an arbitrary  $\Sigma^*$ and $N=10$.
We plot the convergence speed of our algorithm for Gaussian \vi{} in Figure~\ref{LapVI2d} for separation parameters $s=1.5$ and $s=2$, the latter corresponding to a sharper density. Gaussian \vi{} converges very fast and produces a better approximation of the target in terms of KL divergence than the Laplace approximation.

\begin{figure}[!h]
\centering
\includegraphics[scale=0.6]{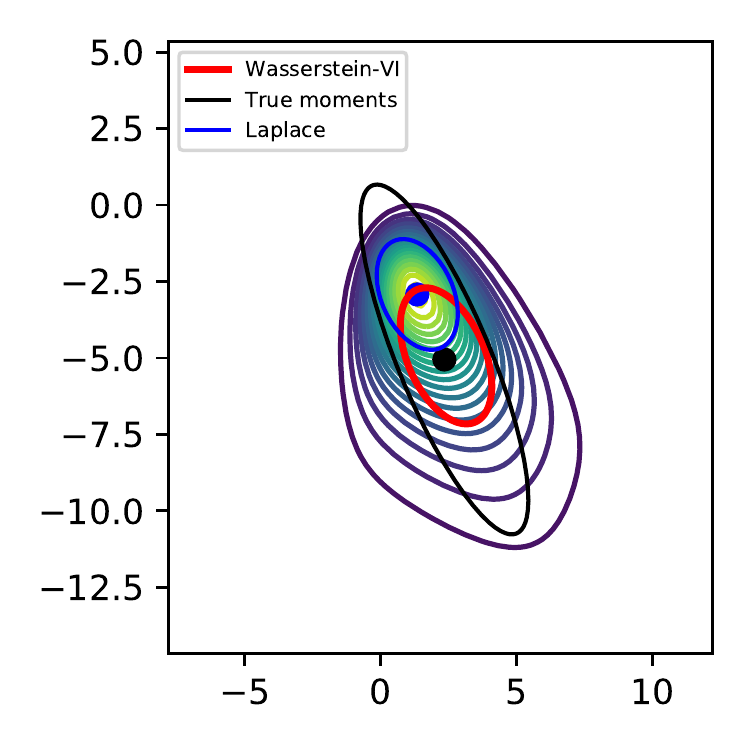}
\includegraphics[scale=0.6]{imgsNeurIPS/GaussianVI-d2-KL}
\includegraphics[scale=0.6]{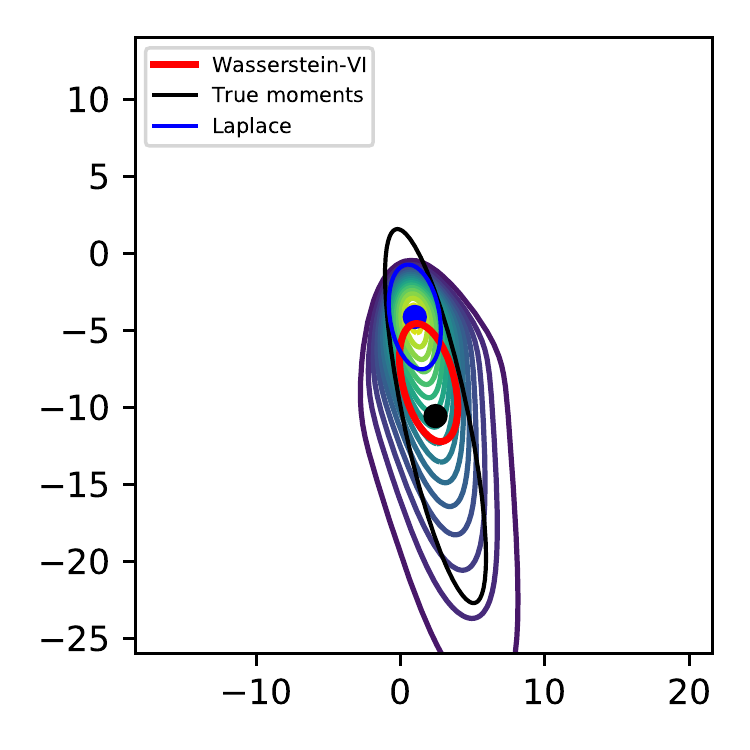}
\includegraphics[scale=0.6]{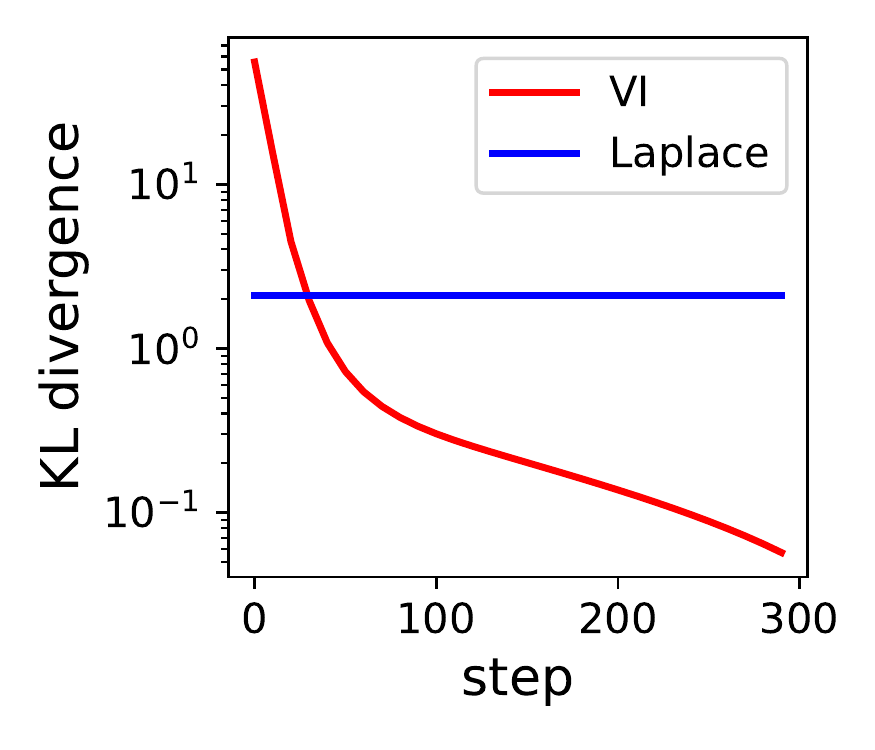}
\caption{Results in dimension $d=2$, $N=10$ for a separation factor $s=1.5$ (upper row) and $s=2$ (lower row). The left column  shows the true density via contour lines, the true mean (black dot) and covariance (black ellipsoid), and the results of the Laplace and Wasserstein \vi{} approximations as blue and red ellipsoids respectively. The right column shows the evolution of the left  KL divergence for Gaussian \vi{} on a logarithmic scale. The corresponding KL divergence obtained with Laplace approximation is shown as a blue straight line.
}
\label{LapVI2d}
\end{figure}

\subsection{Results in higher dimensions} \label{scn:xphd}
We now compare Gaussian \vi{} with the Laplace approximation on the logistic target in dimension $d=10$    and $d=100$. We consider the setting described in Section~\ref{gauss:target} where we let   $\Sigma^* =\frac{1}{d}\,I$, to have consistent norms of the inputs accross dimensions.

For Gaussian \vi{} in high dimension, we find that a step size $1$ for the Runge--Kutta integration method is too large and leads to singular covariance matrices. We thus take the step size equal  to $0.1$.
The initial Gaussian is taken to be $\mc N(0, 100 I)$, to better cover  regions of low density initially.

Results are shown in Figures~\ref{logHD1}
 and~\ref{logHD3} in dimension $d=10$ and $100$ respectively.  Gaussian \vi{} converges very fast and always produces a better approximation of the target in terms of KL divergence than the Laplace approximation. Note that the Laplace approximation can have a very high left KL divergence when the target distribution is sharp (i.e., when the two classes are well-separated). This is because the  Gaussian approximation computed with the Laplace method tends to spill out of the target distribution in region of very low densities. 

\begin{figure}[!h]
\centering
\includegraphics[scale=0.6]{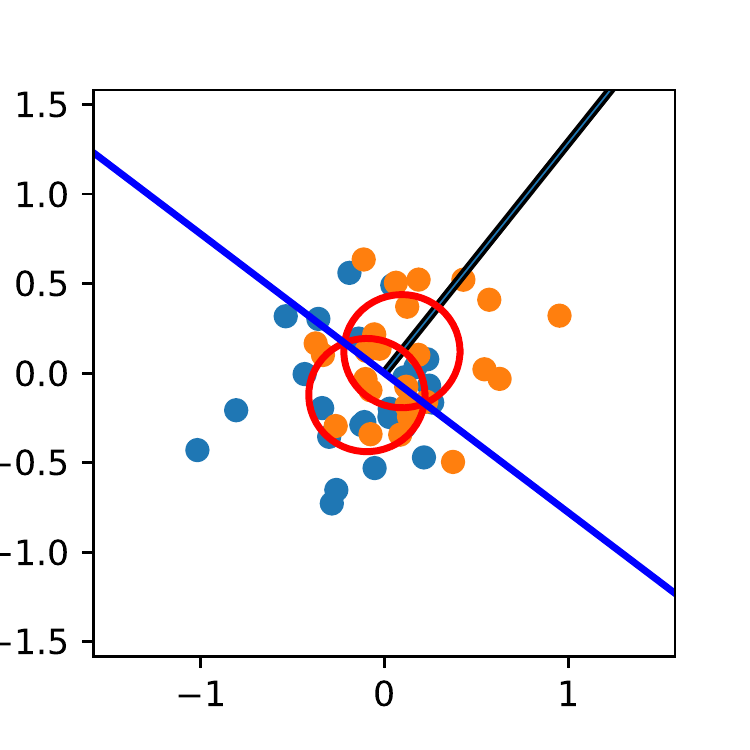}
\includegraphics[scale=0.6]{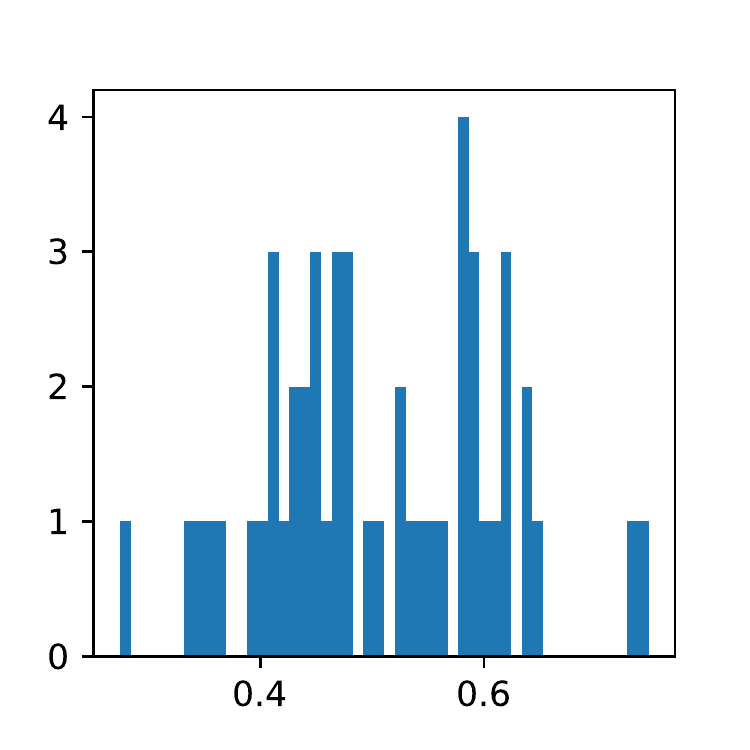}
\includegraphics[scale=0.6]{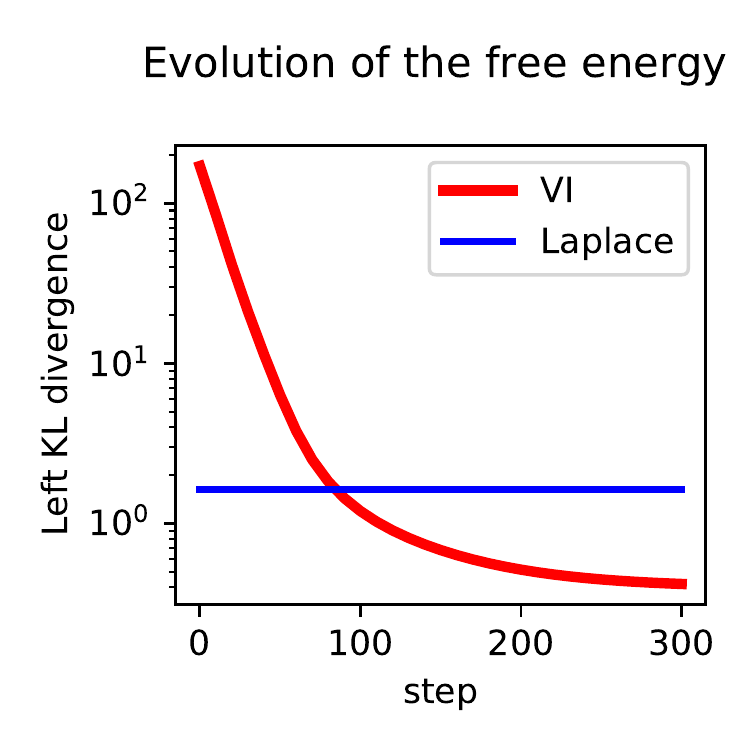}
\includegraphics[scale=0.6]{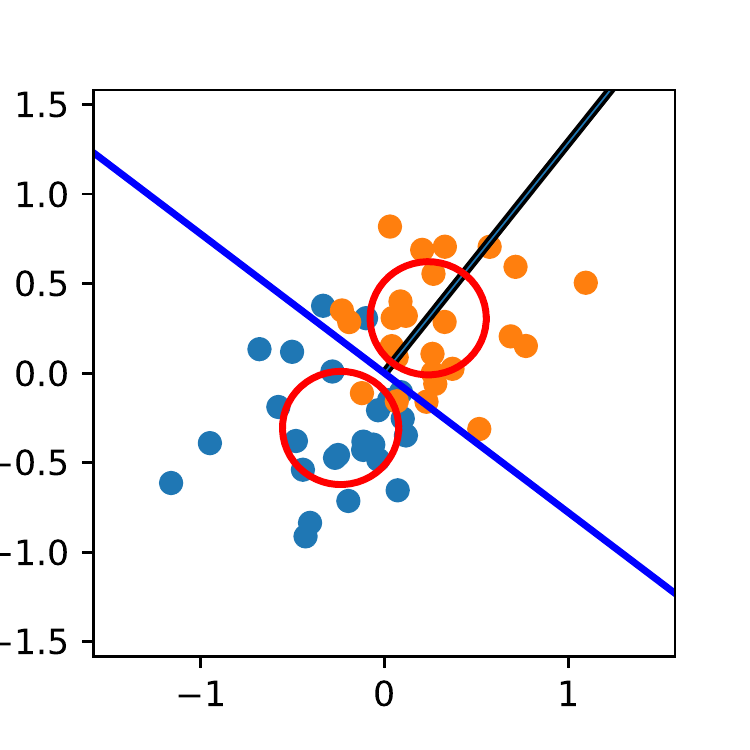}
\includegraphics[scale=0.6]{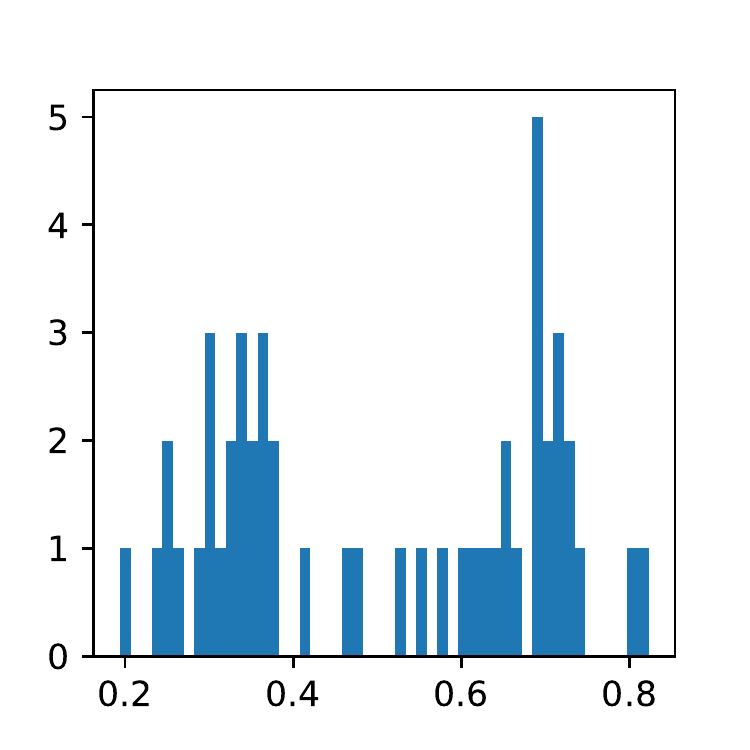}
\includegraphics[scale=0.6]{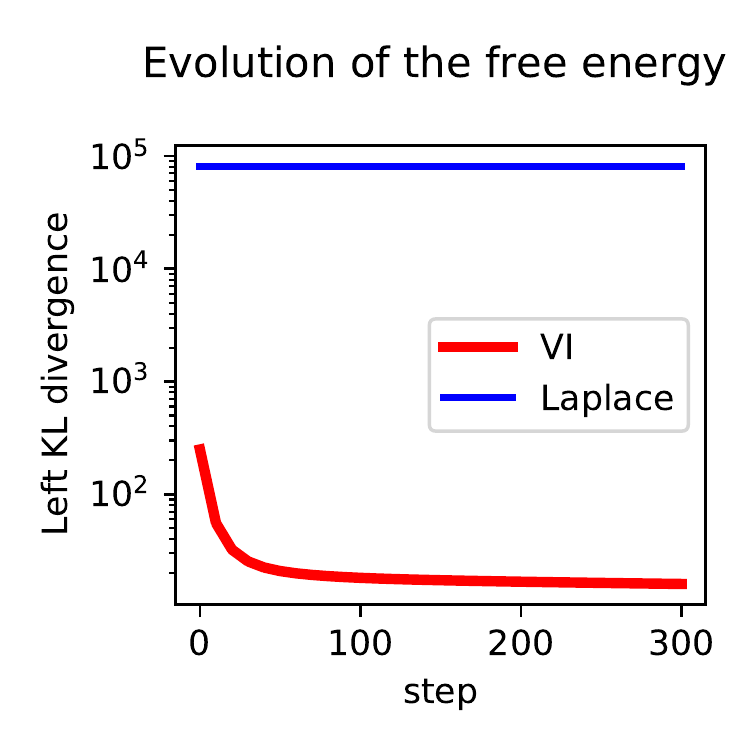}
\caption{Results in dimension $d=10$, $N=50$ for a separation factor $s=0.6$ (upper row) and $s=1.5$ (lower row). Left column:  synthetic dataset projected onto the two first  coordinates. Middle column:   histogram representing the number of examples predicted at a given probability by the obtained classifier. Right column: convergence in terms of  unnormalized KL divergence. The unnormalized KL is computed via \eqref{lnZ:eq} letting $Z=1$ (upper row) and $Z=10^{20}$ (lower row). The Runge--Kutta step size is set to $0.1$.}
\label{logHD1}
\end{figure}

\begin{figure}[!h]
\centering
\includegraphics[scale=0.6]{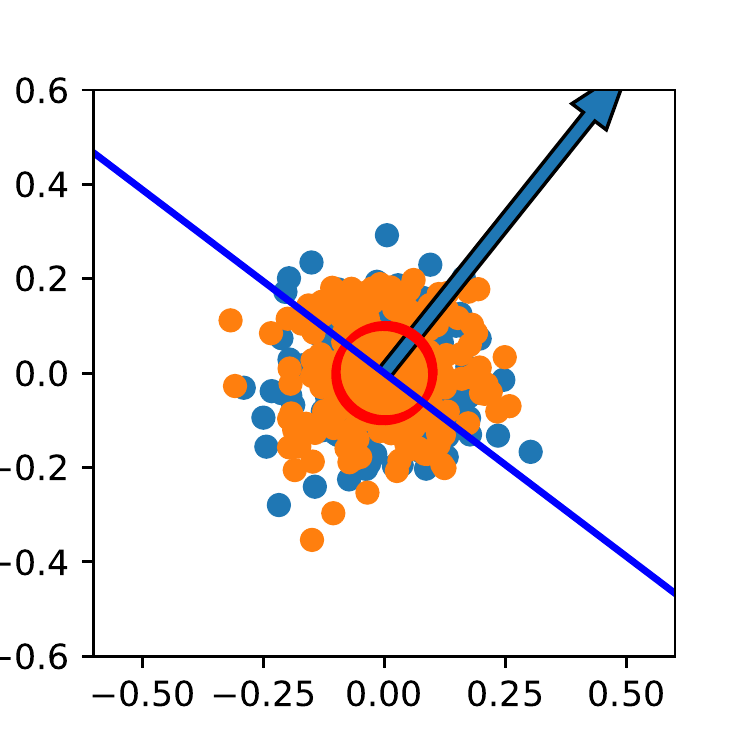}
\includegraphics[scale=0.6]{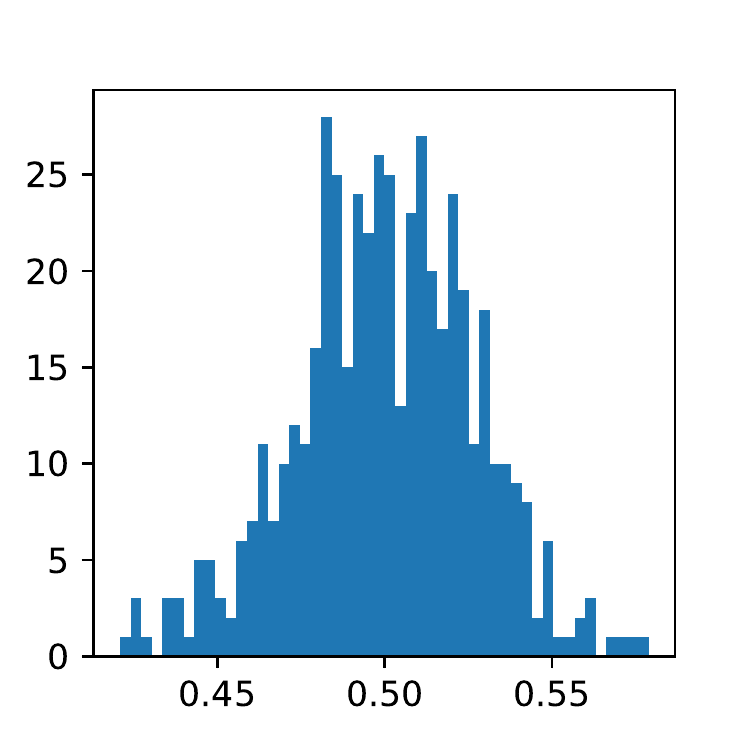}
\includegraphics[scale=0.6]{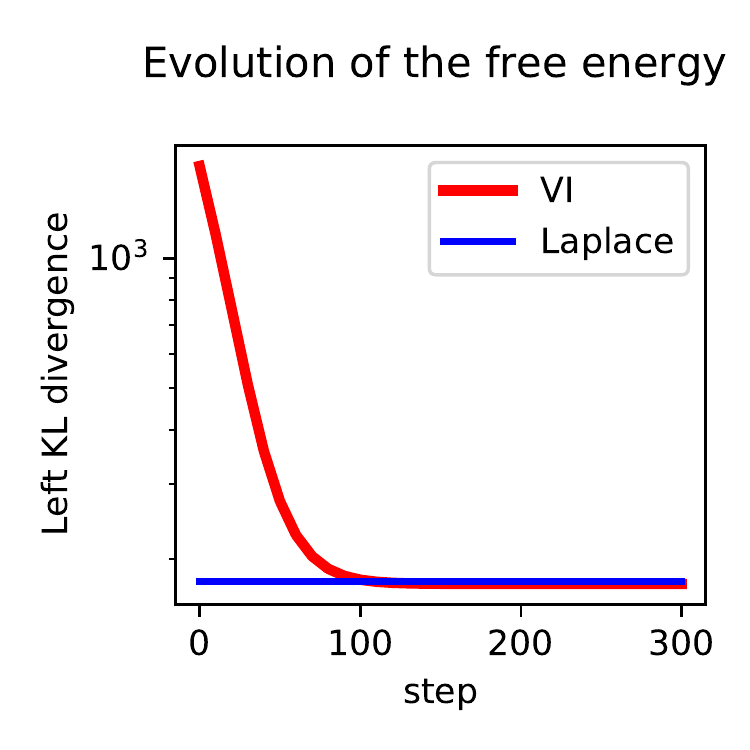}
\includegraphics[scale=0.6]{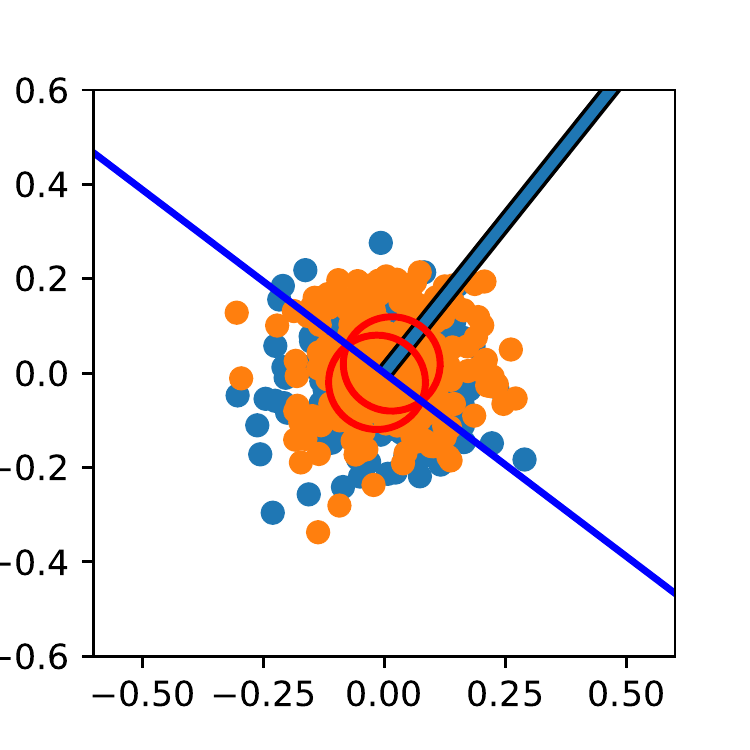}
\includegraphics[scale=0.6]{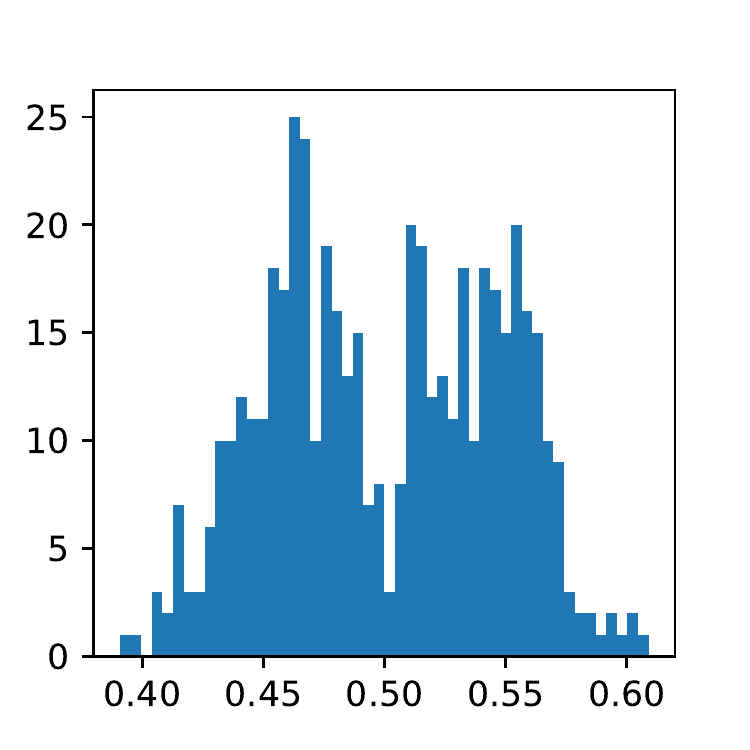}
\includegraphics[scale=0.6]{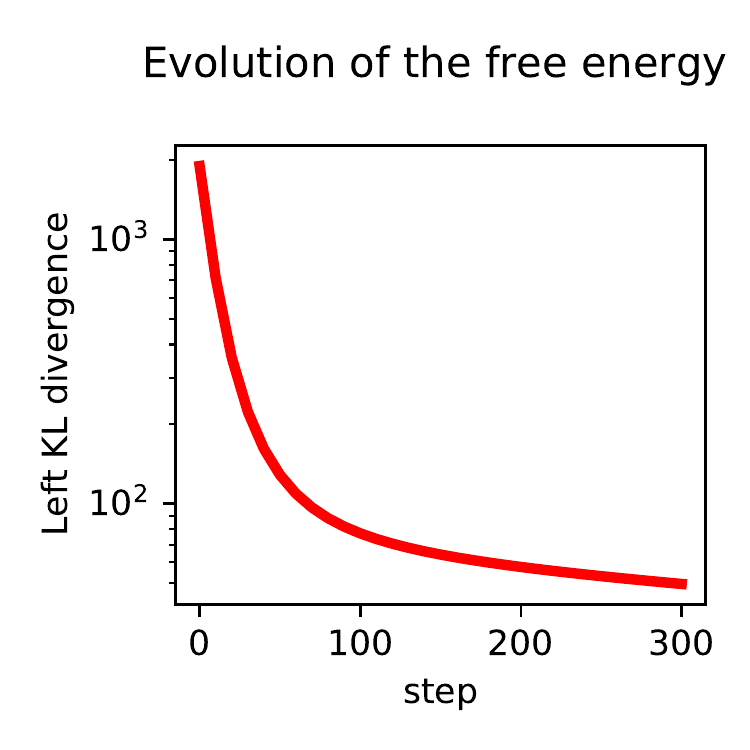}
\caption{Same as Figure~\ref{logHD1} but with dimension $d=100$, $N=500$, with separation factor $s=0.05$ (upper row) and $s=0.3$ (lower row). The unnormalized KL is computed  letting $Z=1$ (upper row) and $Z=10^{100}$ (lower row).  The unnormalized KL divergence for the Laplace method is not shown in the lower plot  because it is too large to be visualized.
}
\label{logHD3}
\end{figure}

\clearpage

\section{Experiments for mixture of Gaussians VI}\label{scn:xp-gmmvi}

In this section, we consider a mixture of Gaussians model to approximate a target distribution in the simple two-dimensional case. The goal is to illustrate the convergence of the approximating particles system \eqref{mean:particle:eq}-\eqref{cov:particle:eq} to an approximation of the target in the form of a finite mixture of Gaussians.

\subsection{Setup} 
We consider the bimodal and logistic targets defined in Section~\ref{scn:xp-vi}, as well as more complex  targets  defined as finite mixtures of Gaussians:
 \begin{align*}
 \pi=\sum_{i=1}^M w^*_i \,\mathcal{N}(m^*_i,\Sigma^*_i)\,.
 \end{align*}
The gradient $\nabla_x \log  \pi(x)$   then writes:
  \begin{align*}
\nabla_x \log \pi(x)=\frac{1}{\pi(x)}\, \nabla_x  \pi(x)=\frac{1}{\pi(x)}  \sum_{i=1}^M w^*_i \,{\Sigma^*_i}^{-1}\, (x-m^*_i)\,\mathcal{N}(x\mid m^*_i,\Sigma^*_i)\,.
 \end{align*}
 We consider $K$ Gaussian samples equally weighted such that our mixture model is $p=  \frac{1}{K} \sum_{i=1}^K p_i= \frac{1}{K} \sum_{i=1}^K \mathcal{N}(m_i,\Sigma_i)$. Even if we are using an approximation with equal weights, contrary to the target (which can be arbitrary in practice), we can hope from Theorem~\ref{THM:MIXTURE_GF}  convergence to a good approximation of $\pi$ when letting $K \gg M$.

\subsection{Implementation details}

\subsubsection{Integration of the ODEs}
Following equations \eqref{mean:particle:eq}-\eqref{cov:particle:eq}, we implement the system of ODEs
 \begin{align*}
&\dot m_k=\E_{p_k}[\nabla_x  \ln \pi] -\E_{p_k}[\nabla_x  \ln p]\,, \\
&\dot \Sigma_k= A + A^\T\,, \\
&\text{where}~A= \E_{p_k} [(x-\mu_k) \otimes \nabla_x  \ln \pi] - \E_{p_k} [(x-\mu_k)\otimes \nabla_x  \ln p]\,.
  \end{align*}
We recall that these equations arise from applying Theorem~\ref{THM:MIXTURE_GF} to a discrete mixing measure and applying integration by parts to obtain Hessian-free updates. To constrain the covariance matrix to remain definite positive along the numerical integration process, we use the same method as in the Gaussian \vi{} case (Section~\ref{scn:numerical}): we replace each ODE for a covariance matrix $\Sigma$ by an ODE for its lower triangular matrix factor $R$ where   $\Sigma = RR^\T$.
To compute the expectations, we use the sigma points with cubature rules as described in Section~\ref{scn:numerical}.

Finally, to solve the ODEs we consider a classical Runge--Kutta scheme of $4^{\rm th}$ order. The coupling between the ODEs is taken into account by applying the Runge--Kutta algorithm on the joint ODE $\dot{X}=F(X)$ where the Gaussian parameters are stacked as follows:
\begin{align*}
    X=\begin{bmatrix} m_1,\dotsc,m_K,\vvec(R_1),\dotsc,\vvec(R_K)\end{bmatrix}\,.
\end{align*}
For our problem, setting the Runge--Kutta step size to $0.1$ is sufficient.
We observe that asymptotic convergence, i.e., complete stability of the ODE system, may require many iterations when we propagate a large number of coupled Gaussian particles. On the other hand, the KL divergence is roughly stable after $30$ steps.

\subsubsection{Initialization of the Gaussian particles} 

We start by illustrating the sensitivity of the algorithm to the initialization on a simple example with one Gaussian particle and a bimodal target (Figure~\ref{XP0}). When the initial particle is close to one of the two modes and has same covariance as each mode, then it moves towards that mode and its covariance remains constant. When the particle is equidistant from the two modes, then the mean of the particle converges to the average of the two modes, and its covariance increases. Perturbing the initial condition slightly leads the particle to be attracted to one of the two modes. 

\begin{figure}[!h]
\includegraphics[scale=0.45]{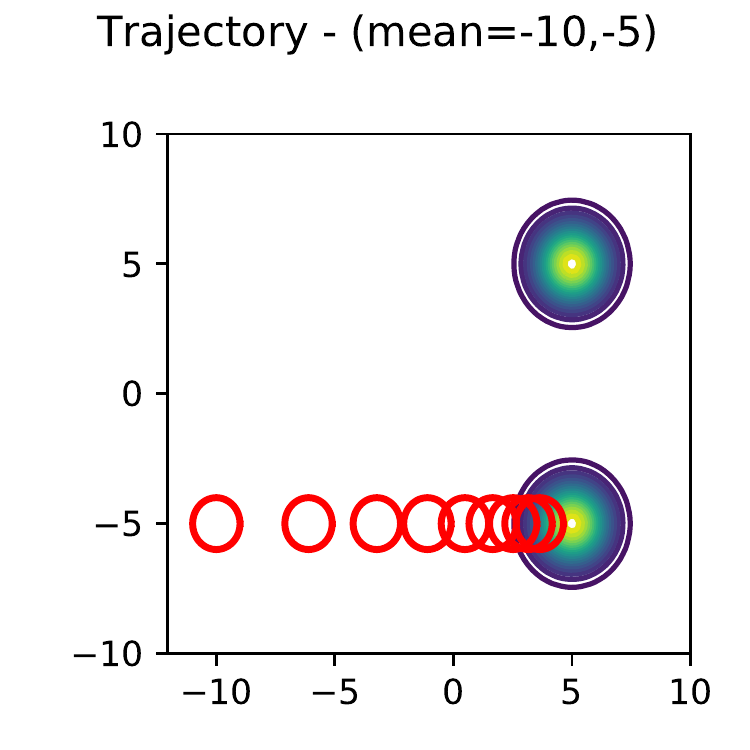}
\includegraphics[scale=0.45]{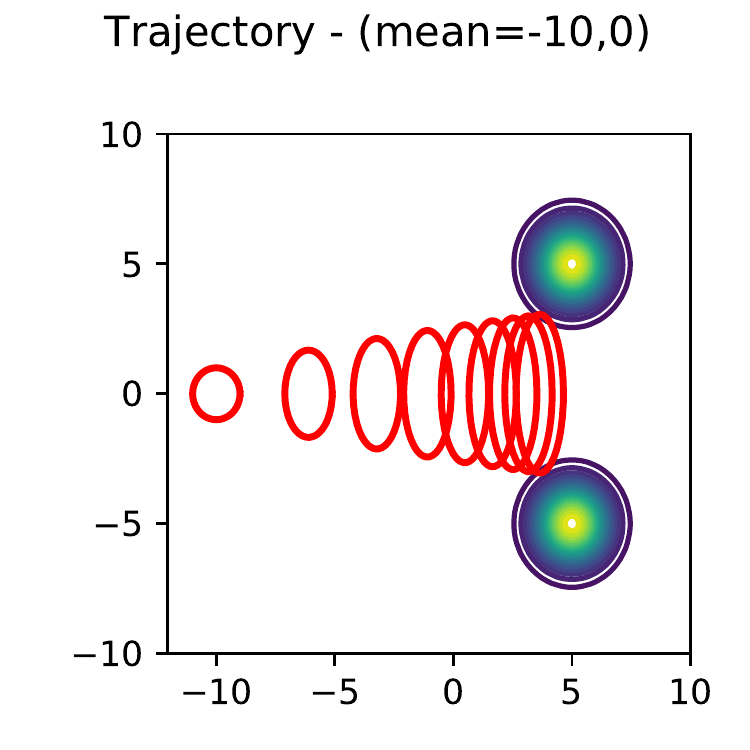}
\includegraphics[scale=0.45]{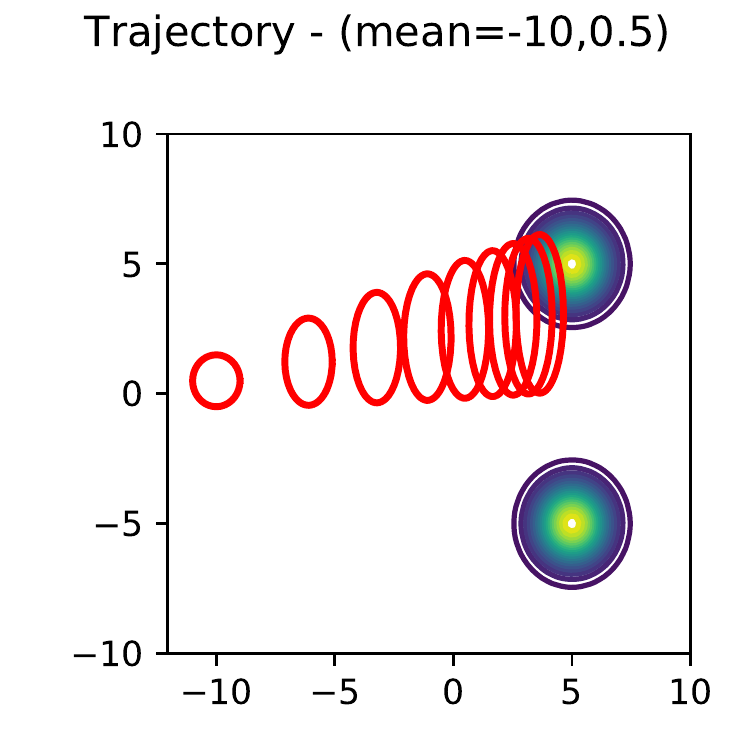}
\includegraphics[scale=0.45]{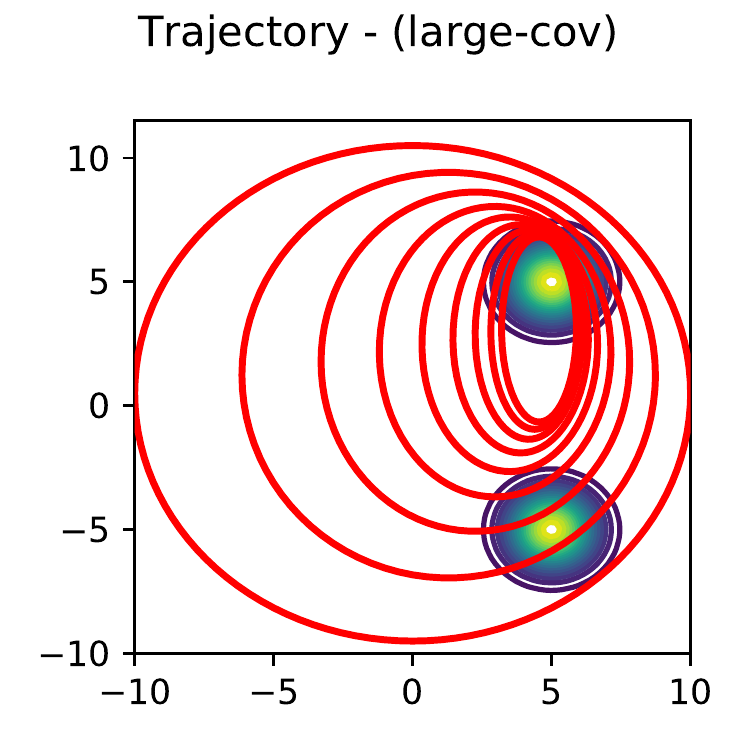}
\caption{Trajectory of a Gaussian particle for different initial conditions. In the three left plots, we initialize the particle with the same covariance as each mode, and in the right plot we initialize the particle with a large covariance.}
\label{XP0}
\end{figure} 

To avoid bad initialization, the idea is  to  generate instead more particles than the number of modes of the target. Finally, we initialize our Gaussian particles
with means randomly chosen from a Euclidean ball which covers most of the mass of the target density.

\subsection{Experimental results}

We show qualitative fits by plotting the contour lines of the approximated density (compared to the true density), as well as quantitative evaluation of the  KL divergence to the target.

The true posterior is computed using a discrete grid of size $100 \times 100$. The KL divergences are evaluated using Monte Carlo sampling.

\subsubsection{Simple targets}

We consider a mixture of $20$ Gaussians to approximate the   targets defined in Section \ref{gauss:target}. We see  in Figure~\ref{SamplingXP1}  that the algorithm captures both modes of the bimodal distribution, and approximates well the logistic target also, see  Figure~\ref{SamplingXP2}.
 
\begin{figure}[!h]
\centering
\includegraphics[scale=0.5]{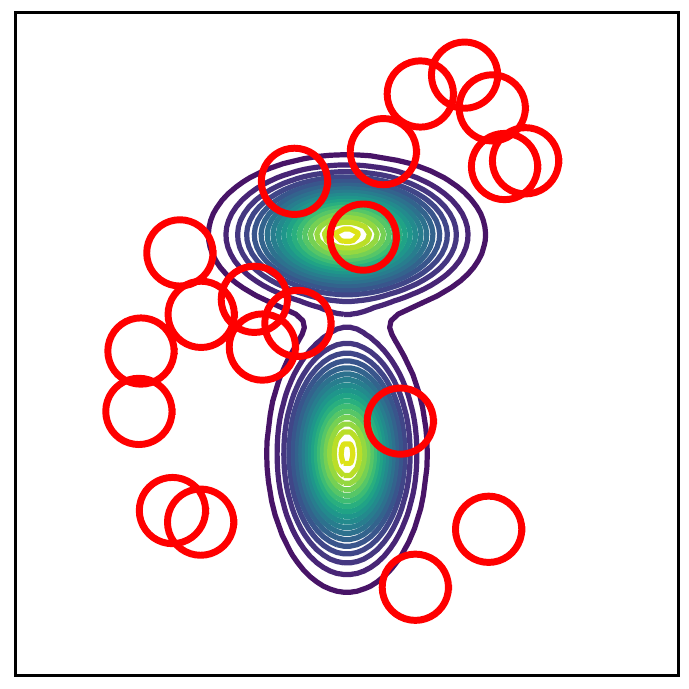}
\includegraphics[scale=0.5]{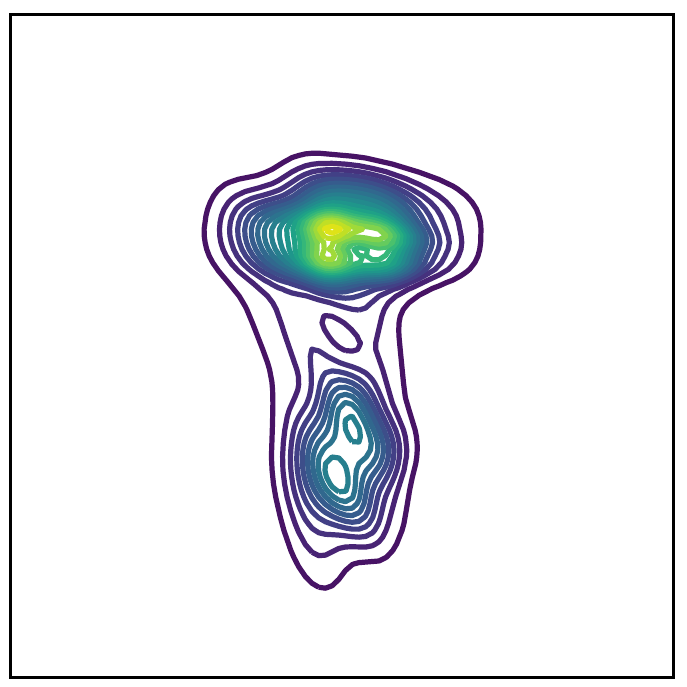}
\caption{Approximation of the  bimodal target using $20$ Gaussian particles at initialization (left) and at final step  (right). We use Runge--Kutta integration with step size $0.1$ and integrattion time $T=30$ (i.e., $300$ steps). }
\label{SamplingXP1}
\end{figure} 

\begin{figure}[!h]
\centering
\includegraphics[scale=0.5]{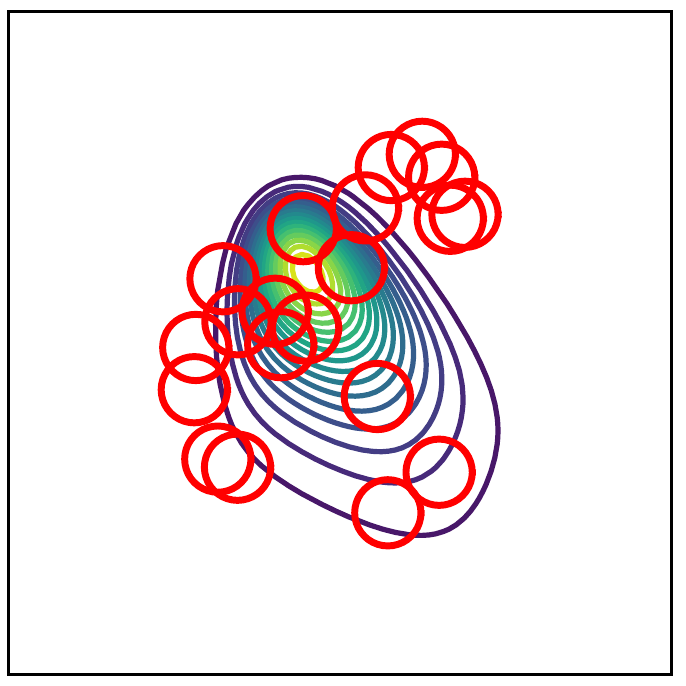}
\includegraphics[scale=0.5]{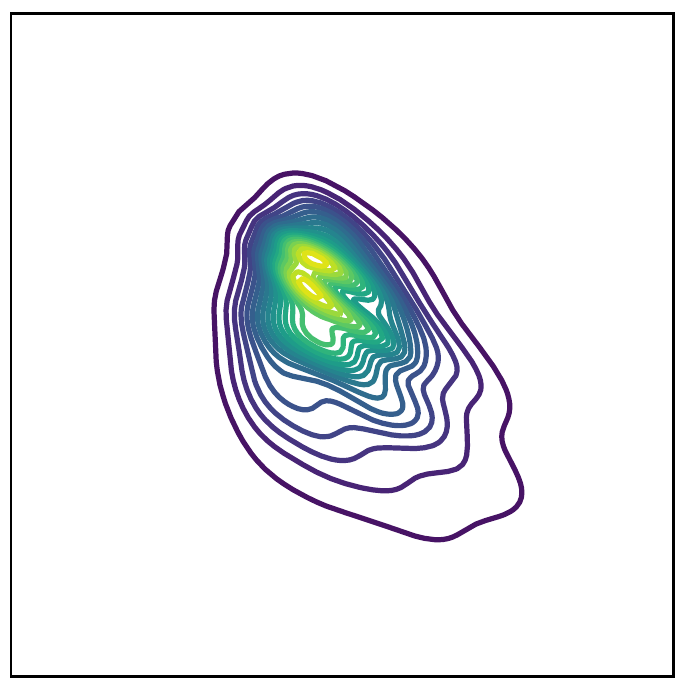}
\caption{Approximation of the logistic target with $20$ Gaussian particles. We use Runge--Kutta integration with step size $0.1$ and integration time $T=30$ (i.e., $300$ steps).}
\label{SamplingXP2}
\end{figure} 

\subsubsection{More complex targets}

We assess the sensitivity to the number of particles in Figures~\ref{XP1_nbParticles},  ~\ref{XP2_nbParticles}, and~\ref{XP3_nbParticles}. When the number of particles increases, better KL divergence is achieved and the distribution is better approximated.  We also  note that when the samples initially cover a low density mode as in Figure~\ref{XP2_nbParticles}, they tend to overestimate the local density before they escape the mode. 

\begin{figure}[!h]
\centering
\includegraphics[scale=0.4]{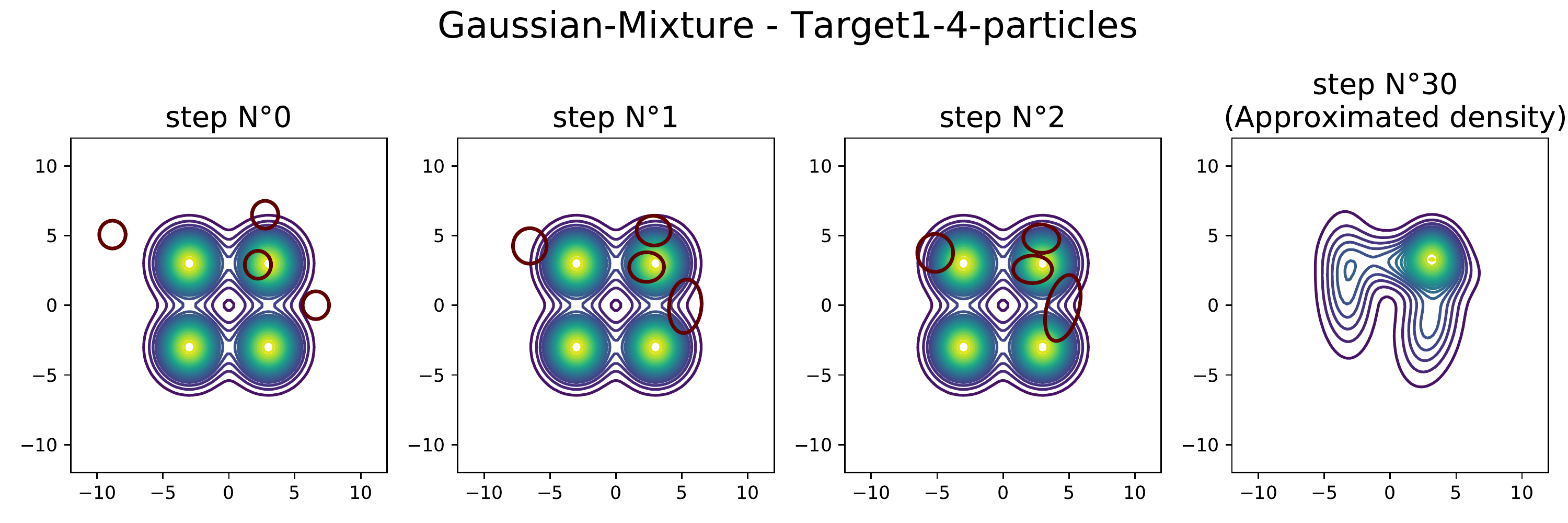}
\includegraphics[scale=0.4]{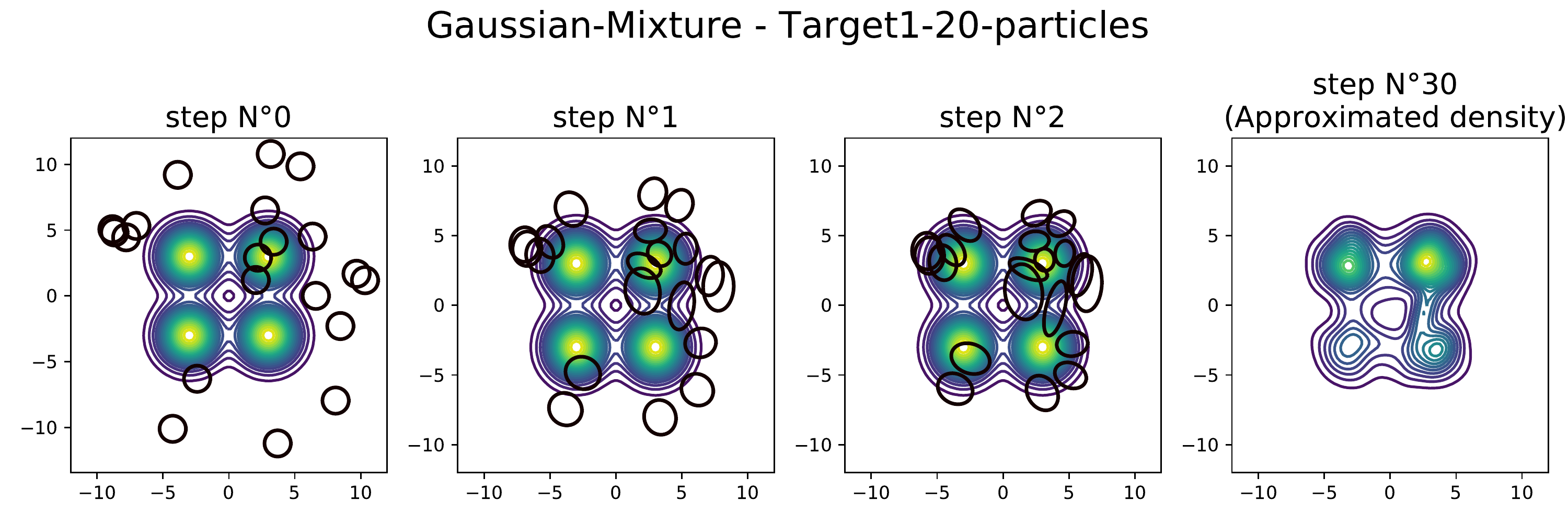}
\includegraphics[scale=0.4]{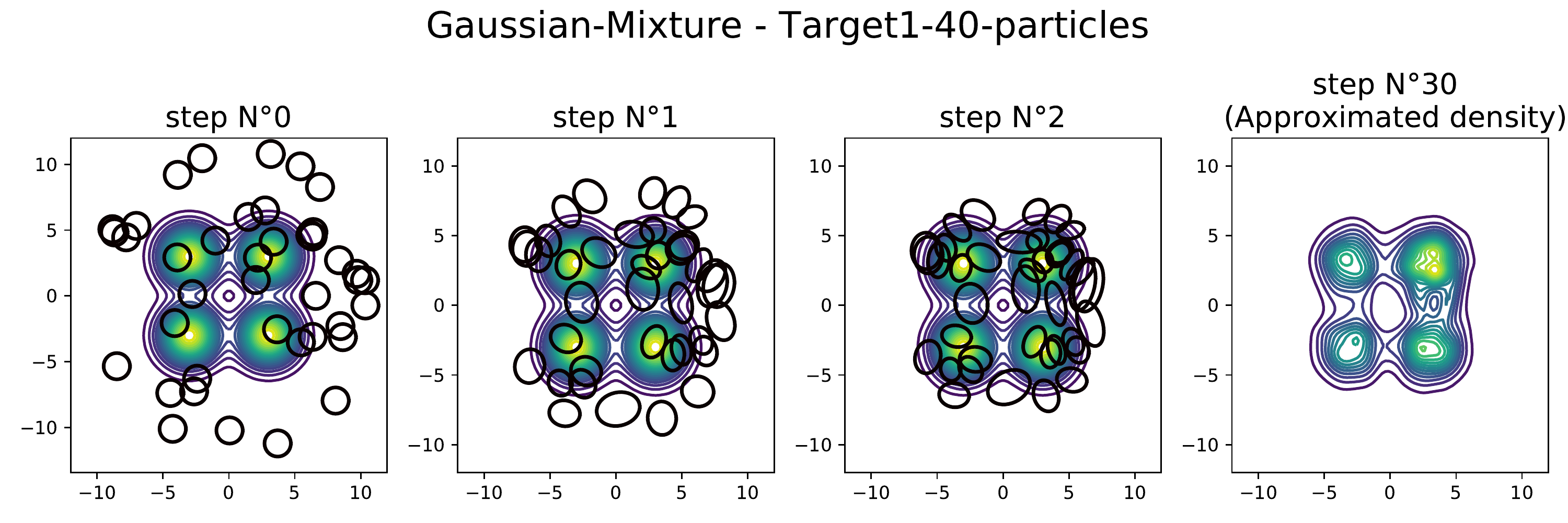}
\includegraphics[scale=0.4]{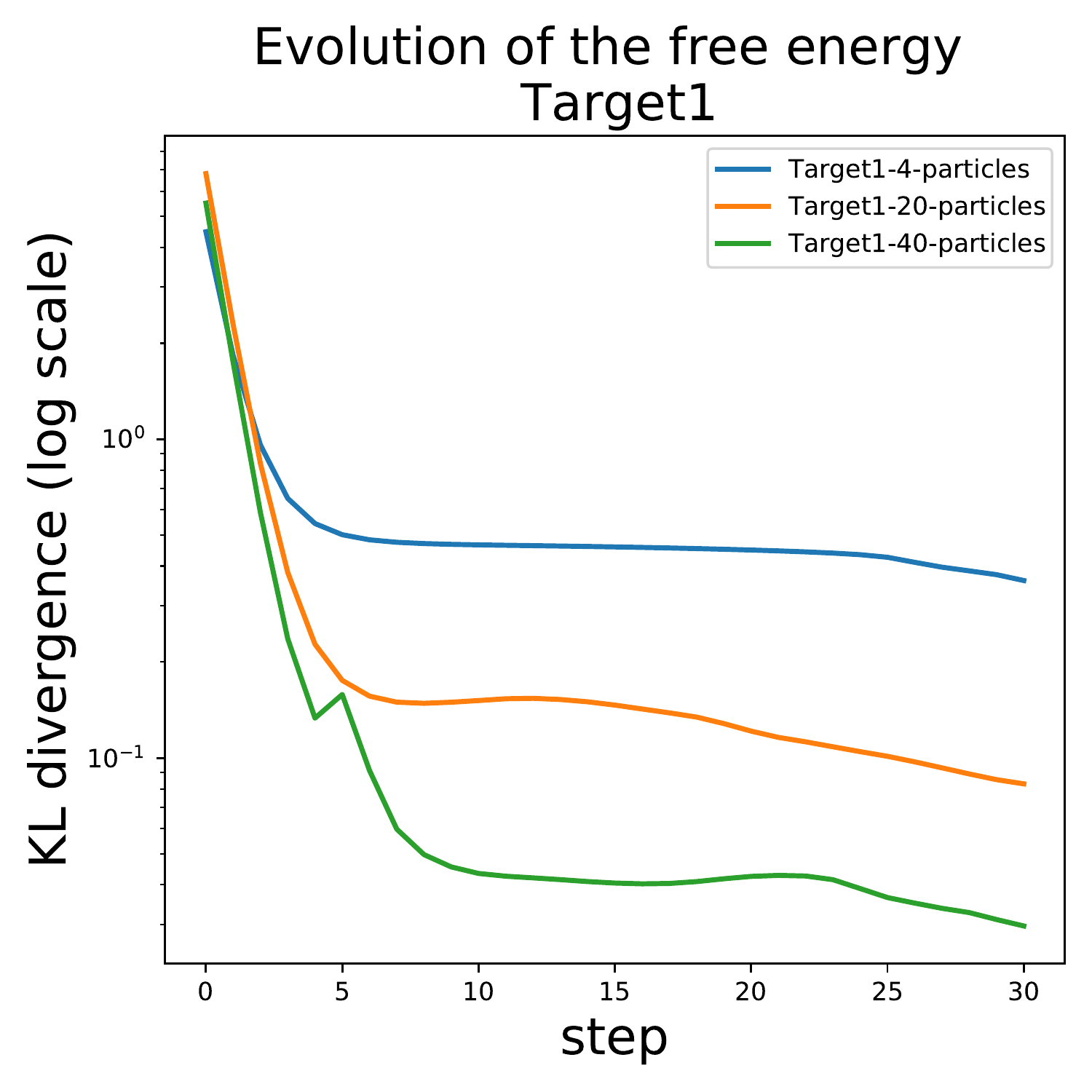}
\caption{A target with $4$ equally weighted modes and isotropic covariances.}
\label{XP1_nbParticles}
\end{figure} 

\begin{figure}[!h]
\centering
\includegraphics[scale=0.4]{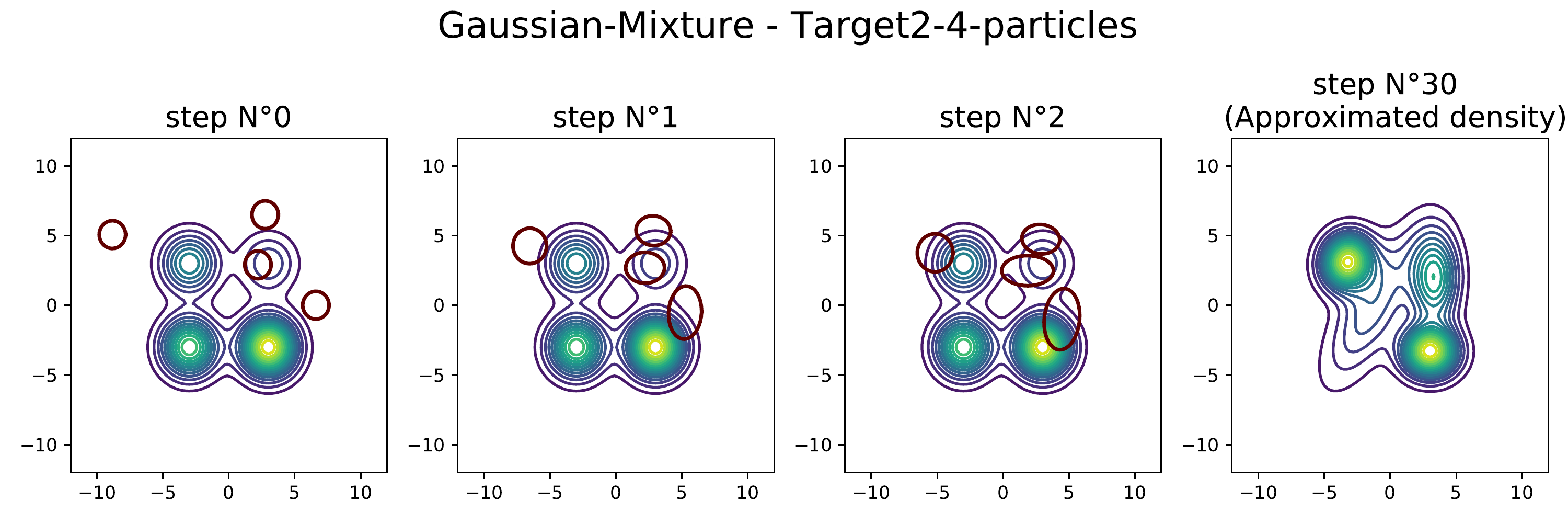}
\includegraphics[scale=0.4]{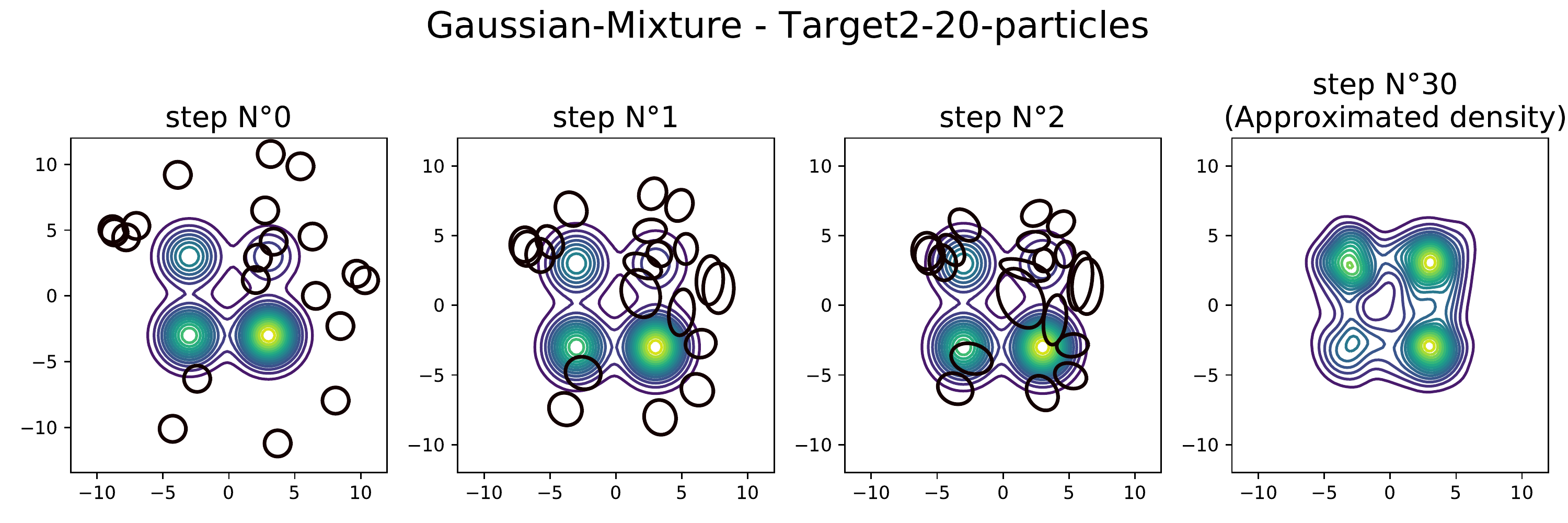}
\includegraphics[scale=0.4]{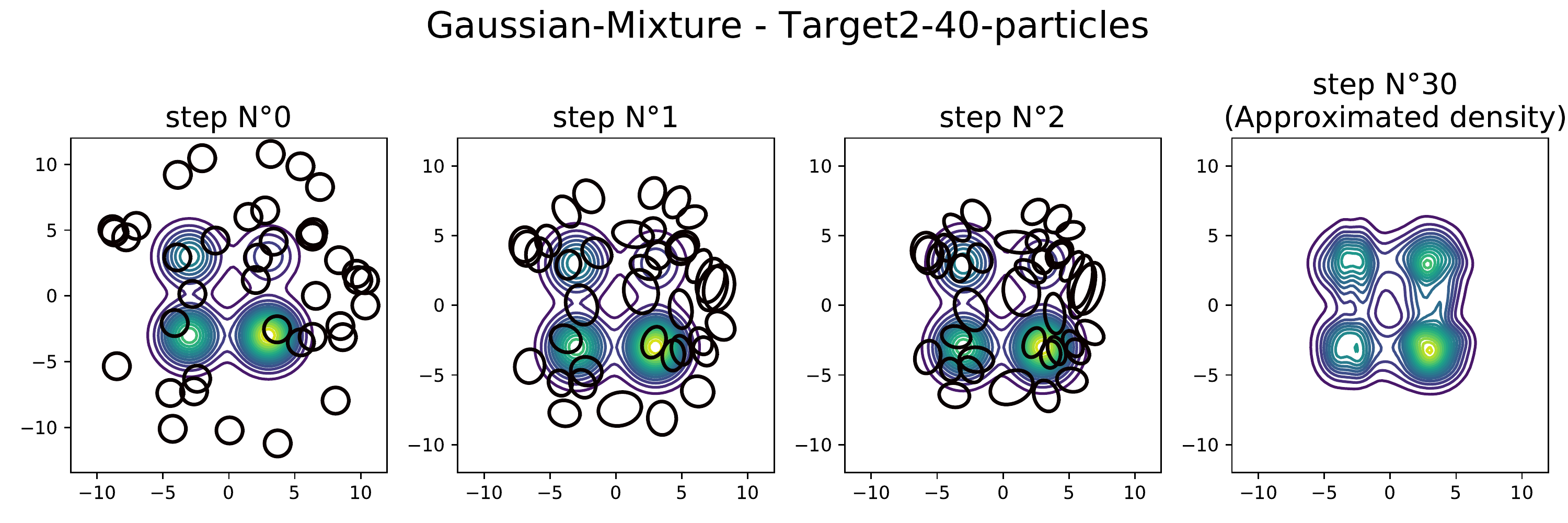}
\includegraphics[scale=0.4]{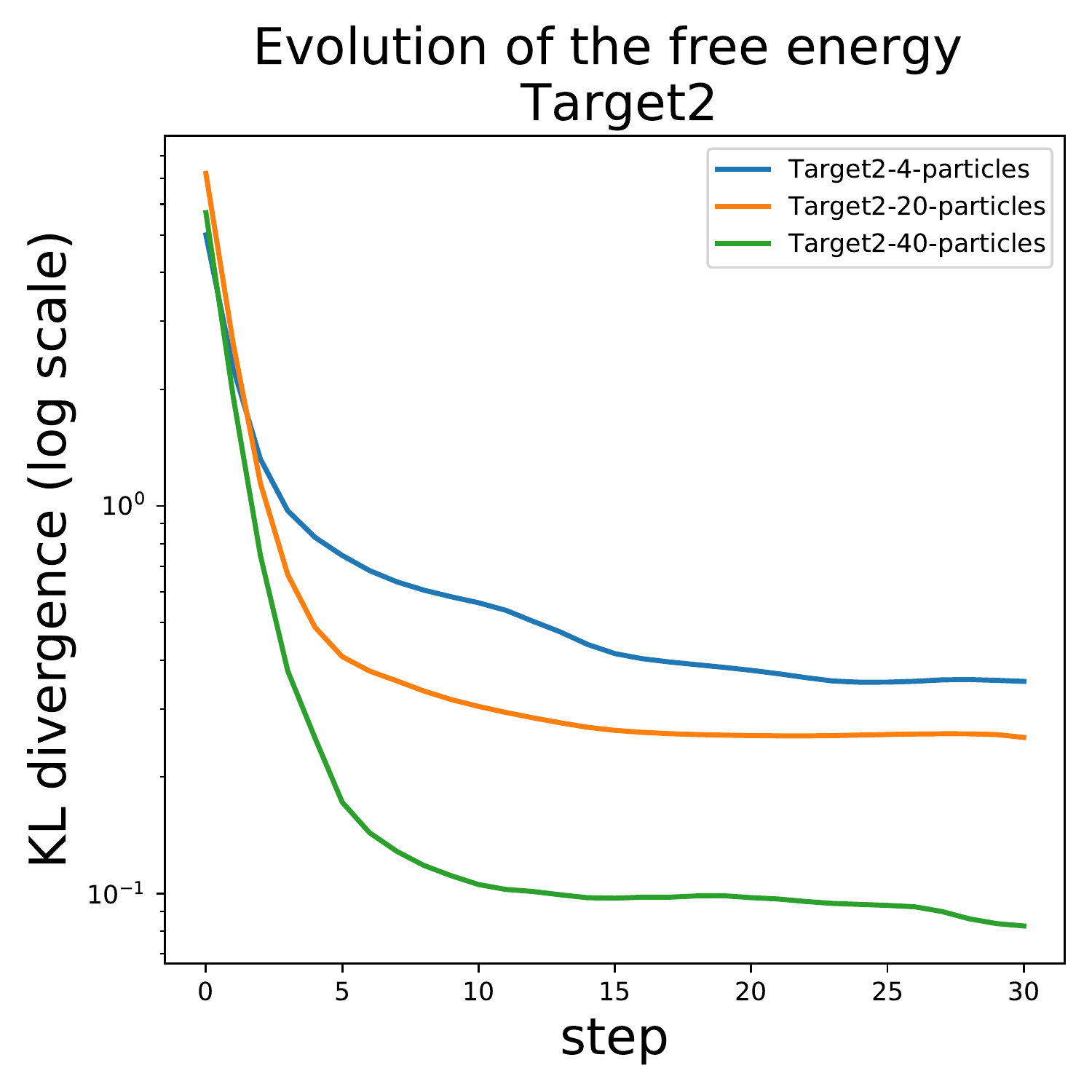}
\caption{A target with $4$ non-equally weighted modes and isotropic covariances.}
\label{XP2_nbParticles}
\end{figure} 

\begin{figure}[!h]
\centering
\includegraphics[scale=0.4]{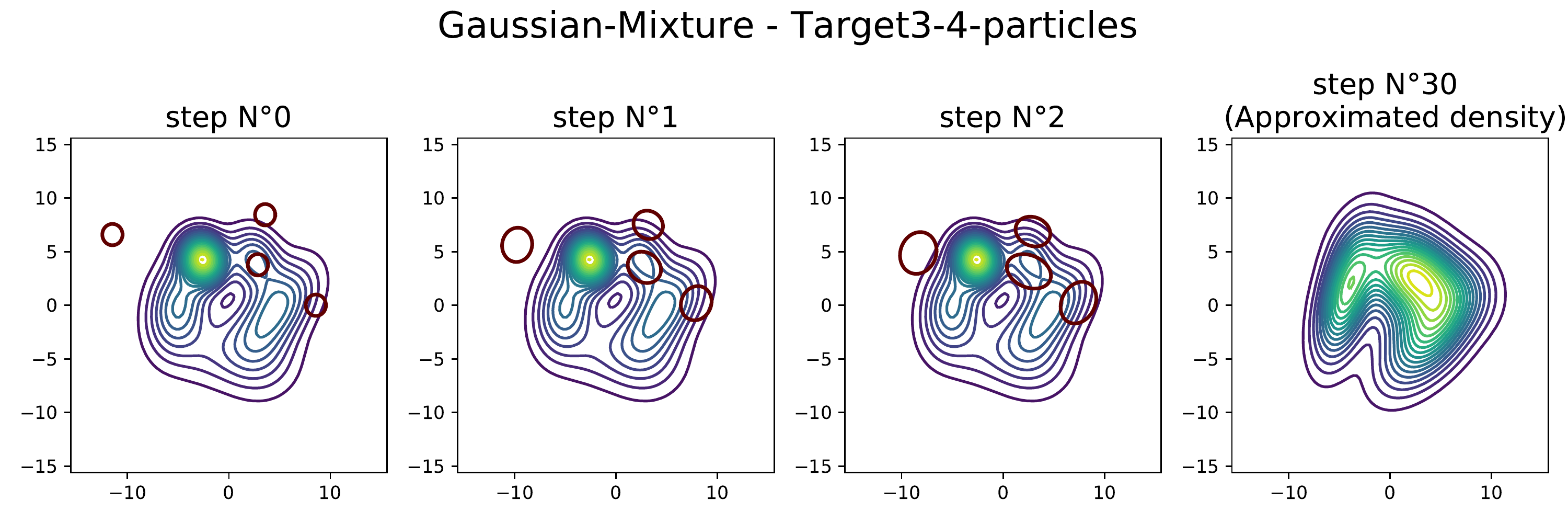}
\includegraphics[scale=0.4]{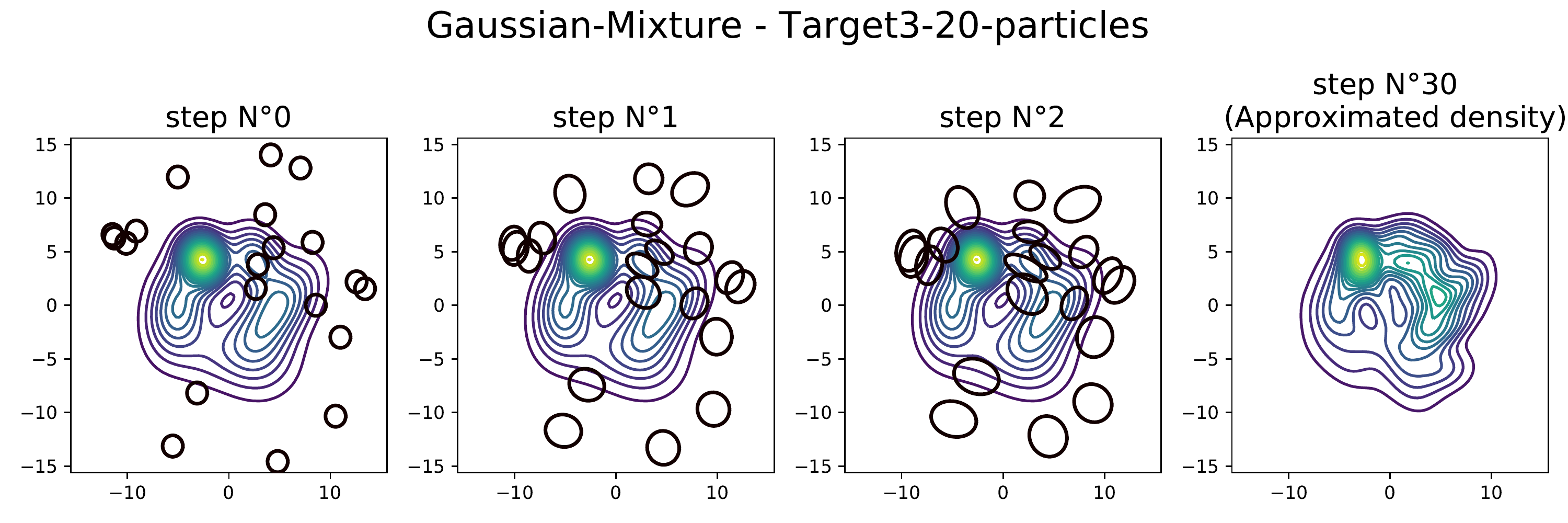}
\includegraphics[scale=0.4]{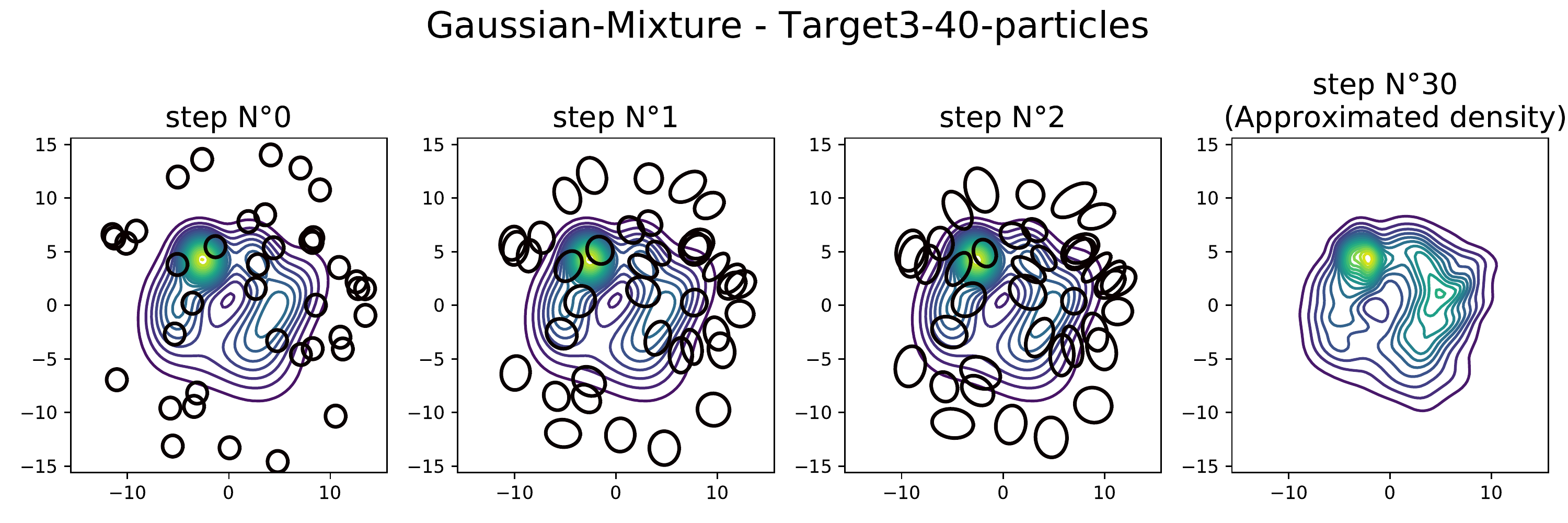}
\includegraphics[scale=0.4]{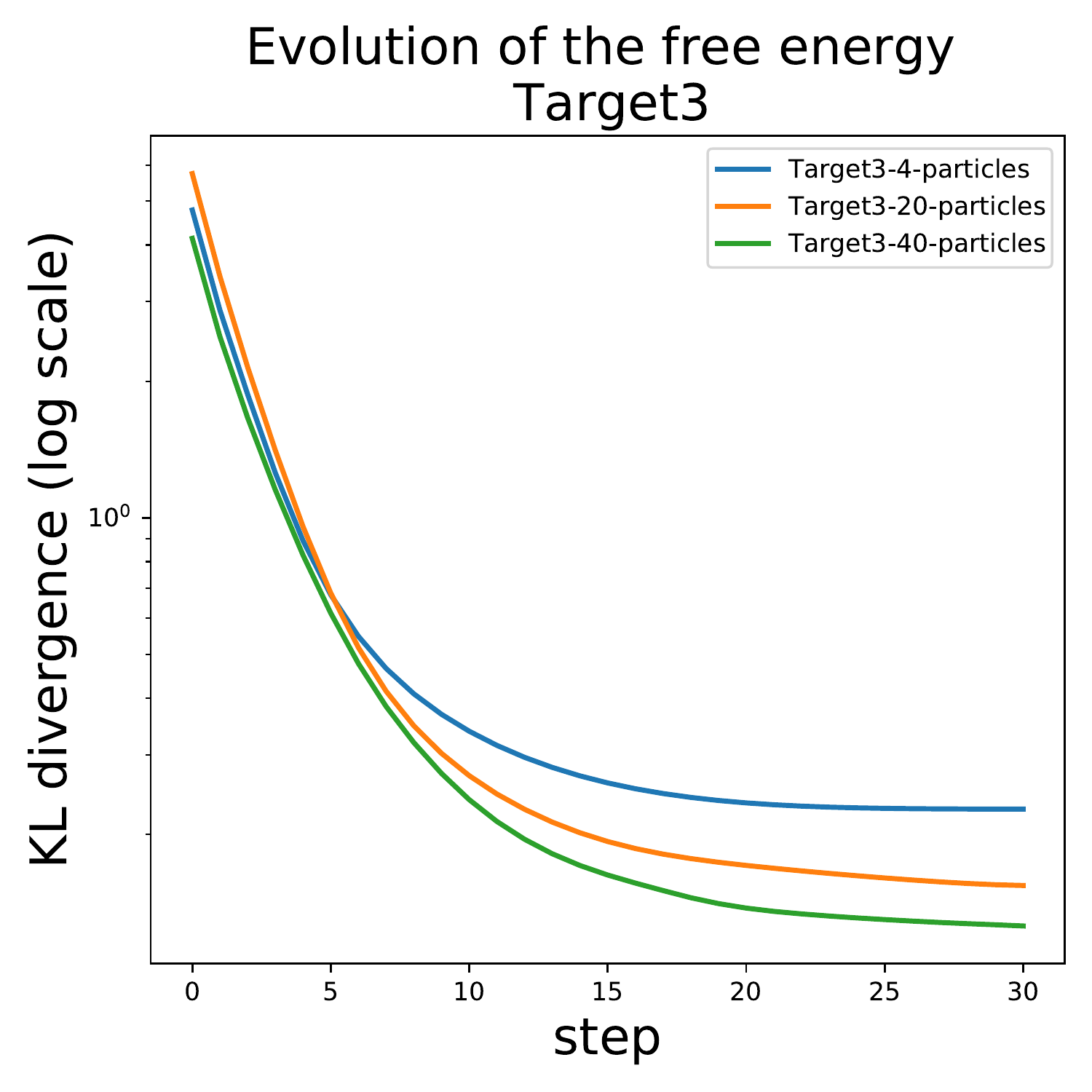}
\caption{A target with $6$ non-equally weighted modes and isotropic covariances.}
\label{XP3_nbParticles}
\end{figure}

\bibliography{Bures-JKO.bib}

\begin{thebibliography}{84}
\providecommand{\natexlab}[1]{#1}
\providecommand{\url}[1]{\texttt{#1}}
\expandafter\ifx\csname urlstyle\endcsname\relax
  \providecommand{\doi}[1]{doi: #1}\else
  \providecommand{\doi}{doi: \begingroup \urlstyle{rm}\Url}\fi

\bibitem[Alquier and Ridgway(2020)]{alquierridgway2020variational}
Pierre Alquier and James Ridgway.
\newblock Concentration of tempered posteriors and of their variational
  approximations.
\newblock \emph{Ann. Statist.}, 48\penalty0 (3):\penalty0 1475--1497, 2020.

\bibitem[Alquier et~al.(2016)Alquier, Ridgway, and Chopin]{alqridcho2016vi}
Pierre Alquier, James Ridgway, and Nicolas Chopin.
\newblock On the properties of variational approximations of {G}ibbs
  posteriors.
\newblock \emph{J. Mach. Learn. Res.}, 17:\penalty0 Paper No. 239, 41, 2016.

\bibitem[Altschuler et~al.(2023)Altschuler, Chewi, Gerber, and
  Stromme]{altschuleretal2023bwbarycenter}
Jason Altschuler, Sinho Chewi, Patrik Gerber, and Austin~J. Stromme.
\newblock Averaging on the {B}ures--{W}asserstein manifold: dimension-free
  convergence of gradient descent.
\newblock \emph{arXiv e-prints}, art. arXiv:2106.08502, 2023.

\bibitem[Amari and Nagaoka(2000)]{AmaNag00}
Shun-ichi Amari and Hiroshi Nagaoka.
\newblock \emph{Methods of information geometry}, volume 191 of
  \emph{Translations of Mathematical Monographs}.
\newblock American Mathematical Society, Providence, RI, 2000.

\bibitem[Ambrosio et~al.(2008)Ambrosio, Gigli, and
  Savar\'{e}]{ambrosio2008gradient}
Luigi Ambrosio, Nicola Gigli, and Giuseppe Savar\'{e}.
\newblock \emph{Gradient flows in metric spaces and in the space of probability
  measures}.
\newblock Lectures in Mathematics ETH Z\"{u}rich. Birkh\"{a}user Verlag, Basel,
  second edition, 2008.

\bibitem[Arasaratnam and Haykin(2009)]{Arasaratnam09}
Ienkaran Arasaratnam and Simon Haykin.
\newblock Cubature {K}alman filters.
\newblock \emph{IEEE Trans. Automat. Control}, 54\penalty0 (6):\penalty0
  1254--1269, 2009.

\bibitem[Ay et~al.(2017)Ay, Jost, L\^{e}, and
  Schwachh\"{o}fer]{ayetal2017infogeometry}
Nihat Ay, J\"{u}rgen Jost, H\^{o}ng~V\^{a}n L\^{e}, and Lorenz
  Schwachh\"{o}fer.
\newblock \emph{Information geometry}, volume~64 of \emph{Ergebnisse der
  Mathematik und ihrer Grenzgebiete. 3. Folge. A Series of Modern Surveys in
  Mathematics [Results in Mathematics and Related Areas. 3rd Series. A Series
  of Modern Surveys in Mathematics]}.
\newblock Springer, Cham, 2017.

\bibitem[Bakry et~al.(2014)Bakry, Gentil, and Ledoux]{bakrygentilledoux2014}
Dominique Bakry, Ivan Gentil, and Michel Ledoux.
\newblock \emph{Analysis and geometry of {M}arkov diffusion operators}, volume
  348 of \emph{Grundlehren der Mathematischen Wissenschaften [Fundamental
  Principles of Mathematical Sciences]}.
\newblock Springer, Cham, 2014.

\bibitem[Barber and Bishop(1997)]{Barber97}
David Barber and Christopher Bishop.
\newblock Ensemble learning for multi-layer networks.
\newblock In \emph{Advances in Neural Information Processing Systems},
  volume~10, 1997.

\bibitem[Benamou and Brenier(1999)]{benamoubrenier1999}
Jean-David Benamou and Yann Brenier.
\newblock A numerical method for the optimal time-continuous mass transport
  problem and related problems.
\newblock In \emph{Monge {A}mp\`ere equation: applications to geometry and
  optimization ({D}eerfield {B}each, {FL}, 1997)}, volume 226 of \emph{Contemp.
  Math.}, pages 1--11. Amer. Math. Soc., Providence, RI, 1999.

\bibitem[Bhatia et~al.(2019)Bhatia, Jain, and Lim]{Bhatia19}
Rajendra Bhatia, Tanvi Jain, and Yongdo Lim.
\newblock On the {B}ures--{W}asserstein distance between positive definite
  matrices.
\newblock \emph{Expo. Math.}, 37\penalty0 (2):\penalty0 165--191, 2019.

\bibitem[Bishop(2006)]{Bishop06}
Christopher~M. Bishop.
\newblock \emph{Pattern recognition and machine learning}.
\newblock Information Science and Statistics. Springer, New York, 2006.

\bibitem[Blei et~al.(2017)Blei, Kucukelbir, and McAuliffe]{BleKucMcA17}
David~M. Blei, Alp Kucukelbir, and Jon~D. McAuliffe.
\newblock Variational inference: A review for statisticians.
\newblock \emph{Journal of the American Statistical Association}, 112\penalty0
  (518):\penalty0 859--877, 2017.

\bibitem[Bonaschi et~al.(2015)Bonaschi, Carrillo, Di~Francesco, and
  Peletier]{bonaschietal2015nonlocal}
Giovanni~A. Bonaschi, Jos\'{e}~A. Carrillo, Marco Di~Francesco, and Mark~A.
  Peletier.
\newblock Equivalence of gradient flows and entropy solutions for singular
  nonlocal interaction equations in 1{D}.
\newblock \emph{ESAIM Control Optim. Calc. Var.}, 21\penalty0 (2):\penalty0
  414--441, 2015.

\bibitem[Bures(1969)]{Bures69}
Donald Bures.
\newblock An extension of {K}akutani's theorem on infinite product measures to
  the tensor product of semifinite {$w^{\ast} $}-algebras.
\newblock \emph{Trans. Amer. Math. Soc.}, 135:\penalty0 199--212, 1969.

\bibitem[Caglioti et~al.(2009)Caglioti, Pulvirenti, and
  Rousset]{cagliotietal2009constrainedns}
Emanuele Caglioti, Mario Pulvirenti, and Fr\'{e}d\'{e}ric Rousset.
\newblock On a constrained 2-{D} {N}avier--{S}tokes equation.
\newblock \emph{Comm. Math. Phys.}, 290\penalty0 (2):\penalty0 651--677, 2009.

\bibitem[Carlen and Gangbo(2003)]{carlengangbo2003constrained}
Eric~A. Carlen and Wilfrid Gangbo.
\newblock Constrained steepest descent in the 2-{W}asserstein metric.
\newblock \emph{Ann. of Math. (2)}, 157\penalty0 (3):\penalty0 807--846, 2003.

\bibitem[Carrillo et~al.(2011)Carrillo, Di~Francesco, Figalli, Laurent, and
  Slep\v{c}ev]{carrilloetal2011aggregation}
Jos\'{e}~A. Carrillo, Marco Di~Francesco, Alessio Figalli, Thomas Laurent, and
  Dejan Slep\v{c}ev.
\newblock Global-in-time weak measure solutions and finite-time aggregation for
  nonlocal interaction equations.
\newblock \emph{Duke Math. J.}, 156\penalty0 (2):\penalty0 229--271, 2011.

\bibitem[Carrillo et~al.(2012)Carrillo, Di~Francesco, Figalli, Laurent, and
  Slep\v{c}ev]{carilloetal2012confinement}
Jos\'{e}~A. Carrillo, Marco Di~Francesco, Alessio Figalli, Thomas Laurent, and
  Dejan Slep\v{c}ev.
\newblock Confinement in nonlocal interaction equations.
\newblock \emph{Nonlinear Anal.}, 75\penalty0 (2):\penalty0 550--558, 2012.

\bibitem[Carrillo et~al.(2019)Carrillo, Craig, and
  Patacchini]{carrillocraigpatacchini2019blob}
Jos\'{e}~A. Carrillo, Katy Craig, and Francesco~S. Patacchini.
\newblock A blob method for diffusion.
\newblock \emph{Calc. Var. Partial Differential Equations}, 58\penalty0
  (2):\penalty0 Paper No. 53, 53, 2019.

\bibitem[Challis and Barber(2013)]{chabar2013gaussiankl}
Edward Challis and David Barber.
\newblock Gaussian {K}ullback--{L}eibler approximate inference.
\newblock \emph{J. Mach. Learn. Res.}, 14:\penalty0 2239--2286, 2013.

\bibitem[{Chen} et~al.(2019){Chen}, {Georgiou}, and
  {Tannenbaum}]{chen2019mixture}
Yongxin {Chen}, Tryphon~T. {Georgiou}, and Allen {Tannenbaum}.
\newblock Optimal transport for {G}aussian mixture models.
\newblock \emph{IEEE Access}, 7:\penalty0 6269--6278, 2019.

\bibitem[Chen et~al.(2020)Chen, Dwivedi, Wainwright, and Yu]{chenetal2020hmc}
Yuansi Chen, Raaz Dwivedi, Martin~J. Wainwright, and Bin Yu.
\newblock Fast mixing of {M}etropolized {H}amiltonian {M}onte {C}arlo: benefits
  of multi-step gradients.
\newblock \emph{J. Mach. Learn. Res.}, 21:\penalty0 Paper No. 92, 71, 2020.

\bibitem[{Chewi} et~al.(2020){Chewi}, Le~Gouic, Lu, Maunu, and
  Rigollet]{chewietal2020svgd}
Sinho {Chewi}, Thibaut Le~Gouic, Chen Lu, Tyler Maunu, and Philippe Rigollet.
\newblock {SVGD} as a kernelized {W}asserstein gradient flow of the chi-squared
  divergence.
\newblock In \emph{Advances in Neural Information Processing Systems},
  volume~33, pages 2098--2109, 2020.

\bibitem[Chewi et~al.(2020)Chewi, Maunu, Rigollet, and Stromme]{CheMauRig20a}
Sinho Chewi, Tyler Maunu, Philippe Rigollet, and {Austin J.} Stromme.
\newblock Gradient descent algorithms for {B}ures--{W}asserstein barycenters.
\newblock In \emph{Proceedings of the Conference on Learning Theory}, volume
  125, pages 1276--1304. PMLR, 09--12 Jul 2020.

\bibitem[Chewi et~al.(2021)Chewi, Erdogdu, Li, Shen, and
  Zhang]{chewietal2021lmcpoincare}
Sinho Chewi, Murat~A. Erdogdu, Mufan~B. Li, Ruoqi Shen, and Matthew Zhang.
\newblock Analysis of {L}angevin {M}onte {C}arlo from {P}oincar\'e to
  log-{S}obolev.
\newblock \emph{arXiv e-prints}, art. arXiv:2112.12662, 2021.

\bibitem[Chizat et~al.(2018)Chizat, Peyr\'{e}, Schmitzer, and
  Vialard]{chizatetal2018wfr}
L\'{e}na\"{\i}c Chizat, Gabriel Peyr\'{e}, Bernhard Schmitzer, and
  Fran\c{c}ois-Xavier Vialard.
\newblock An interpolating distance between optimal transport and
  {F}isher--{R}ao metrics.
\newblock \emph{Found. Comput. Math.}, 18\penalty0 (1):\penalty0 1--44, 2018.

\bibitem[Craig and Bertozzi(2016)]{craigbertozzi2016blob}
Katy Craig and Andrea~L. Bertozzi.
\newblock A blob method for the aggregation equation.
\newblock \emph{Math. Comp.}, 85\penalty0 (300):\penalty0 1681--1717, 2016.

\bibitem[Craig et~al.(2022)Craig, Elamvazhuthi, Haberland, and
  Turanova]{craigetal2022blob}
Katy Craig, Karthik Elamvazhuthi, Matt Haberland, and Olga Turanova.
\newblock A blob method for inhomogeneous diffusion with applications to
  multi-agent control and sampling.
\newblock \emph{arXiv e-prints}, art. arXiv:2202.12927, March 2022.

\bibitem[Dalalyan(2017)]{Dal17a}
Arnak~S. Dalalyan.
\newblock Theoretical guarantees for approximate sampling from smooth and
  log-concave densities.
\newblock \emph{Journal of the Royal Statistical Society. Series B (Statistical
  Methodology)}, 79\penalty0 (3):\penalty0 651--676, 2017.

\bibitem[Dalalyan and Riou-Durand(2020)]{dalalyanrioudurand2020underdamped}
Arnak~S. Dalalyan and Lionel Riou-Durand.
\newblock On sampling from a log-concave density using kinetic {L}angevin
  diffusions.
\newblock \emph{Bernoulli}, 26\penalty0 (3):\penalty0 1956--1988, 2020.

\bibitem[Daudel and Douc(2021)]{daudeldouc2021mixtureopt}
Kam\'{e}lia Daudel and Randal Douc.
\newblock Mixture weights optimisation for alpha-divergence variational
  inference.
\newblock In \emph{Advances in Neural Information Processing Systems},
  volume~34, pages 4397--4408, 2021.

\bibitem[Daudel et~al.(2021)Daudel, Douc, and
  Portier]{daudeldoucportier2021alphadiv}
Kam\'{e}lia Daudel, Randal Douc, and Fran\c{c}ois Portier.
\newblock Infinite-dimensional gradient-based descent for alpha-divergence
  minimisation.
\newblock \emph{Ann. Statist.}, 49\penalty0 (4):\penalty0 2250--2270, 2021.

\bibitem[Delon and Desolneux(2020)]{delon2020mixture}
Julie Delon and Agn\`es Desolneux.
\newblock A {W}asserstein-type distance in the space of {G}aussian mixture
  models.
\newblock \emph{SIAM J. Imaging Sci.}, 13\penalty0 (2):\penalty0 936--970,
  2020.

\bibitem[do~Carmo(1992)]{docarmo1992riemannian}
Manfredo~P. do~Carmo.
\newblock \emph{Riemannian geometry}.
\newblock Mathematics: Theory \& Applications. Birkh\"{a}user Boston, Inc.,
  Boston, MA, 1992.
\newblock Translated from the second Portuguese edition by Francis Flaherty.

\bibitem[Domke(2020)]{domke2020vismooth}
Justin Domke.
\newblock Provable smoothness guarantees for black-box variational inference.
\newblock In Hal~Daumé III and Aarti Singh, editors, \emph{Proceedings of the
  37th International Conference on Machine Learning}, volume 119 of
  \emph{Proceedings of Machine Learning Research}, pages 2587--2596. PMLR,
  13--18 Jul 2020.

\bibitem[{Duncan} et~al.(2019){Duncan}, {Nuesken}, and
  {Szpruch}]{duncan2019geometrysvgd}
Andrew {Duncan}, Nikolas {Nuesken}, and Lukasz {Szpruch}.
\newblock {On the geometry of {S}tein variational gradient descent}.
\newblock \emph{arXiv e-prints}, art. arXiv:1912.00894, December 2019.

\bibitem[Durmus et~al.(2019)Durmus, Majewski, and
  Miasojedow]{durmusmajewski2019lmcconvex}
Alain Durmus, Szymon Majewski, and B\l{}a\.{z}ej Miasojedow.
\newblock Analysis of {L}angevin {M}onte {C}arlo via convex optimization.
\newblock \emph{J. Mach. Learn. Res.}, 20:\penalty0 Paper No. 73, 46, 2019.

\bibitem[Eberle et~al.(2017)Eberle, Niethammer, and
  Schlichting]{eberleniethammerschlichting2017constrainedfp}
Simon Eberle, Barbara Niethammer, and Andr\'{e} Schlichting.
\newblock Gradient flow formulation and longtime behaviour of a constrained
  {F}okker--{P}lanck equation.
\newblock \emph{Nonlinear Anal.}, 158:\penalty0 142--167, 2017.

\bibitem[Galy-Fajou et~al.(2021)Galy-Fajou, Perrone, and
  Opper]{GalPerOpp21vargauss}
Th\'{e}o Galy-Fajou, Valerio Perrone, and Manfred Opper.
\newblock Flexible and efficient inference with particles for the variational
  {G}aussian approximation.
\newblock \emph{Entropy}, 23\penalty0 (8):\penalty0 Paper No. 990, 34, 2021.

\bibitem[Honkela and Valpola(2004)]{honkelavalpola2004variationalbayes}
Antti Honkela and Harri Valpola.
\newblock Unsupervised variational {B}ayesian learning of nonlinear models.
\newblock In \emph{Advances in Neural Information Processing Systems},
  volume~17, 2004.

\bibitem[Huang et~al.(2022)Huang, Huang, Reich, and
  Stuart]{huangetal2022derivfree}
Daniel~Z. Huang, Jiaoyang Huang, Sebastian Reich, and Andrew~M. Stuart.
\newblock Efficient derivative-free {B}ayesian inference for large-scale
  inverse problems.
\newblock \emph{arXiv e-prints}, art. arXiv:2204.04386, 2022.

\bibitem[Jordan et~al.(1999)Jordan, Ghahramani, Jaakkola, and
  Saul]{JorGhaJaa99}
Michael~I. Jordan, Zoubin Ghahramani, Tommi~S. Jaakkola, and Lawrence~K. Saul.
\newblock An introduction to variational methods for graphical models.
\newblock \emph{Mach. Learn.}, 37\penalty0 (2):\penalty0 183--233, 1999.

\bibitem[Jordan et~al.(1998)Jordan, Kinderlehrer, and Otto]{Jordan98}
Richard Jordan, David Kinderlehrer, and Felix Otto.
\newblock The variational formulation of the {F}okker--{P}lanck equation.
\newblock \emph{SIAM Journal on Mathematical Analysis}, 29\penalty0
  (1):\penalty0 1--17, 1998.

\bibitem[Julier and Uhlmann(2004)]{Julier04}
Simon~J. Julier and Jeffrey~K. Uhlmann.
\newblock Unscented filtering and nonlinear estimation.
\newblock \emph{Proceedings of the IEEE}, 92\penalty0 (3):\penalty0 401--422,
  2004.

\bibitem[Julier et~al.(2000)Julier, Uhlmann, and
  Durrant-Whyte]{julieruhlmanndurrantwhyte2000ukf}
Simon~J. Julier, Jeffrey~K. Uhlmann, and Hugh~F. Durrant-Whyte.
\newblock A new method for the nonlinear transformation of means and
  covariances in filters and estimators.
\newblock \emph{IEEE Trans. Automat. Control}, 45\penalty0 (3):\penalty0
  477--482, 2000.

\bibitem[Khan and H{\aa}vard(2022)]{Khan22}
Mohammad~Emtiyaz Khan and Rue H{\aa}vard.
\newblock The {B}ayesian learning rule.
\newblock \emph{arXiv:2107.04562}, 2022.

\bibitem[Knoblauch et~al.(2022)Knoblauch, Jewson, and
  Damoulas]{knojewdam2022vi}
Jeremias Knoblauch, Jack Jewson, and Theodoros Damoulas.
\newblock An optimization-centric view on {B}ayes' rule: reviewing and
  generalizing variational inference.
\newblock \emph{Journal of Machine Learning Research}, 23\penalty0
  (132):\penalty0 1--109, 2022.

\bibitem[Lambert et~al.(2021)Lambert, Bonnabel, and Bach]{Lambert21}
Marc Lambert, Silv{\`e}re Bonnabel, and Francis Bach.
\newblock The limited-memory recursive variational {G}aussian approximation
  ({L-RVGA}).
\newblock \emph{hal-03501920}, 2021.

\bibitem[Lambert et~al.(2022{\natexlab{a}})Lambert, Bonnabel, and
  Bach]{Lambert22}
Marc Lambert, Silv{\`e}re Bonnabel, and Francis Bach.
\newblock The recursive variational {G}aussian approximation ({R-VGA}).
\newblock \emph{Statistics and Computing}, 32\penalty0 (1):\penalty0 10,
  2022{\natexlab{a}}.

\bibitem[Lambert et~al.(2022{\natexlab{b}})Lambert, Bonnabel, and
  Bach]{Lambert22b}
Marc Lambert, Silv{\`e}re Bonnabel, and Francis Bach.
\newblock The continuous-discrete variational {K}alman filter ({CD-VKF}).
\newblock In \emph{2022 61st IEEE Conference on Decision and Control (CDC)},
  2022{\natexlab{b}}.

\bibitem[Lee et~al.(2021)Lee, Shen, and Tian]{leeshentian2021rgo}
Yin~Tat Lee, Ruoqi Shen, and Kevin Tian.
\newblock Structured logconcave sampling with a restricted {G}aussian oracle.
\newblock In \emph{Proceedings of the Conference on Learning Theory}, volume
  134, pages 2993--3050, 15--19 Aug 2021.

\bibitem[Liero et~al.(2016)Liero, Mielke, and
  Savar\'{e}]{lieromielkesavare2016wfr1}
Matthias Liero, Alexander Mielke, and Giuseppe Savar\'{e}.
\newblock Optimal transport in competition with reaction: the
  {H}ellinger--{K}antorovich distance and geodesic curves.
\newblock \emph{SIAM J. Math. Anal.}, 48\penalty0 (4):\penalty0 2869--2911,
  2016.

\bibitem[Liero et~al.(2018)Liero, Mielke, and
  Savar\'{e}]{lieromielkesavare2018wfr2}
Matthias Liero, Alexander Mielke, and Giuseppe Savar\'{e}.
\newblock Optimal entropy-transport problems and a new
  {H}ellinger--{K}antorovich distance between positive measures.
\newblock \emph{Invent. Math.}, 211\penalty0 (3):\penalty0 969--1117, 2018.

\bibitem[Lin et~al.(2019{\natexlab{a}})Lin, Khan, and Schmidt]{Wu19}
Wu~Lin, Mohammad~E. Khan, and Mark Schmidt.
\newblock Stein's lemma for the reparameterization trick with exponential
  family mixtures.
\newblock \emph{arXiv preprint 1910.13398}, 2019{\natexlab{a}}.

\bibitem[Lin et~al.(2019{\natexlab{b}})Lin, Khan, and
  Schmidt]{linkhanschmidt2019mixturevi}
Wu~Lin, Mohammad~E. Khan, and Mark Schmidt.
\newblock Fast and simple natural-gradient variational inference with mixture
  of exponential-family approximations.
\newblock In \emph{Proceedings of the International Conference on Machine
  Learning}, volume~97, pages 3992--4002, 09--15 Jun 2019{\natexlab{b}}.

\bibitem[Liu and Nocedal(1989)]{Liu89}
Dong~C. Liu and Jorge Nocedal.
\newblock On the limited memory {BFGS} method for large scale optimization.
\newblock \emph{Math. Programming}, 45\penalty0 (3, (Ser. B)):\penalty0
  503--528, 1989.

\bibitem[Liu(2017)]{liu2017svgdgf}
Qiang Liu.
\newblock Stein variational gradient descent as gradient flow.
\newblock In \emph{Advances in Neural Information Processing Systems},
  volume~30, 2017.

\bibitem[Liu and Wang(2016)]{liuwang2016svgd}
Qiang Liu and Dilin Wang.
\newblock Stein variational gradient descent: a general purpose {B}ayesian
  inference algorithm.
\newblock In \emph{Advances in Neural Information Processing Systems},
  volume~29, 2016.

\bibitem[{Lu} et~al.(2019){Lu}, {Lu}, and {Nolen}]{lulunolen2019birthdeath}
Yulong {Lu}, Jianfeng {Lu}, and James {Nolen}.
\newblock {Accelerating {L}angevin sampling with birth-death}.
\newblock \emph{arXiv e-prints}, art. arXiv:1905.09863, May 2019.

\bibitem[Ma et~al.(2021)Ma, Chatterji, Cheng, Flammarion, Bartlett, and
  Jordan]{maetal2021nesterovmcmc}
Yi-An Ma, Niladri~S. Chatterji, Xiang Cheng, Nicolas Flammarion, Peter~L.
  Bartlett, and Michael~I. Jordan.
\newblock Is there an analog of {N}esterov acceleration for gradient-based
  {MCMC}?
\newblock \emph{Bernoulli}, 27\penalty0 (3):\penalty0 1942 -- 1992, 2021.

\bibitem[Malag\`o et~al.(2018)Malag\`o, Montrucchio, and
  Pistone]{malmonpis2018bw}
Luigi Malag\`o, Luigi Montrucchio, and Giovanni Pistone.
\newblock Wasserstein {R}iemannian geometry of {G}aussian densities.
\newblock \emph{Inf. Geom.}, 1\penalty0 (2):\penalty0 137--179, 2018.

\bibitem[Modin(2017)]{modin2017matrixdecomposition}
Klas Modin.
\newblock Geometry of matrix decompositions seen through optimal transport and
  information geometry.
\newblock \emph{J. Geom. Mech.}, 9\penalty0 (3):\penalty0 335--390, 2017.

\bibitem[Morf et~al.(1977)Morf, Levy, and Kailath]{Morf77}
Martin Morf, Bernard Levy, and Thomas Kailath.
\newblock Square-root algorithms for the continuous-time linear least squares
  estimation problem.
\newblock In \emph{1977 IEEE Conference on Decision and Control including the
  16th Symposium on Adaptive Processes and A Special Symposium on Fuzzy Set
  Theory and Applications}, pages 944--947, 1977.

\bibitem[Opper and Archambeau(2009)]{Opper09}
Manfred Opper and C\'{e}dric Archambeau.
\newblock The variational {G}aussian approximation revisited.
\newblock \emph{Neural Comput.}, 21\penalty0 (3):\penalty0 786--792, 2009.

\bibitem[Otto(1998)]{otto1998magnetic}
Felix Otto.
\newblock Dynamics of labyrinthine pattern formation in magnetic fluids: a
  mean-field theory.
\newblock \emph{Arch. Rational Mech. Anal.}, 141\penalty0 (1):\penalty0
  63--103, 1998.

\bibitem[Otto(2001)]{otto2001porousmedium}
Felix Otto.
\newblock The geometry of dissipative evolution equations: the porous medium
  equation.
\newblock \emph{Comm. Partial Differential Equations}, 26\penalty0
  (1-2):\penalty0 101--174, 2001.

\bibitem[Paisley et~al.(2012)Paisley, Blei, and Jordan]{PaiBleJor12}
John Paisley, David~M. Blei, and Michael~I. Jordan.
\newblock Variational {B}ayesian inference with stochastic search.
\newblock In \emph{Proceedings of the International Conference on Machine
  Learning}, pages 1363--1370, 2012.

\bibitem[{Peyré} and {Cuturi}(2019)]{peyre2019computational}
Gabriel {Peyré} and Marco {Cuturi}.
\newblock \emph{Computational optimal transport: with applications to data
  science}.
\newblock Now, 2019.

\bibitem[Ranganath et~al.(2014)Ranganath, Gerrish, and Blei]{RanGerBle14}
Rajesh Ranganath, Sean Gerrish, and David~M. Blei.
\newblock Black box variational inference.
\newblock In \emph{Proceedings of International Conference on Artificial
  Intelligence and Statistics}, volume~33, pages 814--822, Reykjavik, Iceland,
  22--25 Apr 2014.

\bibitem[Santambrogio(2015)]{santambrogio2015ot}
Filippo Santambrogio.
\newblock \emph{Optimal transport for applied mathematicians}, volume~87 of
  \emph{Progress in Nonlinear Differential Equations and their Applications}.
\newblock Birkh\"{a}user/Springer, Cham, 2015.
\newblock Calculus of variations, PDEs, and modeling.

\bibitem[S\"{a}rkk\"{a}(2007)]{Sarkka07}
Simo S\"{a}rkk\"{a}.
\newblock On unscented {K}alman filtering for state estimation of
  continuous-time nonlinear systems.
\newblock \emph{IEEE Trans. Automat. Control}, 52\penalty0 (9):\penalty0
  1631--1641, 2007.

\bibitem[Seeger(1999)]{seeger1999bayesianmodelselection}
Matthias Seeger.
\newblock Bayesian model selection for support vector machines, {G}aussian
  processes and other kernel classifiers.
\newblock In \emph{Advances in Neural Information Processing Systems},
  volume~12, 1999.

\bibitem[Shen and Lee(2019)]{shenlee2019randomizedmidpoint}
Ruoqi Shen and Yin~Tat Lee.
\newblock The randomized midpoint method for log-concave sampling.
\newblock In \emph{Advances in Neural Information Processing Systems},
  volume~32, 2019.

\bibitem[Tudorascu and Wunsch(2011)]{tudorascuwunsch2011nonlocal}
Adrian Tudorascu and Marcus Wunsch.
\newblock On a nonlinear, nonlocal parabolic problem with conservation of mass,
  mean and variance.
\newblock \emph{Comm. Partial Differential Equations}, 36\penalty0
  (8):\penalty0 1426--1454, 2011.

\bibitem[Vempala and Wibisono(2019)]{vempala2019ulaisoperimetry}
Santosh Vempala and Andre Wibisono.
\newblock Rapid convergence of the unadjusted {L}angevin algorithm:
  isoperimetry suffices.
\newblock In \emph{Advances in Neural Information Processing Systems 32}, pages
  8094--8106. 2019.

\bibitem[Villani(2003)]{villani2003topics}
C\'{e}dric Villani.
\newblock \emph{Topics in optimal transportation}, volume~58 of \emph{Graduate
  Studies in Mathematics}.
\newblock American Mathematical Society, Providence, RI, 2003.

\bibitem[Villani(2009)]{villani2009ot}
C\'{e}dric Villani.
\newblock \emph{Optimal transport}, volume 338 of \emph{Grundlehren der
  Mathematischen Wissenschaften [Fundamental Principles of Mathematical
  Sciences]}.
\newblock Springer-Verlag, Berlin, 2009.
\newblock Old and new.

\bibitem[Wainwright and Jordan(2008)]{WaiJor08}
Martin~J. Wainwright and Michael~I. Jordan.
\newblock Graphical models, exponential families, and variational inference.
\newblock \emph{Foundations and Trends in Machine Learning}, 1\penalty0
  (1--2):\penalty0 1--305, 2008.

\bibitem[Wang and Blei(2019)]{wangblei2019vbayesconsistency}
Yixin Wang and David~M. Blei.
\newblock Frequentist consistency of variational {B}ayes.
\newblock \emph{J. Amer. Statist. Assoc.}, 114\penalty0 (527):\penalty0
  1147--1161, 2019.

\bibitem[Wibisono(2018)]{wibisono2018samplingoptimization}
Andre Wibisono.
\newblock Sampling as optimization in the space of measures: the {L}angevin
  dynamics as a composite optimization problem.
\newblock In \emph{Proceedings of the 31st Conference On Learning Theory},
  volume~75, pages 2093--3027, 2018.

\bibitem[Wu et~al.(2022)Wu, Schmidler, and Chen]{wuschche2022minimaxmala}
Keru Wu, Scott Schmidler, and Yuansi Chen.
\newblock Minimax mixing time of the {M}etropolis-adjusted {L}angevin algorithm
  for log-concave sampling.
\newblock \emph{Journal of Machine Learning Research}, 23\penalty0
  (270):\penalty0 1--63, 2022.

\bibitem[Xu and Campbell(2022)]{XuCam22computationalgvi}
Zuheng Xu and Trevor Campbell.
\newblock The computational asymptotics of {G}aussian variational inference and
  the {L}aplace approximation.
\newblock \emph{Stat. Comput.}, 32\penalty0 (4):\penalty0 Paper No. 63, 37,
  2022.

\bibitem[Zhang et~al.(2018)Zhang, Sun, Duvenaud, and Grosse]{ZhaSunDuv18}
Guodong Zhang, Shengyang Sun, David Duvenaud, and Roger Grosse.
\newblock Noisy natural gradient as variational inference.
\newblock In \emph{Proceedings of the International Conference on Machine
  Learning}, volume~80, pages 5852--5861, 2018.

\end{thebibliography}

\end{document}